\documentclass{article}
\usepackage[utf8]{inputenc}
\usepackage{fancyhdr}
\usepackage{extramarks}
\usepackage{amsmath}
\usepackage{amsthm}
\usepackage{amsfonts}
\usepackage{tikz}
\usepackage{algorithm}
\usepackage[noend]{algpseudocode}
\usepackage{bbm}
\usepackage{calc}  
\usepackage{enumitem}  
\usepackage{booktabs}   
\usepackage{ltablex}
\usepackage{hyperref}
\usepackage{commath}
\usepackage{mathtools}
\usepackage{bbm}
\usepackage{bm}
\usepackage{amssymb}
\usepackage{url}
\usepackage{authblk}
\usepackage{caption}
\usepackage{siunitx}
\usepackage{natbib}

\hypersetup{
    colorlinks=true,
    linkcolor=blue,
    filecolor=magenta,      
    urlcolor=cyan,
    citecolor = blue,
    pdftitle={Overleaf Example},
    pdfpagemode=FullScreen,
    }

\newtheorem*{assumption*}{\assumptionnumber}
\providecommand{\assumptionnumber}{}
\makeatletter
\newenvironment{assumption}[1]
 {%
  \renewcommand{\assumptionnumber}{Assumption #1}%
  \begin{assumption*}%
  \phantomsection%
  \protected@edef\@currentlabel{#1}%
 }
 {%
  \end{assumption*}
 }
\makeatother

\makeatletter
\newenvironment{breakablealgorithm}
  {% \begin{breakablealgorithm}
   \begin{center}
     \refstepcounter{algorithm}% New algorithm
     \hrule height.8pt depth0pt \kern2pt% \@fs@pre for \@fs@ruled
     \renewcommand{\caption}[2][\relax]{% Make a new \caption
       {\raggedright\textbf{\fname@algorithm~\thealgorithm} ##2\par}%
       \ifx\relax##1\relax % #1 is \relax
         \addcontentsline{loa}{algorithm}{\protect\numberline{\thealgorithm}##2}%
       \else % #1 is not \relax
         \addcontentsline{loa}{algorithm}{\protect\numberline{\thealgorithm}##1}%
       \fi
       \kern2pt\hrule\kern2pt
     }
  }{% \end{breakablealgorithm}
     \kern2pt\hrule\relax% \@fs@post for \@fs@ruled
   \end{center}
  }
\makeatother

\newtheorem{prop}{Proposition}
\newtheorem{thm}{Theorem}
\newtheorem{lem}{Lemma}

\newtheorem{remark}{Remark}

\newcommand{\ind}{\perp\!\!\!\!\perp}

\allowdisplaybreaks

\title{
    \textmd{\textbf{Online Linear Programming with Batching}}\\
    }
\author[1]{Haoran Xu}
\author[1]{Peter W. Glynn}
\author[1]{Yinyu Ye}
\affil[1]{Department of Management Science and Engineering, Stanford University}
\affil[1]{\{haoran14, glynn, yyye\}@stanford.edu}
\date{}

\begin{document}

\maketitle
\begin{abstract}
    We study Online Linear Programming (OLP) with batching. The planning horizon is cut into $K$ batches, and decisions on orders can be delayed to the end of their associated batch. The ability to delay decisions improves operational performance, as measured by regret. We study two questions: (1) What is a lower bound on the regret as a function of $K$ and the length of the planning horizon? (2) Which algorithms can achieve this regret lower bound? This paper analyzes these questions when the distribution of the reward has a continuous support. We provide an $\Omega(\log K)$ regret lower bound in the single-resource case, and we provide pricing algorithms having an $O(\log K)$ regret in the setting with Poisson arrivals and multiple types of resources. All the algorithms update the prices at most $K$ times and only delay orders of the first and the last batches. All the regret bounds are independent of the length of the planning horizon. Finally, we study a more realistic large support setting where the number of distinct order types is finite but scales with the total number of orders. We prove that our $\Theta(\log K)$ bounds still hold for the batching operation of a multisecretary problem with discretized uniform rewards in this large support setting, provided the support grows fast enough. This suggests that the continuous support setting that is the paper's focus serves as a useful theoretical surrogate for the more realistic large finite support setting.
\end{abstract}
\section{Introduction}
Many resource allocation problems can be formulated as Linear Programming (LP) problems. In some settings, the decision maker can first collect information from all the orders and then solve the LP problem to obtain the optimal allocation. We call this a static or offline formulation. However, assuming a static environment may be unrealistic in many real-world applications. The dynamic formulation known as online resource allocation has attracted the attention of the operations research and management science community. In this paper, we use Online Linear Programming (OLP) (See \cite{Agrawal2014}) to study the online resource allocation problem.

A key feature of online resource allocation is that orders arrive sequentially and the decision maker is required to make immediate and irrevocable decisions without knowledge of individual future customer demands. This requirement brings online resource allocation problems closer to real-world applications in which data is revealed sequentially; however, since customers are usually willing to tolerate a short wait, this requirement may be unnecessarily restrictive because it implicitly assumes that all customers are impatient (i.e., must be processed immediately). Under this strong assumption, the decision maker may lose the opportunity to improve performance by delaying decisions on some orders. Thus, it is appropriate to study the online resource allocation problem when batching is permitted.

The impact of postponing real-time decisions has been recently studied in many applications of online decision making, including order fulfillment (\cite{Wang2023}), kidney exchange (\cite{Ashlagi2021}), and ride hailing (\cite{Feng.et.al.2023}). In this paper, we study the impact of the batching operation on OLP. The major part of our paper focuses on the following standard regime and the setting in which the distribution has continuous support. Under this regime, the total number of orders $n$ is known and used as the length of the planning horizon. In the batching operation associated with the standard regime, the $n$ orders are divided into $K$ batches with equal size, and the decisions on the orders arriving within a batch can be delayed to the end of their associated batch. We call $B=\frac{n}{K}$ the batch size. The total amount of resource available to the decision maker is limited, and the goal of the decision maker is to maximize the expected total reward by accepting and rejecting the orders. We focus on the random input model introduced in \cite{li2019olp}. The reward and resource consumption of the orders are modeled as i.i.d. random variables following some joint distribution $\mathcal{P}_{r,a}$. In addition, we consider the setting in which the total amount of resource available is proportional to the total number of orders. We call the total amount of resources available divided by $n$ the initial average resource. We measure the performance of the algorithms by regret, and we define the regret of an algorithm to be the difference between the expected total reward earned by the algorithm and the expected optimal reward given by the offline formulation. We are interested in the following two research questions: (1) What is a lower bound of the regret as a function of the number of batches and the length of the planning horizon? (2) Which algorithms can achieve the lower bound? We give a thorough discussion of these two questions in the standard regime and continuous support setting. The main contributions of this paper are summarized in the next subsection.
\subsection{Main Contributions}

We formulate OLP with batching as a Markov Decision Process (MDP) under the standard regime and continuous support setting, and we establish that the regret of the optimal online policy is $\Omega(\log K)$ when there is only a single type of resource. For the case with multiple resource types, we propose two algorithms that achieve an $O(\log K)$ regret. The hidden constants in the Big-O and Big-Omega results we derive depend only on $\mathcal{P}_{r,a}$ and the initial average resource. Consequently, the regret depends solely on the number of batches $K$ and is uniform in the total number of orders $n$. Regarding the two proposed algorithms, the first requires $B$ i.i.d. samples from $\mathcal{P}_{r,a}$ at the beginning of the planning horizon; it allows for weaker assumptions on resource consumption and only delays orders in the final batch. We view this assumption as corresponding to a real-world setting in which a limited amount of historical data is available. The second algorithm assumes $\mathcal{P}_{r,a}$ is completely unknown but imposes stronger distributional assumptions on resource consumption. As a result, it requires order delays in both the first and last batches. Because the single-resource case inherently satisfies the stronger assumptions required by the second algorithm, our analysis indicates the order of regret for batched OLP is tight at $\Theta(\log K)$ in the single-resource case. Furthermore, we demonstrate that neither prior knowledge of $\mathcal{P}_{r,a}$ nor delaying orders in the intermediate batches improves this asymptotic regret order, at least in the single-resource case. Finally, our analysis of the second algorithm introduces a Leave-Some-Out technique for evaluating offline dual prices, which may be of independent theoretical interest.

Our proposed algorithms and regret analyses also bring important managerial insights. A primary insight is that under the batching operation, the key driver of regret is the number of batches $K$. This contrasts with existing theory for OLP without batching, where the regret is driven by the total number of orders $n$. This shift occurs because our $O(\log K)$ regret bound holds uniformly in $n$. We note that the existing theory in the continuous support setting corresponding to the case $K=n$, so our results subsume the existing theory. Furthermore, all the approaches proposed in this paper are pricing algorithms that update prices only at the beginning of each batch by re-solving an LP. In the setting without batching, standard re-solving algorithms are often criticized as impractical due to the heavy computational burden of solving an LP at every time step. Consequently, methods like dual stochastic gradient descent are frequently preferred in practice because they require less computation, even though they cannot achieve optimal order regret. However, when batching is allowed, re-solving algorithms become highly practical. Because we only need to solve one LP for each batch, the overall computational complexity is drastically reduced. Making re-solving algorithms computationally viable is therefore a distinct managerial benefit of batching. Beyond computation, many real-world applications impose constraints on the frequency of price updates. Changing prices frequently requires extensive internal communication, making step-by-step updates impossible. By updating prices infrequently, our algorithms naturally satisfy these managerial constraints and ensure more stable operations. Ultimately, the batching operation reduces both the managerial overhead of frequent price changes and the computational burden of policy implementation. Decision-makers can therefore focus on choosing an optimal $K$ that balances regret (the expected lost profit) against customer order delays and computational costs. The number of orders $n$ becomes largely irrelevant as a managerial decision parameter.

Another important managerial insight is that delaying orders only in the first and last batches is sufficient to achieve the optimal order of regret. This modest delay requirement motivates the practicality of the batching operation, because holding back a large portion of customer orders is simply impossible in many real-world applications. While this approach avoids delays by processing orders immediately during the intermediate batches, we must still carefully manage how decisions are made in these periods. Specifically, achieving tight regret bounds requires an intelligent re-solving schedule. Our algorithm utilizes the Action-History-Dependent Learning Algorithm (AhdLA) framework to re-compute prices for these intermediate batches. We demonstrate that if no price updates occur before the start of the last batch, the regret degrades significantly to $O(\min\{K, n^{1/2}\})$. This result highlights that the specific timing of our infrequent re-solving drives the optimal performance. Consequently, our analysis provides practical guidance for managing resource allocation throughout these intermediate periods.

As a key parameter of the batching operation, the total number of batches $K$ also represents a higher-level managerial decision. Intuitively, a smaller $K$ leads to lower regret, and our theoretical regret bounds rigorously justify this intuition. However, a smaller $K$ implies longer batch durations, which inevitably increases customer waiting times. To address the practical question of how to optimally select $K$, we introduce the concepts of a delay budget and customer patience. Rather than treating the delay budget simply as a fixed parameter, we frame it as a strategic managerial decision. Increasing the delay budget incurs an implicit cost to the business, but it allows for a smaller $K$. Our established logarithmic regret bounds enable decision-makers to formally balance this implicit cost of delay against the operational cost of resource misallocation (the regret). Meanwhile, customer patience serves as a hard constraint on the system, dictating the maximum acceptable waiting time for any individual order. Ultimately, selecting the optimal $K$ becomes a process of optimizing the delay budget to minimize total economic costs, strictly subject to the customer patience constraint.

Apart from the standard regime, we also study a second regime where the total number of orders $n$ is random and unknown to the decision maker. In this regime, we assume orders arrive according to a Poisson process over a planning horizon of time length $T$. The batching operation divides this horizon $[0,T]$ into $K$ equal-length sub-intervals. Orders arriving within the same sub-interval form a batch, and the time duration of each batch is defined as $B=\frac{T}{K}$. We assume the total available resource capacity is proportional to the expected total number of orders. The assumptions on reward and resource consumption remain identical to the standard regime. Provided that the inter-arrival times are independent of the reward and resource consumption random variables, we propose an algorithm that also achieves an $O(\log K)$ regret. Therefore, under appropriate assumptions, a random total number of orders does not lead to a higher order of regret when the batching operation is employed. Crucially, the hidden constants in this $O(\log K)$ bound depend only on the joint distribution of the reward and resource consumption and the initial average resource. This mathematical property leads to a valuable managerial insight. In practice, it is reasonable to assume that a customer's maximum acceptable waiting time is independent of the overall customer arrival rate. Therefore, delaying customers only in the first and last batches and scaling the customer arrival rate to infinity, while keeping $B$ and $T$ fixed, represents a realistic asymptotic regime. Under this regime, our theory shows that the regret remains constant as the business scale (customer arrival rate) grows.

The known regret bounds for finitely supported distributions differ significantly from those available in the continuous support setting. The finite support case with a large $n$ corresponds to a real-world setting with a small number of product types. In contrast, the continuous support case can be viewed as describing a ``make-to-order" environment, where every order is customized to the customer's specific requirements. This leaves open the question of which of these two theories is most relevant to the case where both the number of product types and the total number of orders $n$ are large. To address this and to justify our paper's emphasis on the continuous support setting, we consider an asymptotic regime where the support size scales as $S = \Theta(n^\alpha)$ for $\alpha > 0$. To analyze the regret behavior in this environment, we study a multisecretary problem with discretized uniform rewards. We prove that the $O(\log K)$ regret upper bound and the $\Omega(\log K)$ lower bound derived under the continuous support setting still hold in this large support setting when $\alpha \geq 1/2$. Hence, the continuous support assumption serves as a rigorous and useful theoretical surrogate for this large support environment. Existing regret bounds for the finite support setting (\cite{Xie.et.al.2023}, \cite{Chen.et.al.2023}, \cite{Vera2021}) could technically be applied here; however, those bounds exhibit a poor dependence on number of order types $S$. Consequently, deriving bounds directly from finite support theory yields much worse results than our $\Theta(\log K)$ bound. Our results provide strong evidence that continuous support is an appropriate approximation for the large support setting, provided the support size grows fast enough. The continuity of the conditional distribution can thus be viewed merely as a technical assumption that simplifies mathematical analysis, allowing the managerial insights derived from our continuous support theory to be reasonably applied to the large support setting.
\subsection{Related works}
Online resource allocation has been extensively studied within the operations research and management science community. We refer readers to \cite{BalseiroSurvey} for a comprehensive survey of the online resource allocation problem. Several works have studied the impact of delay and batching on this problem. \cite{Golrezaei2021} study an online resource allocation problem with both impatient and partially patient customers. In their model, the decision maker must make immediate and irrevocable decisions on orders from impatient customers, but decisions on orders from partially patient customers can be delayed for several time periods. \cite{Xie.et.al.2023} provide a foundational framework for online decision making with delay, elegantly demonstrating the value of batching in online resource allocation. They prove that as long as the underlying distribution has a finite support size that is independent of the length of the planning horizon $n$, their algorithm achieves a regret upper bound that is independent of $n$ and decreases exponentially in the batch size $B$ (i.e., $O(\rho^B)$ where $\rho < 1$). However, as discussed above, this regret upper bound degrades badly when the support size $S\rightarrow +\infty$. In addition, this regret bound implies that the batch size $B$ is the key parameter of the batching operation under the modeling assumptions of \cite{Xie.et.al.2023}, which makes a significant difference from our work since the key parameter of the batching operation is the number of batches $K$ under our modeling assumptions. Another key difference between our work and \cite{Xie.et.al.2023} is our operational delay mechanism. Our algorithms delay decisions only for orders arriving in the first and last batches; for orders arriving in all intermediate batches, we make immediate and irrevocable decisions. On the other hand, \cite{Xie.et.al.2023} delay orders in all the batches. We demonstrate that the regret can still be reduced significantly even when the vast majority of orders are processed immediately, resulting in a highly practical operational environment.
 
\cite{BF2025} also study the batching operation and mention the importance of delaying orders arriving at the end of the planning horizon. They consider the finite support setting and propose an algorithm that only delays orders in the last batch. In addition, their algorithm only updates the acceptance rule at the beginning of the last batch via re-optimization. They show that this algorithm achieves an $O(1)$ regret when the size of the last batch is of a larger order than $n^{1/2}$, where $n$ is the total number of customer orders. Although both our work and \cite{BF2025} emphasize that delaying orders only in the last batch is sufficient for good performance, $O(1)$ regret is not optimal in the finite support batched setting according to \cite{Xie.et.al.2023}. Thus, \cite{BF2025} did not prove that one could obtain the optimal regret bound by delaying only the orders in the last batch within their finite support setting. By contrast, our lower bound implies that the optimal order of regret can be achieved with such a policy, thereby fully closing this theoretical loop, albeit in the continuous setting. Furthermore, our result holds without any condition on the batch size.

There are also papers studying how batching and delay can improve algorithm performance in the online matching problem. Many of them focus on the analysis of the competitive ratio. \cite{LeeSingla} study online matching problems in which the online arrivals of edges in a graph are batched into $s$ stages. For the $s$-stage integral general matching problem, they propose a polynomial-time algorithm with a $\frac{1}{2}+\frac{1}{2^{O(s)}}$ competitive ratio. For the two-stage fractional general matching problem, they propose a polynomial-time algorithm with a 0.6 competitive ratio. \cite{FengandNia2022} study a vertex-weighted matching problem, where online arrivals are batched into $K$ stages. They propose a matching algorithm with a $1-\left(1-1/K\right)^K$ competitive ratio for the problem. The randomness of online arrivals and the impact of customer impatience are also studied in the literature related to online matching with batching and delay. For instance, \cite{Blanchet2022} consider online matching problems in a two-sided market, where buyers and sellers arrive following Poisson processes and leave the market after an exponentially distributed amount of time due to impatience. The decision maker needs to decide when to match and how to match based on the number of buyers and sellers and the utility of possible matches. In our paper, we also study the setting in which orders arrive following a Poisson process. For more work on the impact of delay on online matching, we refer readers to \cite{Kerimov2023} and references therein.

In this paper, we study the online resource allocation problem under the framework of OLP. There are two major classes of online linear programming algorithms. One class of OLP algorithms makes decisions by repeatedly re-solving LPs, and the other class of OLP algorithms uses the dual stochastic gradient descent first-order methods. \cite{Agrawal2014} study the OLP under the random permutation model and propose the Dynamic Learning Algorithm (DLA). DLA periodically solves dual LPs to update the dual price and uses the dual price to make resource allocation decisions. These dual LPs are formulated based on the revealed information of the orders, but the average resource used in the formulation is always the initial average resource. \cite{li2019olp} study the OLP under the random input model and propose the Action-History-Dependent Learning Algorithm (AhdLA). AhdLA can be viewed as a modification of DLA, which formulates dual LPs using the remaining average resource instead of the initial average resource. \cite{li2019olp} introduce the concept of dual convergence and use the dual convergence result to prove that AhdLA has an $O(\log n\log\log n)$ regret. \cite{Bray2019} improves the dual convergence result and shows that the regret of AhdLA is $O(\log n)$. In addition, \cite{Bray2019} provides an $\Omega(\log n)$ lower bound for the problem and shows that AhdLA achieves the optimal order of regret. AhdLA needs to solve an LP when each order arrives. Thus, AhdLA has a high computational complexity. Compared to algorithms that repeatedly solve LPs, algorithms using first-order methods have lower computational complexity. \cite{Balseiro2023} propose a dual stochastic gradient descent method, and \cite{lisunye2023} propose a projected stochastic subgradient descent method. Both of these methods have an $O(\sqrt{n})$ regret. \cite{Gao2024} propose a first-order method that has $o(\sqrt{n})$ regret. Although first-order methods have lower computational complexity, to the best of our knowledge, there is no first-order method that has been shown to achieve logarithmic regret. There are also papers combining the re-solving methods and the first-order methods, for example, \cite{LiWangZhang2024} and \cite{Sun2025}. The regret of the algorithms in these papers is better than that of the pure first-order methods because of the re-solving steps. The computational complexity of these algorithms is lower than that of the pure re-solving algorithms since they re-solve LPs less frequently. The algorithms in our paper are related to these papers because our algorithms also re-solve LPs infrequently.

The paper is organized as follows. Section \ref{Section: Preliminary} introduces preliminaries and useful properties of dual price used in the regret analysis. Section \ref{sec: olp with delay} provides problem formulation, the algorithms and the regret bounds. Proofs of the regret bounds are given in Section \ref{Section: Proofs}. Numerical experiments and discussions of the algorithms are given in Section \ref{Section: NumExp}. Finally, we summarize the paper in Section \ref{Sec: Conclusion}.
\section{Preliminaries}\label{Section: Preliminary}
The linear resource allocation problem can be formulated as the following linear integer program.
\begin{equation}\label{eqn: IntegerProgramming}
    \begin{array}{ll}
        \max & \sum_{j=1}^nr_jx_j  \\
         s.t & \sum_{j=1}^na_j\cdot x_j\leq b_0 \\
          & x_j\in \{0,1\}\quad\forall j=1,\cdots,n
    \end{array}
\end{equation}
where $n\in\mathbb{N}_{+}$ is the total number of orders, $b_0\in\mathbb{R}^m_{+}$ is the total resource, $m$ is the number of the types of resources, and  $(r_j,a_j)\in\mathbb{R}^{m+1}$ is the reward and resource consumption of the $j$th order. Consider the following linear relaxation
\begin{equation}\label{eqn: PrimalLinearProgramming}
    \begin{array}{ll}
        \max & \sum_{j=1}^nr_jx_j  \\
         s.t & \sum_{j=1}^na_j\cdot x_j\leq b_0 \\
          & 0\leq x_j\leq 1\quad\forall j=1,\cdots,n.
    \end{array}
\end{equation}
The dual formulation of (\ref{eqn: PrimalLinearProgramming}) is
\begin{equation}\label{eqn:DualLinearProgramming}
    \begin{array}{ll}
       \min  & b_0^Tp + \sum_{j=1}^n y_j \\
       s.t  & {a_j}^Tp+y_j\geq r_j \quad\forall j=1,\cdots,n\\
       &p\geq 0,\;y_j\geq 0\quad\forall j=1,\cdots,n\;\;\;\;.
    \end{array}
\end{equation}
Let $x_j^*$ and $p_n^*$ be the optimal solution of (\ref{eqn: PrimalLinearProgramming}) and (\ref{eqn:DualLinearProgramming}). By complementary slackness,
\begin{equation}\label{eqn: ComplementarySlackness}
    x_j^*=\left\{
    \begin{array}{ll}
      1   & \mathrm{if}\;r_j>{a_j}^Tp_n^* \\
      0   & \mathrm{if}\;r_j<{a_j}^Tp_n^*\;\;\;\;.
    \end{array}
    \right.
\end{equation}
Except for the case in which $r_j = {a_j}^T{p_n^*}$,  (\ref{eqn: ComplementarySlackness}) implies that knowing $p_n^*$ is equivalent to knowing $x_j^*$. When the problem is in general position (\cite{Agrawal2014}), i.e., $r_j = {a_j}^T{p_n^*}$ happens for at most $m$ orders, $\left\{\mathbbm{1}\left\{r_j>{a_j}^Tp_n^*\right\}\right\}_{j=1}^n$ is a near-optimal solution of (\ref{eqn: IntegerProgramming}) when $n$ is much larger than $m$. In addition, if $p_n^*$ is provided to the decision maker in advance, the decision $\mathbbm{1}\left\{r_j>{a_j}^Tp_n^*\right\}$ for each order $j$ can be implemented without knowing the reward and the resource consumption of all the other orders. Throughout this paper, we define
\begin{equation*}
    \Delta_j(p) = r_j - {a_j}^Tp
\end{equation*}
where we can view $\Delta_j(p)$ as the net value of order $j$ under the dual price $p$. 

Let $\pi$ be the resource allocation policy used by the decision maker, and $x_j^{\pi}\in \{0,1\}$ be the decision made on order $j$ under policy $\pi$. In the setting of OLP, orders arrive sequentially, and the decision maker needs to accept or reject an order without knowing information regarding subsequent orders that have not yet arrived. Therefore, (\ref{eqn: ComplementarySlackness}) motivates the extensively studied dual-based algorithms of OLP. Let $\mathcal{H}_j$ be all the information collected by the decision maker up to the arrival of order $j$. Then, the idea of dual-based algorithms is to compute a dual price $p(\mathcal{H}_j)$ and make the allocation decision on order $j$ through the dual decision rule (\ref{eqn:DualDecisionRule}).
\begin{equation}\label{eqn:DualDecisionRule}
    x_j^{\pi}=\mathbbm{1}\left\{\Delta_j\left(p(\mathcal{H}_j)\right)>0\right\}
\end{equation}
AhdLA in \cite{li2019olp} is one of the most extensively studied dual-based algorithms. Since it is closely related to the algorithms proposed in this paper, we summarize its details below. Denote AhdLA as $\pi_0$.
\begin{algorithm}[H]
    \caption{Action-History-Dependent Learning Algorithm (AhdLA)}
    \begin{algorithmic}[1]
        \State Initialize $p_1=0$.
        \For{$j=1,\cdots,n$} 
        \State Set $x_j^{\pi_0}=\mathbbm{1}\left\{\Delta_j(p_j)>0\right\}\mathbbm{1}\left\{b_{j-1}\geq a_j\right\}$ and $b_j = b_{j-1}-a_jx^{\pi_0}_j$ and compute 
        \begin{equation*}
            p_{j+1}=\arg\min_{p\geq 0} {\frac{{b_{j}}^Tp}{n-j}}+\frac{1}{j}\sum_{t=1}^{j}(\Delta_j(p))^+.
        \end{equation*}
        \EndFor
    \end{algorithmic}
\end{algorithm}
Since the information related to an order is unknown prior to its arrival in the setting of OLP, we model the reward and resource consumption $(r_j,a_j)$ as a random vector. Denote the optimal objective value of (\ref{eqn: PrimalLinearProgramming}) as $R_n^*$ and we call it optimal offline reward. Let $R_n(\pi)=\sum_{j=1}^nr_jx_j^{\pi}$ be the online reward earned by the policy $\pi$. In this paper, we use regret to measure the performance of a policy $\pi$, which is defined as
\begin{equation*}
    Regret_n(\pi) = \mathbb{E}\left[R_n^*-R_n(\pi)\right].
\end{equation*}
For simplicity, we sometimes use big-$O$, little-$o$, big-$\Omega$, and big-$\Theta$ notations to state the upper and lower bounds of the regret of algorithms, which are defined as follows. A function $f(x)$ is $O(g(x))$ if there exists a positive constant $C$ such that $|f(x)|\leq C\cdot g(x)$. A function $f(x)$ is $\Omega(g(x))$ if $g(x)$ is $O(f(x))$, and a function $f(x)$ is $\Theta(g(x))$ if $f(x)$ is both $O(g(x))$ and $\Omega(g(x))$. Finally, if $f(x)$ is $o(g(x))$, then for any $\epsilon>0$, there exists $x_0$ such that $|f(x)|\leq \epsilon g(x)$ for all $x\geq x_0$. Throughout this paper, we make the following assumptions on the distribution of $(r_1,a_1)$ and the total resource $b_0$. Let $e\in \mathbb{R}^m$ be the vectors with all entries equal to $1$.
\begin{assumption}{1}\label{Assumption: DistributionMulti}
    \;\\
    (a) $\left\{(r_j,a_j)\right\}_{j=1}^{+\infty}$ is a sequence of i.i.d. random pairs in $\mathbb{R}\times\mathbb{R}^{m}$, for which there exists positive real numbers $\underline{a}$, $\bar a$, $\bar r$ such that $\mathbb{P}\{0\leq r_1\leq \bar r\}=1$ and $\mathbb{P}\{\underline{a}e\leq a_1\leq \bar ae\}=1$.\\
    (b) $b_0=nd_0$ where $d_0\in\otimes_{i=1}^m(\underline{d},\bar{d})$ with $\bar{d}>\underline{d}>0$.
\end{assumption}
\noindent We call $d_0$ in Assumption \ref{Assumption: DistributionMulti} the initial average resource. Define $x^+$ to be $\max\{x,0\}$. For a positive integer $N$, define the offline dual price $p_N^*(\cdot):\mathbb{R}^m_{+}\longrightarrow\mathbb{R}^m_{+}$ to be the random function
\begin{equation*}
    p_N^*(d) = \arg\min_{p\geq 0} d^Tp + \frac{1}{N}\sum_{j=1}^N\left(\Delta_j(p)\right)^+.
\end{equation*}
We call $N$ the sample size of the offline dual price in this paper. In addition, for a non-positive integer $N\leq 0$, define
\begin{equation*}
    p_N^*(d) = 0\quad\forall d\in \mathbb{R}^m_{+}.
\end{equation*}
Define the population dual price $p^*(\cdot):\mathbb{R}^m_{+}\longrightarrow\mathbb{R}^m_{+}$ to be a deterministic function
\begin{equation*}
    p^*(d) = \arg\min_{p\geq 0} d^Tp + \mathbb{E}\left[\left(\Delta_1(p)\right)^+\right].
\end{equation*}
When the optimization problems in the definition of $p_N^*(\cdot)$ and $p^*(\cdot)$ have multiple optimizers, we assume one optimal solution is selected to be the value of the function. We do not specify which optimal solution is selected since it does not affect the regret analysis in this paper. With Assumption \ref{Assumption: DistributionMulti}, we have the following boundedness properties of population dual price and offline dual price:
\begin{equation}\label{eqn: Lemma1(a)}
    \begin{aligned}
        &e^Tp^*(d)\in \left[0,\frac{\bar{r}}{\underline{a}}\right]\\
        &\mathbb{P}\left\{e^Tp_N^*(d)\in \left[0,\frac{\bar{r}}{\underline{a}}\right]\right\}=1
    \end{aligned}
\end{equation}
To conduct the regret analysis, we make the following additional assumptions throughout the paper. These assumptions are used in \cite{li2019olp} and \cite{Bray2019}. Let $\norm{\cdot}_{\infty}$ be the sup norm in $\mathbb{R}^m$.
\begin{assumption}{2}\label{Assumption: DistributionDensity}
    \;\\
    (a) The entries of $p^*(d_0)$ are all positive.\\
    (b) Given $a_1$, $r_1$ has a bounded conditional density $f_{r|a}$, i.e, there exists a constant $\beta>0$ such that,
    \begin{equation*}
        f_{r|a}(u|v)\leq\beta\quad\forall u,v.
    \end{equation*}
    (c) There exists $\delta_p>0$ such that the Hessian of $g(p)=d^Tp+\mathbb{E}\left[\left(\Delta_1(p)\right)^+\right]$ exists, is positive definite, and is continuous in $p$ in $\left\{p:\norm{p-p^*(d_0)}_{+\infty}<\delta_p\right\}$.
\end{assumption}
\noindent Assumption~\ref{Assumption: DistributionDensity}(a) can be viewed as a non-degeneracy assumption, which requires that all the resources are asymptotically consumed when the number of orders goes to infinity. Assumption~\ref{Assumption: DistributionDensity}(b) and (c) indicates that the reward distribution has a continuous support. Assumption \ref{Assumption: DistributionDensity} also guarantees that there are at most $m$ orders such that $r_j={a_j}^Tp$ for any price vector $p$ with probability $1$. In other words, the problem is in general position almost surely.

We emphasize that Assumption~\ref{Assumption: DistributionDensity} only requires the conditional distribution of the reward given the resource consumption to be continuous, and the distribution of the resource consumption itself can be discrete. For example, let $a_1=1$ with probability $1$, $r_1$ follows $Uniform(0,1)$, and $d_0=0.5$. This example is the well known multi-secretary problem and satisfies Assumption~\ref{Assumption: DistributionDensity}. 

Under Assumption~\ref{Assumption: DistributionMulti} and Assumption~\ref{Assumption: DistributionDensity}, the regret of AhdLA is $\Theta(\log n)$ (\cite{Bray2019}). These assumptions also guarantee that we can apply the following useful properties of the offline and population dual prices from the literature. Let $\norm{\cdot}_{2}$ be the Euclidean norm in $\mathbb{R}^m$.
\begin{lem}\label{lem: BoundedLipschitz}
     (Lemma 1 of \cite{Bray2019}, Lemma 2 of \cite{Bray2019}, Lemma 12 of \cite{li2019olp}) There exists $\delta_d>0$ such that the population price $p^*(d)$ is unique whenever $d\in\Omega(\delta_d)\equiv\{d':\norm{d'-d_0}_\infty<\delta_d\}$. Furthermore, $p^*(\cdot)$ is Lipschitz continuous in $\Omega(\delta_d)$, i.e., there exists a constant $L$ such that,
    \begin{equation*}
        \norm{p^*(d)-p^*(d')}_2^2\leq L\norm{d-d'}_2^2\quad\forall d,d'\in\Omega(\delta_d).
    \end{equation*}
\end{lem}
\noindent The above assumptions have been used to derive dual convergence results in the literature, allowing us to give an upper bound on $\mathbb{E}\left[\norm{p_N^*(d)-p^*(d)}_2^2\right]$. \cite{li2019olp} show that $\mathbb{E}\left[\norm{p_N^*(d)-p^*(d)}_2^2\right]=O\left(\frac{\log\log N}{N}\right)$ for all $d$ in a neighborhood of $d_0$, which is a pointwise convergence result. \cite{Bray2019} improves the convergence rate and proves a uniform convergence result, which is summarized in Lemma~\ref{lem: UniformDualConvergence}.
\begin{lem}\label{lem: UniformDualConvergence}
    There exist constants $C_{Dual}$, $\delta_d>0$ such that $\Omega(\delta_d)\subseteq\otimes_{i=1}^m(\underline{d},\bar{d})$, and
    \begin{equation*}
        \mathbb{E}\left[\sup_{d\in\Omega(\delta_d)}\norm{p_N^*(d)-p^*(d)}_2^2\right]\leq\frac{C_{Dual}}{N}.
    \end{equation*}
\end{lem}
\noindent Note that we can and will assume the neighborhood $\Omega(\delta_d)$ in Lemma~\ref{lem: BoundedLipschitz} and Lemma~\ref{lem: UniformDualConvergence} to be the same. 
\section{Online Linear Programming with Batching}\label{sec: olp with delay} 
We start our discussion from a basic version of OLP with batching. We defer the discussion of several other variants of the problem to later sections. Let the total number of orders $n$ be known by the decision maker, and let $K$ be the total number of batches. We assume $n$ can be divided by $K$ and let $B=\frac{n}{K}$, where we call $B$ the batch size. In the batching operation, a feasible policy accepts or rejects the individual orders no later than the final order of the associated batch is received. In other words, let $\mathcal{F}_k$ be the sigma algebra generated by $\{(r_j,a_j):1\leq j\leq kB\}$, then for any feasible policy $\pi$, $x_j^{\pi}$ is measurable with respect to $\mathcal{F}_k$ if $(k-1)B<j\leq kB$. This ability to delay decisions on orders is the most important difference between OLP with batching and the standard OLP without batching. In addition to the measurability constraint, a feasible policy should also satisfy the budget constraint $\sum_{j=1}^na_jx_j^{\pi}\leq b_0$ almost surely. Together with the i.i.d. property stated in Assumption \ref{Assumption: DistributionMulti}, OLP with batching can be formulated as a Markov Decision Process (MDP) when the distribution of $(r_1,a_1)$ is known. The length of the horizon is the total number of batches $K$. The state is the average remaining resource $d$, and the action $\mathbf{x}_k=\left\{x_j\right\}_{j=kB+1}^{(k+1)B}$ is the acceptance or rejection decision for each of the orders in batch $k+1$. The Bellman equations are  
\begin{equation*}
    \begin{aligned}
        &V_k(d)=\mathbb{E}\left[\max_{\mathbf{x}_k\in X_k(d)}\sum_{j=kB+1}^{(k+1)B}r_jx_j+V_{k+1}\left(h_{k}(d,\mathbf{x}_k)\right)\right]\quad k=0,\cdots K-1\\
        &V_K(d) = 0
    \end{aligned}
\end{equation*}
where
\begin{equation*}
    \begin{aligned}
        &X_k(d)=\left\{\mathbf{x}_k:\sum_{j=kB+1}^{(k+1)B}a_jx_j\leq (K-k)Bd,\;\textbf{x}_k\in \{0,1\}^B\right\}\\
        &h_{k}(d,\mathbf{x}_k)=\frac{(K-k)Bd-\sum_{j=kB+1}^{(k+1)B}a_jx_j}{(K-k-1)B}.
    \end{aligned}
\end{equation*}
Solving the Bellman equations gives the optimal policy $\pi^*$ for OLP with batching.
\subsection{Regret Lower Bound}\label{Section: LB}
AhdLA is obviously a feasible policy for OLP with batching; however, it ignores the opportunity to obtain extra information by delaying decisions within each batch. Intuitively, an algorithm of OLP with batching that appropriately uses the extra information can have smaller regret than the $\Theta(\log n)$ achieved by AhdLA. Then, a natural question to ask is: what is a lower bound on regret when batching is permitted? Theorem \ref{thm: RegretLowerBound} answers this question by providing a regret lower bound in the case where only a single type of resource is present.
\begin{thm}\label{thm: RegretLowerBound}
    Under Assumptions \ref{Assumption: DistributionMulti} and \ref{Assumption: DistributionDensity}, if $m=1$, there exists constants $\underline{K}$ and $\underline{C}$ such that, when $K>\underline{K}$,
    \begin{equation*}
        Regret_n(\pi^*)\geq \underline{C}\log K
    \end{equation*}
    uniformly in $n$, where $\underline{K}$ and $\underline{C}$ only depend on the distribution of $(r_1,a_1)$ and the initial average resource $d_0$.
\end{thm}
We defer the proof of Theorem \ref{thm: RegretLowerBound} to Section \ref{Section: LowerBoundProof}. Although Theorem \ref{thm: RegretLowerBound} only provides a lower bound for the single resource case, one expects that the lower bound also holds when there are multiple types of resources, since the problem becomes harder. However, the proof of the lower bound in the multiple resource case is left for future work.

On the one hand, Theorem \ref{thm: RegretLowerBound} suggests the possibility of designing an algorithm having a regret that does not depend on the total number of orders $n$. Thus, the use of batching has potential for breaking the $\Omega(\log n)$ regret lower bound of the standard OLP without batching. For example, if $K=\log n$, then the regret would be $\Theta(\log\log n)$, provided we find an algorithm that achieves the lower bound of Theorem \ref{thm: RegretLowerBound}. However, Theorem \ref{thm: RegretLowerBound} also asserts that any algorithm achieving the lower bound will have regret $\Omega(\log n)$ when $B=n^{\alpha}$ for $0\leq \alpha<1$. To make regret $o(\log n)$, the batch size needs to be at least almost proportional to $n$. Of course, if the batch size is set to be proportional in $n$, the regret would be of constant order, but there may be then too much delays incurred, which may make batching not practical in some of real-world applications. However, we want to emphasize that the total delays are not fully decided by the batch size, because a feasible algorithm is not required to delay orders in every batch. Suppose the batch size is set to be one percent of $n$, and an algorithm only delays orders in the first and last batch. In such an implementation, the algorithm may still be practically interesting.

The above insights from Theorem \ref{thm: RegretLowerBound} bring our attention to two additional performance measurements that arise when we consider batching. The first one is \textit{customer patience}. In practice, the amount of time that a customer is willing to wait for a decision is limited. How patient the customers can be decides the maximum batch size the decision maker can select. We define the \textit{delay} of an order to be the number of new orders arrived before the decisions made on the order, and we define the customer patience of an order to be the maximum delay accepted. Then, for a given batch, if the decisions on every order of this batch are made together at the end of the batch, the batch size should be smaller than the customer patience of the first order. The second one is \textit{delay budget}, which is defined to be the total delays of the orders allowed. Even if all the customers are willing to wait until the end of the planning horizon, the decision maker may still want to reduce the total waiting time of the customers since it may bring some external benefit which is not counted in the rewards of the orders. Thus, there is usually a limited delay budget on the batching operation. For two algorithms with similar regrets, we may prefer the algorithm with lower requirements on customer patience and delay budget.

Under our assumptions on the joint distribution of the reward and the resource consumption, there are cases in which we need to delay decisions in all the batches when we implement the optimal policy $\pi^*$. Delaying decisions in all the batches means that we need the customer of the $j$th order of each batch to accept a delay equal to $B-j$ and we need a $\frac{KB(B-1)}{2}$ delay budget. Thus, although $\pi^*$ is the most hopeful algorithm to have $O(\log K)$ regret, it may have a very high requirement of customer patience and the delay budget. In addition, computing $\pi^*$ requires solving a MDP with continuous state space which is difficult in general. Thus, we want algorithms that has lower requirement on customer patience and delay budget, can be easily computed, and has $O(\log K)$ regret at the same time. In the next section, we propose algorithms that satisfy these requirements.
\subsection{Algorithms for OLP with Batching}\label{Section: MTR}
We propose algorithms achieving $O(\log K)$ regret in this section. Namely, there exists a constant $C_{\pi}$ that only depends on the joint distribution of $(r_1,a_1)$ and the initial average resource $d_0$ such that
\begin{equation*}
    Regret_n(\pi)\leq C_{\pi}(1+\log K)
\end{equation*}
In other words, the regret of each proposed algorithm is independent of the total number of orders $n$.

We consider two different settings. In one setting, we make weak assumptions on the distribution of $a_1$ but are given some information on the joint distribution of $(r_1,a_1)$ at the beginning of the planning horizon. In the other setting, we make stronger assumptions on the distribution of $a_1$ but are given no information of the joint distribution of $(r_1,a_1)$ at the beginning of the planning horizon.

\subsubsection{Better Knowledge and Weaker Assumption}
We assume we are given $B$ i.i.d. copies of $(r_1,a_1)$ at the beginning of the planning horizon. We require these $B$ i.i.d. copies and $\{(r_j,a_j)\}_{j=1}^n$ to be independent, and we can view them as the available historical data. The following re-solving algorithm of OLP with batching is inspired by the idea of AhdLA.
\;\\
\begin{breakablealgorithm}
    \caption{}\label{alg: AhdLAMulti}
    \begin{algorithmic}[1]
        \State Set $t_k=kB$ for $0\leq k\leq K$ and let $\left\{(r_j,a_j)\right\}_{j=-B+1}^{0}$ be the $B$ i.i.d. copies of $(r_1,a_1)$.
        \For{$k=1,\cdots,K-1$}
        , compute
        \begin{equation*}
            p_{k}=\arg\min_{p\geq 0}{\frac{{b_{t_{k-1}}}^{T}p}{n-t_{k-1}}}+\frac{1}{t_{k}}\sum_{j=-B+1}^{t_{k-1}}(\Delta_j(p))^+.
        \end{equation*}
            \For{$j=t_{k-1}+1,\cdots,t_k$} $x^{\pi}_j=\mathbbm{1}\left\{\Delta_j(p_k)>0\right\}\mathbbm{1}\left\{b_{j-1}\geq a_j\right\}$ and $b_j = b_{j-1}-a_jx^{\pi}_j$.
            \EndFor
        \EndFor
        \State Compute
        \begin{equation*}
            p_K = \arg\min_{p\geq 0} {b_{t_{K-1}}}^Tp+\sum_{j=t_{K-1}+1}^n(\Delta_j(p))^+
        \end{equation*}
        \For{$j=t_{K-1}+1,\cdots,n$} $x^{\pi}_j=\mathbbm{1}\left\{\Delta_j(p_K)>0\right\}$.
        \EndFor
    \end{algorithmic}
\end{breakablealgorithm}
\;\\
Since the computation of $p_k$ only uses the information known before batch $k$ for $k\leq K-1$, Algorithm \ref{alg: AhdLAMulti} only delays orders in the last batch, and the total delay budget required by Algorithm \ref{alg: AhdLAMulti} is $\frac{B(B-1)}{2}$, which is much smaller than the delay budget required for $\pi^*$ when $\pi^*$ needs to delay decisions in all the batches. In addition, the implementation of Algorithm \ref{alg: AhdLAMulti} only requires solving $K$ LPs, which can be easily done by using standard optimization software. Denoting Algorithm \ref{alg: AhdLAMulti} as $\pi_1$, Theorem \ref{Thm:RegretBoundDAMulti} provides a regret upper bound for $\pi_1$. 
\begin{thm}\label{Thm:RegretBoundDAMulti}
    Under Assumptions \ref{Assumption: DistributionMulti}, \ref{Assumption: DistributionDensity} , there exists a constant $C_1$ such that,
    \begin{equation*}
        Regret_n(\pi_1)\leq C_1(1+\log K)
    \end{equation*}
    for all $B,\;K\in\mathbb{N}_{+}$, where $C_1$ only depends on the distribution of $(r_1,a_1)$ and the initial average resource $d_0$.
\end{thm}
The proof of Theorem \ref{Thm:RegretBoundDAMulti} is deferred to Section \ref{Section: ProofUpperBound}. Compared with AhdLA, there are two factors that allow $\pi_1$ to reduce the regret from $O(\log n)$ to $O(\log K)$. The first is the availability of the $B$ i.i.d. copies of $(r_1,a_1)$. Thus, the dual price used by $\pi_1$ in the first batch is much closer to the optimal offline dual price than the dual prices used by AhdLA. The second is the ability to delay decisions allows $\pi_1$ to allocate resource in the last batch in an offline fashion. Thus, $\pi_1$ utilizes the resources in a more efficient way than AhdLA by having almost no resource left at the end of the planning horizon. Apart from the difference in regret, Algorithm \ref{alg: AhdLAMulti} re-solves and updates the prices much less frequently than AhdLA, which can be viewed as an additional benefit of Algorithm \ref{alg: AhdLAMulti}. 
\subsubsection{Unknown Distribution but Stronger Assumption}\label{Section: SimulatorNotNecessary}
When the distribution of $(r_1,a_1)$ is completely unknown at the beginning of the planning horizon, we can modify Algorithm \ref{alg: AhdLAMulti} by replacing $p_1$ with 
\begin{equation}\label{eqn: Alg1Modify1}
    \tilde{p}_{1}=\arg\min_{p\geq 0} {\frac{{b_0}^Tp}{n}}+\frac{1}{t_1} \sum_{j=1}^{t_1}(\Delta_j(p))^+
\end{equation}
and $p_k$ with
\begin{equation}\label{eqn: Alg1Modify2}
    \tilde{p}_{k}=\arg\min_{p\geq 0} {\frac{{b_{t_{k-1}}}^Tp}{n-t_{k-1}}}+\frac{1}{t_{k-1}} \sum_{j=1}^{t_{k-1}}(\Delta_j(p))^+
\end{equation}
when $2\leq k\leq K-1$. This modification replaces the data provided at the beginning of the planning horizon in Algorithm \ref{alg: AhdLAMulti} by the information related to orders in the first batch. A technical difficulty on the regret analysis brings by this modification is that $\tilde p_1$ and $\{(r_j,a_j)\}_{j=1}^{t_1}$ are not independent. We propose a \textit{leave-some-out} analysis to deal with it, and the details are deferred to Section \ref{Section: LOO}. Our \textit{leave-some-out} analysis requires the following extra assumption on the distribution of $a_1$.
\begin{assumption}{3}\label{Assumption: LOO}
    There exists a set of linear independent vectors $\{\alpha_i\}_{i=1}^I$ such that every point in the support of $a_1$ equals to one of these $I$ vectors multiplied by a scalar.
\end{assumption}
Assumption \ref{Assumption: LOO} means that there are only $I$ types of products and $\alpha_i$ is the unit resource consumption of a type $i$ product. Each order demands only a single type of product, but the amount of a certain type of product required varies among orders that demands that type of product. In addition, Assumption \ref{Assumption: LOO} is only about the distribution of the resource consumption and has no conflict with Assumption \ref{Assumption: DistributionMulti} and Assumption \ref{Assumption: DistributionDensity}. Denoting the modified algorithm as $\tilde\pi_1$, Proposition \ref{Prop:RegretBoundDA_2} states that the $O(\log K)$ regret upper bound can be extended to the case that the joint distribution of $(r_1,a_1)$ is completely unknown at least when Assumption \ref{Assumption: LOO} holds.
\begin{prop}\label{Prop:RegretBoundDA_2}
    Under Assumptions \ref{Assumption: DistributionMulti}, \ref{Assumption: DistributionDensity}, \ref{Assumption: LOO}, there exists a constant $\tilde C_1$ such that,
    \begin{equation*}
        Regret_n(\tilde\pi_1)\leq \tilde C_1(1+\log K)
    \end{equation*}
    for all $B,\;K\in\mathbb{N}_{+}$, where $\tilde C_1$ is a constant that only depends on the distribution of $(r_1,a_1)$ and the initial average resource $d_0$.
\end{prop}
When there is only a single type of resource, Assumption \ref{Assumption: LOO} automatically holds. However, when there are multiple types of resource, the linear independence assumption implies that the number of the types of product cannot be greater than the number of the types of resource. We believe that the $O(\log K)$ regret upper bound still holds when Assumption \ref{Assumption: LOO} is relaxed.

We want to emphasize that, although having information on the joint distribution of $(r_1,a_1)$ at the beginning of the planning horizon is not necessary for having $O(\log K)$ regret, knowing the distribution of $(r_1,a_1)$ well before making any allocation in the first batch is essential. Thus, $\tilde\pi_1$ additionally delays orders in the first batch and requires a $B(B-1)$ delay budget which is higher than the delay budget required by $\pi_1$.
\subsubsection{Infrequent Re-solving and Updating Through Batching}
We say a dual-based algorithm is fully adapted to the history if it re-evaluates the dual price after every order received, for example, AhdLA. Algorithm \ref{alg: AhdLAMulti} is not fully adapted to the history since it re-solves the dual price infrequently, and Theorem \ref{Thm:RegretBoundDAMulti} implies that being fully adapted to the history is not necessary for achieving $O(\log K)$ regret. In fact, if two dual-based algorithms have similar regret, we prefer the algorithm that re-solves and updates dual prices less frequently because it brings managerial benefits. Updating the dual prices less frequently can reduce the required communications within the organization and avoid costs associated with miscommunication. Furthermore, an algorithm that updates dual prices frequently may create confusion among the customers. For example, in AhdLA, it may accept order $j$ but reject order $j+1$ even if order $j+1$ is strictly better in the sense that $r_{j+1}>r_j$ and $a_{j+1}\leq a_j$. An algorithm that infrequently updates the dual prices, such as Algorithm \ref{alg: AhdLAMulti}, limites the possibility that this situation occurs frequently. The stability and the consistency in decision-making related to infrequent re-solving and updating algorithms such as Algorithm \ref{alg: AhdLAMulti} make them practically interesting.   

By observing the importance of infrequent re-solving and updating, a natural question to ask is: How does the frequency of re-solving and updating dual prices affect the regret of a dual-based algorithm in OLP with batching? In the following, we consider the extreme case in which no re-solving is done. Let $\left\{(r_j,a_j)\right\}_{j=-n+1}^0$ be $n$ i.i.d. copies of $(r_1,a_1)$. Define $\tilde p=\arg\min_{p\geq 0} {d_0}^Tp + \frac{1}{n}\sum_{j=-n+1}^0(\Delta_j(p))^+$ and let $\hat\pi_1$ be the algorithm that modifies Algorithm \ref{alg: AhdLAMulti} by replacing $p_k$ with $\tilde p$ for all $1\leq k\leq K-1$. Proposition \ref{Proposition: Infrequent1} gives a regret upper bound for $\hat\pi_1$ and provides some insights into the above question. The complete proof of Proposition \ref{Proposition: Infrequent1} is in the appendix. 
\begin{prop}\label{Proposition: Infrequent1}
   Under Assumptions \ref{Assumption: DistributionMulti}, \ref{Assumption: DistributionDensity}, and \ref{Assumption: LOO}, there exists constants $\hat C_1$ such that, 
    \begin{equation*}
        Regret_n(\hat\pi_1)\leq \hat C_1\min\left\{K,\sqrt n\right\}
    \end{equation*}
    for all $B,\;K\in\mathbb{N}_{+}$, where $\hat C_1$ only depends on the distribution of $(r_1,a_1)$ and the initial average resource $d_0$.
\end{prop}
\begin{remark}
    (i) Without Assumption \ref{Assumption: LOO}, $Regret_n(\hat\pi_1)\leq \hat C_1 K^2$. The proof is almost the same as the proof of Theorem \ref{Thm:RegretBoundDAMulti} and is omitted. (ii) The proof of Proposition \ref{Proposition: Infrequent1} holds as long as the sample size of $\tilde p$ is $\alpha n$ for some $\alpha>0$.
\end{remark}
Proposition \ref{Proposition: Infrequent1} shows that, when $B$ is $\Omega(\sqrt{n})$, the regret of $\hat\pi_1$ is $O(K)$, and when $B$ is $O(\sqrt{n})$, the regret of $\hat\pi_1$ is $O(\sqrt{n})$. In particular, when $K$ is constant in the number of orders $n$, in other words, the batch size $B$ is proportional to the number of orders $n$, a constant regret can be achieved without any re-solving and updating. In fact, the idea of infrequent re-solving and updating need not be limited to dual-based algorithms. Let $\Pi$ be the set of all mappings from the support of $(r_1,a_1)$ to $\{0,1\}$. An open question to ask is: how frequently should we reselect an allocation rule from $\Pi$ based on revealed information to achieve $O(\log K)$ regret in OLP with batching? 
\subsubsection{Selection of $K$}
As established, our regret bounds depend exclusively on the number of batches, $K$. Consequently, $K$ serves as the key parameter of the batching operation, making its selection a higher-level managerial decision. If the delay budget is fixed and all batches have the same size, the decision is straightforward. We simply select the smallest $K$ that does not violate the delay budget, because a smaller $K$ yields a lower regret. We also consider settings where the delay budget is fixed but individual batch sizes can vary. In this scenario, the delay budget can be viewed as an additional resource. The decision-maker must jointly allocate the delay budget and the actual resources over the planning horizon, without violating customer patience constraints. To systematically solve this problem, we provide an MDP formulation in the appendix. Moving beyond a fixed delay budget, we can also shift our perspective to view the delay budget itself as a strategic managerial decision. This shift introduces a fundamental trade-off. A smaller $K$ reduces regret but requires a larger delay budget, which inherently carries a higher implicit business cost. In practice, the decision-maker can first calibrate a cost function that maps different delay budget decisions to their corresponding costs. Our quantitative logarithmic regret bounds then enable decision-makers to formally weigh this cost function directly against the regret. Therefore, the decision-maker can ultimately select the optimal $K$ by minimizing the total economic cost, which consists of the regret and the cost of the delay budget, while ensuring that the limits of customer patience are strictly respected.

\subsection{Poisson Arrival Process}\label{Section: PAP}
All the analyses above is based on the assumption that the total number of orders $n$ is known. Although this is a common assumption made in the literature of OLP, $n$ can be unknown in many real world applications. In this section, we consider a setting in which $n$ is random. 

Let orders arrive following a Poisson process $(N(t):0\leq t\leq T)$ with arrival rate $\lambda$. We still denote the total number of batches as $K$, and let $B=\frac{T}{K}$ be the batch size. As the total number of orders $N(T)$ is random, we now assume the initial total resource $b_0$ is proportional to the expected total number of orders, i.e., $b_0=\lambda Td_0$ with $d_0$ satisfies the requirement in Assumption \ref{Assumption: DistributionMulti}(b). We also modify the definition of regret to
\begin{equation*}
    Regret_T(\pi) = \mathbb{E}\left[R_{N(T)}^*-R_{N(T)}(\pi)\right].
\end{equation*}
Let $\mathcal{F}_k=\sigma\left(\{(r_j,a_j)\}_{j=1}^{N(kB)}\right)$. For a feasible policy $\pi$, we require $x_{N(t)}^{\pi}$ to be measurable with respect to $\mathcal{F}_k$ if $(k-1)B<t\leq kB$, and $\sum_{j=1}^{N(T)}a_jx_j^{\pi}\leq b_0$ almost surely. We propose the following re-solving algorithms for this Poisson arrival setting, which does not require the arrival rate $\lambda$ and the joint distribution of $(r_1,a_1)$ to be known.\\
\begin{breakablealgorithm}
    \caption{}\label{alg: RaAhdLA}
    \begin{algorithmic}[1]
        \State Set $\hat{\lambda}_0=\frac{\max\{N(t_1),1\}}{t_1}$, $\tilde d_0=\frac{b_0}{\hat{\lambda}_0T}$, and compute
        \begin{equation*}
            p_1 = \arg\min_{p\geq 0} N(t_1){\Tilde{d}_0}^Tp+\sum_{j=1}^{N(t_1)}\left(\Delta_j(p)\right)^+.
        \end{equation*}
        \For{$j=1,\cdots,N(t_1)$} $x^{\pi}_j=\mathbbm{1}\left\{\Delta_j(p_1)>0\right\}\mathbbm{1}\left\{b_{j-1}\geq a_j\right\}$ and $b_j = b_{j-1}-a_jx^{\pi}_j$.
        \EndFor
        \For{$k=2,\cdots,K-1$}
             , set $\hat{\lambda}_{k-1}=\frac{\max\{N(t_{k-1}),1\}}{t_{k-1}}$, $\Tilde{d}_{k-1}=\frac{b_{N(t_{k-1})}}{\hat{\lambda}_{k-1}(T-t_{k-1})}$, and compute
            \begin{equation*}
                p_k = \arg\min_{p\geq 0} N(t_{k-1}){\Tilde{d}_{k-1}}^Tp+\sum_{j=1}^{N(t_{k-1})}\left(\Delta_j(p)\right)^+.
            \end{equation*}
            \For{$j=N(t_{k-1})+1,\cdots,N(t_{k})$} $x^{\pi}_j=\mathbbm{1}\left\{\Delta_j(p_k)>0\right\}\mathbbm{1}\left\{b_{j-1}\geq a_j\right\}$ and $b_j = b_{j-1}-a_jx^{\pi}_j$.
            \EndFor
        \EndFor
        \State Compute
        \begin{equation*}
            p_K=\arg\min_{p\geq 0} {b_{N(t_{K-1})}}^Tp+\sum_{j=N(t_{K-1})+1}^{N(T)}(\Delta_j(p))^+
        \end{equation*}
        \For{$j=N(t_{K-1})+1,\cdots,N(T)$} $x^{\pi}_j=\mathbbm{1}\left\{\Delta_j(p_K)>0\right\}$.
        \EndFor
    \end{algorithmic}
\end{breakablealgorithm}
\;\\
To simplify the regret analysis, we make the following independence assumption.
\begin{assumption}{4}\label{Assumption: Random Arrival}
    $(N(t):t\geq 0)$ is independent of $\left\{(r_j,a_j)\right\}_{j=1}^{+\infty}$.
\end{assumption}
With Assumption \ref{Assumption: Random Arrival}, we can generalize the techniques used to analyze the regret of the known $n$ case to the regret analysis of Algorithm \ref{alg: RaAhdLA}, but we need additional steps to carefully deal with the issue of only having a noisy estimate $\tilde d_k$ of the average remaining resource. Denote Algorithm \ref{alg: RaAhdLA} as $\pi_2$, and Theorem \ref{Thm: RegretPA} states an upper bound of the regret of Algorithm \ref{alg: RaAhdLA}. Details of the proof is provided in the appendix.
\begin{thm}\label{Thm: RegretPA}
    Under Assumptions \ref{Assumption: DistributionMulti}, \ref{Assumption: DistributionDensity}, \ref{Assumption: LOO}, \ref{Assumption: Random Arrival}, there exists a constant $C_2$ such that,
    \begin{equation*}
        Regret_T(\pi_2)\leq C_2(1+\log K)
    \end{equation*}
    for all $\lambda,\;T\in\mathbb{R}_{+}$ and $K\in\mathbb{N}_+$ where the constant $C_2$ only depend on the distribution of $(r_1,a_1)$ and the initial average resource $d_0$.
\end{thm}
\begin{remark}
    If we are available to an independent and identical copy of $\{(r_j,a_j)\}_{j=1}^{N(t_1)}$ at the beginning of the planning horizon, then we can modify Algorithm \ref{alg: RaAhdLA} by computing $\hat \lambda_0$ and $p_1$ using this independent and identical copy, and the regret upper bound in Theorem \ref{Thm: RegretPA} still holds for the modified algorithm without Assumption \ref{Assumption: LOO}.
\end{remark}
Theorem \ref{Thm: RegretPA} implies that, under appropriate assumptions, switching the setting from known total number of orders to random total number of orders does not worsen the order of the regret of OLP with batching. The regret upper bound still only depends on the total number of batches $K$. 
\subsection{Large Support}
We have so far focused on the continuous support setting. A common criticism of this assumption is that no real-world transaction system can differentiate between an uncountably infinite number of order types. In this section, we argue that the continuous support setting serves as an appropriate theoretical surrogate for a more realistic large support environment. In practical applications, the number of possible order types is indeed finite; i.e, the support of the joint distribution of $(r_1,a_1)$ is finite. However, the parameter $n$ typically reflects the overall scale of the business. Under the observe-then-decide structure of online linear programming, where customers propose the rewards for each order, a larger-scale business naturally encounters a vast variety of customer types. Therefore, the standard finite support theory, which assumes a small and fixed number of product types, fails to adequately describe systems where both the number of distinct order types and the total number of orders $n$ are large. To formally capture this large support setting, we analyze an asymptotic regime where the support size is finite for any fixed $n$, but grows to infinity as $n \to \infty$. 

Specifically, we consider a multisecretary problem formulated using the OLP notation in Section \ref{Section: Preliminary}. Let the resource consumptions $a_j = 1$ with probability $1$ for all $j$, and rewards $\{r_{j}(n)\}_{j=1}^{n}$ be i.i.d. random variables such that
\begin{equation*}
    \mathbb{P}\left\{r_1(n)=\frac{s-0.5}{\lfloor n^{\alpha}\rfloor}\right\}=\frac{1}{\lfloor n^{\alpha}\rfloor}\quad\forall 1\leq s\leq \lfloor n^{\alpha}\rfloor
\end{equation*}
where $\alpha>0$. Thus, the number of order types $S$ is $\lfloor n^{\alpha}\rfloor$. Note that we set $S=\lfloor n^{\alpha}\rfloor$ for simplicity and our theory works as long as $S$ is an integer and of the order $\Theta(n^{\alpha})$. In addition, we set the initial average resource $d_0$ to satisfy the continuous support setting assumptions (Assumption \ref{Assumption: DistributionMulti} and Assumption \ref{Assumption: DistributionDensity}) when the reward follows $Uniform(0,1)$ and the resource consumption is $1$ almost surely. We assume the total number of orders $n$ is known so that we consider the batching operation introduced at the beginning of Section \ref{sec: olp with delay}. For a fixed $n$, we denote the type of order $j$ as $\theta_j$, and $\theta_j=s$ if $r_j=\frac{s-0.5}{\lfloor n^{\alpha}\rfloor}$. We assume the total number of order types $S$ and the rewards of each type of order $\{\mu_s\}_{s=1}^S$ are known, but the probability of each order type is unknown. The following Algorithm \ref{alg: finite_conti} is the batching algorithm for the finite support case using randomized tie-breaking inspired by \cite{Chen.et.al.2023}. 
\begin{breakablealgorithm}
    \caption{}
    \label{alg: finite_conti}
    \begin{algorithmic}[1]
        \State Set $t_k=kB$ for $0\leq k\leq K$, let $\left\{r_j\right\}_{j=-B+1}^{0}$ be the $B$ i.i.d. copies of $r_1$, and let $\{\bar U_j\}_{j=1}^n$ be an independent sequence of i.i.d. $Uniform(0,1)$ random variables.
        \For{$k=1,\cdots,K-1$}
        \State Set $N_s(k)=\sum_{j=-B+1}^{t_{k-1}}\mathbbm{1}\{\theta_j=s\}\;1\leq s\leq S$.
        \State Set $s_k=\min\{s\in\mathbb{N}_+:\sum_{l=1}^sN_{l}(k)>t_k-(t_kb_{t_{k-1}})/(n-t_{k-1})\}$
        \State Set $\eta_k =\frac{t_kb_{t_{k-1}}}{n-t_{k-1}} - \sum_{s=s_k+1}^S N_s(k)$ and $q_{s,k}=\mathbbm{1}\{s>s_k\} + \frac{\eta_k}{N_s(k)}\mathbbm{1}\{s=s_k\}\;\forall 1\leq s\leq S$
            \For{$j=t_{k-1}+1,\cdots,t_k$}
            \State $x^{\pi}_j=\mathbbm{1}\{\bar U_j>1-q_{\theta_j,k}\}\mathbbm{1}\{b_{j-1}\geq 1\}$ and $b_j = b_{j-1}-x^{\pi}_j$ 
            \EndFor
        \EndFor
        \State Set $\{x^{\pi}_j\}_{j=t_{K-1}+1}^n=\arg\max\left\{\sum_{j=t_{K-1}+1}^nr_jx_j|\sum_{j=t_{K-1}+1}^nx_j \leq b_{t_{K-1}}\;\mathrm{and}\;x_j\in \{0,1\}\;\forall t_{K-1}<j\leq n\right\}$
    \end{algorithmic}
\end{breakablealgorithm}
Denote Algorithm \ref{alg: finite_conti} as $\pi_3$. The following proposition summarizes our result of the regret of the specific multisecretary problem described above. The proof is given in the appendix.
\begin{prop}\label{Proposition: Growing Support}
    Under the above setup of the multisecretary problem, when $\alpha\geq 0.5$,
    \begin{enumerate}
        \item[(a)] There exists constants $\underline{K}'$ and $\underline{C}'$ such that, when $K>\underline{K}'$,
        \begin{equation*}
          Regret_n(\pi^*)\geq \underline{C}'\log K  
        \end{equation*}
        uniformly in $n$, where $\underline{K}'$ and $\underline{C}'$ only depend on the initial average resource $d_0$.
    \item[(b)] There exists a constant $C_3$ such that
    \begin{equation*}
      Regret_n(\pi_3)\leq C_3(1+\log K)  
    \end{equation*}
    for all $B,\;K\in\mathbb{N}_{+}$, where $C_3$ only depends on the initial average resource $d_0$.
    \end{enumerate}
\end{prop}
\section{Proofs}\label{Section: Proofs}
\subsection{Proof of Theorem \ref{thm: RegretLowerBound}}\label{Section: LowerBoundProof}
Define $R_{n-kB}^*(d)$ and $p_{kB:n}(d)$ as the optimal objective value and the optimal solution of
\begin{equation}\label{eqn: RecedeHorizonOfflineDual}
    \min_{p\geq 0}\;(n-kB)d^Tp + \sum_{j=kB+1}^n\left(\Delta_j(p)\right)^+
\end{equation}
and $\{x^*_{k+1,j}(d)\}_{j=kB+1}^n$ be the optimal solution of the corresponding primal LP of (\ref{eqn: RecedeHorizonOfflineDual}). Let
\begin{equation*}
d_{k+1}^*(d) = \frac{(n-kB)d-\sum_{j=t_{k}+1}^{t_{k+1}}a_jx_{k+1,j}^*(d)}{n-(k+1)B},    
\end{equation*}
then $d_{k+1}^*(d)$ is the average remaining resource at the end of batch $k+1$ if the average remaining resource at beginning of batch $k+1$ is $d$ and the resource is allocated in an offline fashion in the last $K-k$ batches. 

Denote $x^{\pi^*}_j(d)$ as the action taken by $\pi^*$ when the average remaining resource at beginning of the associated batch or order $j$ is $d$. Let $d^{\pi^*}_k$ be the average remaining resource at the beginning of batch $k+1$ when we apply $\pi^*$ to OLP with batching. Note that $d^{\pi^*}_0=d_0$. To simplify the notation, we denote $t_k=kB$ and write $d^{\pi^*}_k$ as $d_k$ in the rest of this subsection. Let
\begin{equation*}
    R_{n-kB}(\pi^*) = \sum_{l=k+1}^{K}\sum_{j=t_{l-1}+1}^{t_{l}}r_jx_j^{\pi^*}(d_{l-1})
\end{equation*} 
be the total reward earned by $\pi^*$ in the last $K-k$ batches, then
\begin{equation*}
    \begin{aligned}
        R_n^*-R_n(\pi^*)&=\sum_{j=1}^{t_1}r_j(x_{1,j}^*(d_0)-x_j^{\pi^*}(d_0))+R_{n-B}^*\left(d_1^*(d_0)\right)-R_{n-B}^*\left(d_1\right)+R_{n-B}^*\left(d_1\right)-R_{n-B}(\pi^*)
    \end{aligned}
\end{equation*}
As $p_{B:n}^*(d_1^*(d_0))\geq 0$, it is feasible to the optimization problem defining $R_{n-B}^*\left(d_1\right)$. Thus, 
\begin{equation*}
    \begin{aligned}
        R_{n-B}^*\left(d_1^*(d_0)\right)-R_{n-B}^*\left(d_1\right)\geq& (n-B)(d_1^*(d_0)-d_1)p_{B:n}^*(d_1^*(d_0))\\
        =&\sum_{j=1}^{t_1}(x_j^{\pi^*}(d_0)-x_{1,j}^*(d_0))a_jp_{B:n}^*(d_1^*(d_0))
    \end{aligned}
\end{equation*}
Define
\begin{equation*}
\begin{aligned}
    u^{\pi^*}_{j,k}(d)=&\Delta_j\left(p^*_{kB:n}(d_k^*(d))\right)\left(\mathbbm{1}\left\{\Delta_j\left(p^*_{kB:n}(d_k^*(d))\right)>0\right\}-x_{j}^{\pi^*}(d_0)\right)
\end{aligned}    
\end{equation*}
By the definition of $d_1^*(d_0)$, we can select $p^*_{B:n}(d_1^*(d_0))=p^*_{0:n}(d_0)$. Together with the dual decision rule (\ref{eqn: ComplementarySlackness}), $u^{\pi^*}_{j,1}(d_0)=\Delta_j\left(p^*_{B:n}(d_1^*(d_0))\right)\left(x^*_{1,j}(d_0)-x_{j}^{\pi^*}(d)\right)$. Then,
\begin{equation*}
    \begin{aligned}
        R_n^*-R_n(\pi^*)\geq \sum_{j=1}^{t_1}u^{\pi^*}_{j,1}(d_0)+R_{n-B}^*\left(d_1\right)-R_{n-B}(\pi^*)
    \end{aligned}
\end{equation*}
Repeat the above argument for the first $K-1$ batches and use the fact that $R^*_{n-{K-1}B}(d_{K-1})\geq R_{n-(K-1)B}(\pi^*)$ with probability $1$, we get the following regret lower bound,
\begin{equation*}
    Regret_n(\pi^*)\geq\mathbb{E}\left[\sum_{k=0}^{K-2}\sum_{j=t_k+1}^{t_{k+1}}u^{\pi^*}_{j,k+1}(d_{k})\right]
\end{equation*}
For $\delta>0$, let $\Omega(\delta)=[d_0-\delta,d_0+\delta]$. Then,
\begin{equation}\label{eqn: LBDecomp}
    \begin{aligned}
        &\mathbb{E}\left[\sum_{k=0}^{K-2}\sum_{j=t_k+1}^{t_{k+1}}u^{\pi^*}_{j,k+1}(d_k)\right]\\
        \geq &\mathbb{E}\left[\sum_{k=0}^{K-2}\sum_{j=t_k+1}^{t_{k+1}}\left(\mathbb{E}\left[u^{\pi^*}_{j,k+1}(d_k)\bigg\rvert d_k\right]\mathbbm{1}\{d_k\in\Omega(\delta)\}\right)\right]\\
        \geq& \sum_{k=0}^{K-2}\left(\mathbb{P}\left\{d_k\in\Omega(\delta)\right\}\sum_{j=t_k+1}^{t_{k+1}}\inf_{d\in\Omega(\delta)}\mathbb{E}\left[u^{\pi^*}_{j,k+1}(d)\right]\right)
    \end{aligned}
\end{equation}
The first inequality is by the fact that $u^{\pi^*}_{j,k}(d)\geq 0$ with probability $1$ for all $d$, which is because the signs of $\Delta_j\left(p^*_{kB:n}(d_k^*(d))\right)$ and $\mathbbm{1}\left\{\Delta_j\left(p^*_{kB:n}(d_k^*(d))\right)>0\right\}-x_{j}^{\pi^*}(d)$ are always the same. The second inequality is by the independence between $d_k$ and $\left\{(r_j,a_j)\right\}_{j=t_{k}+1}^n$. To continue the proof, we need the following two Lemmas. 
\begin{lem}\label{Lemma: OptimalRS}
    For all $\delta>0$, there a constant $\bar N$ such that, when $n\geq\bar N$, if $n^{\frac{3}{4}}\leq (K-k)B\leq \frac{n}{2}$,
    \begin{equation*}
        \mathbb{P}\{d_k\in [d_0-\delta,d_0+\delta]\}\geq \frac{1}{4}
    \end{equation*}
\end{lem}
\begin{lem}\label{Lemma: LBMyopic}
    There exists constants $\bar{\delta}$ such that for all $\delta\in(0,\bar\delta)$, there exists constants $k_{LB}$, and $C_{LB}$ such that, if $K-k>k_{LB}$,
    \begin{equation*}
        \inf_{d\in\Omega(\delta)}\mathbb{E}\left[u^{\pi^*}_{j,k+1}(d)\right]\geq\frac{C_{LB}}{(K-k-1)B}
    \end{equation*}
\end{lem}
Lemma~\ref{Lemma: OptimalRS} provides a lower bound of $\mathbb{P}\left\{d_k\in\Omega(\delta)\right\}$ for $\frac{K}{2}\leq k\leq K-K^{\frac{3}{4}}$, and Lemma~\ref{Lemma: LBMyopic} provides a lower bound of $\inf_{d\in\Omega(\delta)}\mathbb{E}\left[u^{\pi^*}_{j,k}(d)\right]$ when $\delta$ is small enough. Proofs of these two Lemmas are in the appendix. Fix $0<\delta<\bar{\delta}$, select $\underline{K}\geq\max\{ k_{LB}^{\frac{4}{3}}+1,\bar{N},32\}$, then
\begin{equation*}
    \begin{aligned}
        Regret_n(\pi^*)\geq \frac{C_{LB}}{4}\sum_{k=\frac{K}{2}}^{K-\left\lceil K^{\frac{3}{4}}\right\rceil}\frac{1}{K-k-1}\geq \frac{C_{LB}}{80}\log K\equiv \underline{C}\log K
    \end{aligned}
\end{equation*}
\subsection{Proof of Theorem \ref{Thm:RegretBoundDAMulti}}\label{Section: ProofUpperBound}
Let $d_{k}^{\pi_1}$ be the average remaining resource at the beginning of batch $k+1$ when we apply Algorithm \ref{alg: AhdLAMulti} to OLP with batching. Note that $d_0^{\pi_1}=d_0$. For simplicity of the notation, we abuse the notation $d_k$ to denote $d_{k}^{\pi_1}$ instead of $d_k^{\pi^*}$. Recall $\Omega(\delta_d)$ is the neighborhood of $d_0$ stated in Lemma \ref{lem: BoundedLipschitz} and Lemma \ref{lem: UniformDualConvergence}. Define
\begin{equation*}
    \bar{\kappa}=\inf\left\{1\leq k\leq K-1:d_k\notin\Omega(\delta_d) \right\}\wedge K.
\end{equation*}
$\bar \kappa$ is a stopping time adapt to $\{\mathcal{F}_k\}_{k=1}^K$, and when $k<\bar\kappa$, Lemma \ref{lem: BoundedLipschitz} and Lemma \ref{lem: UniformDualConvergence} can be applied to the price evaluated at $d_k$. Let $\textbf{x}_k$ and $X_k$ be defined the same as in Section \ref{sec: olp with delay} and $\tilde X_k$ be the linear relaxation of $X_k$, i.e, replace the integer constraints $\textbf{x}_k\in\{0,1\}^B$ with $\textbf{x}_k\in [0,1]^B$. Let $t_k$, $R_{n-kB}^*(d)$ and $p_{kB:n}(d)$ be defined the same as in Section \ref{Section: LowerBoundProof}, and let $x_j^{\pi_1}$ be the action taken by Algorithm \ref{alg: AhdLAMulti} on order $j$ and $R_{n-j}(\pi)=\sum_{t=n-j+1}^nr_jx_j^{\pi_1}$. Then,
\begin{equation*}
    \begin{aligned}
        R_n^*-R_n(\pi_1)=&\max_{\mathbf{x}_0\in \tilde X_0(d_0)}\sum_{j=1}^{t_1}r_j(x_j-x_j^{\pi_1})+ R_{n-B}^*\left(h_0(d_0,\mathbf{x}_0)\right)-R_{n-B}^*(d_1)+R_{n-B}^*(d_1)-R_{n-B}(\pi_1).
    \end{aligned}
\end{equation*}
Since $p_{B:n}^*(d_1)$ is a feasible solution for evaluating $R_{n-B}^*\left(h_0(d_0,\mathbf{x})\right)$ for any $\mathbf{x}_0\in \tilde X_0(d_0)$, then
\begin{equation*}
    \begin{aligned}
        R_{n-B}^*\left(h_0(d_0,\mathbf{x}_0)\right)-R_{n-B}^*(d_1)\leq&\sum_{j=1}^{t_1}(x_j^{\pi_1}-x_j){a_j}^Tp_{B:n}^*(d_1).
    \end{aligned}
\end{equation*}
Thus,
\begin{equation*}
    \begin{aligned}
        R_n^*-R_n(\pi_1)\leq&\max_{\mathbf{x}\in [0,1]^B}\sum_{j=1}^{t_1}\Delta_j\left(p_{B:n}^*(d_1)\right)(x_j-x_j^{\pi_1})+R_{n-B}^*(d_1)-R_{n-B}(\pi_1)\\
        =&\sum_{j=1}^{t_1}\Delta_j\left(p_{B:n}^*(d_1)\right)\left(\mathbbm{1}\left\{\Delta_j\left(p_{B:n}^*(d_1)\right)>0\right\}-x_j^{\pi_1}\right)+R_{n-B}^*(d_1)-R_{n-B}(\pi_1).
    \end{aligned}
\end{equation*}
Repeat the above argument for the rest of the batches, Assumption~\ref{Assumption: DistributionMulti}(b),
\begin{equation}\label{eqn: GeneralRegretUpperBound}
    \begin{aligned}
        Regret_n(\pi_1)\leq& \mathbb{E}\left[\sum_{k=1}^{K-1}\sum_{j=t_{k-1}+1}^{t_k}\Delta_j\left(p_{kB:n}^*(d_k)\right)\left(\mathbbm{1}\left\{\Delta_j\left(p_{kB:n}^*(d_k)\right)>0\right\}-x_j^{\pi_1}\right)\right]\\
        &+\mathbb{E}\left[R_{n-(K-1)B}^*(d_{K-1})-R_{n-(K-1)B}(\pi_1)\right]
    \end{aligned} 
\end{equation}
We start from the case $B>\frac{\bar a}{\underline{d}}$. Then, if $d_k\in\Omega(\delta_d)$, the remaining resource of each type at the end of batch $k$ is greater than $\bar ae$, which means there is enough resource to implement the dual decision rule for orders in batch $1,\cdots,k$. Then, $x_j^{\pi_1}=\mathbbm{1}\{\Delta_j(p_k)>0\}$ when $j\leq t_{\bar\kappa-1}$. When $j>t_{\bar\kappa-1}$, by Assumption \ref{Assumption: DistributionMulti}(a) and (\ref{eqn: Lemma1(a)}), $\Delta_j(p^*_{kB:n})\leq \bar r+\frac{\bar r\bar a}{\underline{a}}$ almost surely. Thus, 
\begin{equation*}
    \begin{aligned}
        &Regret_n(\pi_1)\\
        \leq &\mathbb{E}\left[R_{n-(K-1)B}^*(d_{K-1})-R_{n-(K-1)B}(\pi_1)\right]+\left(\bar{r}+\frac{\bar{r}\bar{a}}{\underline{a}}\right)\mathbb{E}\left[t_{K-1}-t_{\bar{\kappa}-1}\right]\\
        &+\sum_{k=1}^{K-1}\sum_{j=t_{k-1}+1}^{t_k}\mathbb{E}\left[\mathbbm{1}\{k<\bar\kappa\}\Delta_j\left(p_{kB:n}^*(d_k)\right)\left(\mathbbm{1}\left\{\Delta_j\left(p_{kB:n}^*(d_k)\right)>0\right\}-\mathbbm{1}\{\Delta_j(p_k)>0\}\right)\right].        
    \end{aligned}
\end{equation*}
First, under Assumption \ref{Assumption: DistributionDensity}, $\abs{\left\{j:t_{K-1}<j\leq n\;\mathrm{and}\;\Delta_j(p^*_{(K-1)B:n})=0\right\}}\leq m$ almost surely. Thus, 
\begin{equation}\label{eqn: RG1}
    \mathbb{E}\left[R_{n-(K-1)B}^*(d_{K-1})-R_{n-(K-1)B}(\pi_1)\right]\leq m\bar r
\end{equation}
Secondly, to analyze $\mathbb{E}\left[t_{K-1}-t_{\bar\kappa-1}\right]$, we define $d'_0=d_0$ and for $0\leq k\leq K-2$ define 
\begin{equation*}
    \begin{aligned}
        d'_{k+1} = \frac{(n-t_k)d_{k}-\sum_{j=t_k+1}^{t_{k+1}}a_j\mathbbm{1}\left\{r_j>{a_j}^Tp_{k+1}\right\}}{n-t_{k+1}}\mathbbm{1}\{k<\bar\kappa\}+d'_k\mathbbm{1}\{k\geq\bar\kappa\}
    \end{aligned}
\end{equation*}
Then, $d_k\mathbbm{1}\{k<\bar\kappa\}=d'_k\mathbbm{1}\{k<\bar\kappa\}$. In addition, because $x_j^{\pi_1}\leq \mathbbm{1}\{r_j>{a_j}^Tp_{k+1}\}$ for $t_k<j\leq t_{k+1}$ and each element of $a_j$ is positive with probability $1$, $d'_{\bar\kappa}$ is smaller than $d_{\bar\kappa}$ in all the entries with probability $1$. Denote $d_k(i)$ and $d'_k(i)$ as the $i$th component of these two vectors. Then, if there exists $i$ such that $d_{\bar\kappa}(i)<d_0(i)-\delta_d$, then $d'_{\bar\kappa}(i)<d_0(i)-\delta_d$. If $d_{\bar\kappa}(i)\geq d_0(i)-\delta_d$ for all $i$, then it means $x_j^{\pi_1}= \mathbbm{1}\{r_j>{a_j}^Tp_{k+1}\}$ for all $j\leq t_{\bar\kappa}$, which implies that $d'_{\bar\kappa}=d_{\bar\kappa}$. Together with the fact that $d'_{k+1}=d'_k$ if $k\geq\bar\kappa$, $\mathbb{P}\{\bar\kappa\leq k\}\leq \mathbb{P}\{d'_{k}\notin \Omega(\delta_d)\}$ for all $1\leq k\leq K-1$.
Thus,
\begin{equation*}
    \begin{aligned}
        \mathbb{E}[K-\bar{\kappa}]=&\sum_{k=1}^{K-1}\mathbb{P}\{\bar{\kappa}\leq k\}\leq \sum_{k=1}^{K-1}\sum_{i=1}^{m}\frac{\mathbb{E}
        \left[\abs{d'_k(i)-d_0(i)}^2\right]}{\delta_d^2}.
    \end{aligned}
\end{equation*}
The following Lemma \ref{lem: DARes} allows us to further upper bound the right-hand-side of the above inequality.
\begin{lem}\label{lem: DARes}
    There exists constant $C_{Res}$ such that,
    \begin{equation*}
        \sum_{k=1}^{K-1}\sum_{i=1}^m\mathbb{E}\left[\abs{d'_k(i)-d_0(i)}^2\right]\leq\frac{C_{Res}\log K}{B}.
    \end{equation*}
\end{lem}
Note that if $p_{k+1}$ in the definition of $d_k$ and $d'_k$ are replaced with $p^*(d_k)$, $\{d'_k\}_{k=0}^{K-1}$ forms a martingale adapt to $\{\mathcal{F}_k\}_{k=0}^{K-1}$ and the left-hand-side of the inequality in Lemma \ref{lem: DARes} is in fact the sum of the variances of $d'_k$'s. Since the noisy estimation $p_{k+1}$ of $p^*(d_k)$ is in fact used, Lemma \ref{lem: DARes} then says that the noisy version of this sum of variances decay in the rate of $\frac{1}{B}$, which is one of the most essential reasons on why the regret upper bound only depends on the number of batches $K$. By Lemma \ref{lem: DARes},
\begin{equation}\label{eqn: RG2}
    \left(\bar{r}+\frac{\bar{r}\bar{a}}{\underline{a}}\right)\mathbb{E}\left[t_{K-1}-t_{\bar{\kappa}-1}\right]\leq \left(\bar{r}+\frac{\bar{r}\bar{a}}{\underline{a}}\right)\frac{C_{Res}}{\delta_d^2}\log K
\end{equation}
Finally, define 
\begin{equation*}
    u_{j,k}(p,d)=\Delta_j\left(p_{kB:n}^*(d)\right)\left(\mathbbm{1}\left\{\Delta_j\left(p_{kB:n}^*(d)\right)>0\right\}-\mathbbm{1}\{\Delta_j(p)>0\}\right)    
\end{equation*}
$\mathbb{E}\left[\mathbbm{1}\{k<\bar\kappa\}u_{j,k}(p_k,d_k)\right]$ is the expected difference between the net value of order $j$ earned by using $p_k$ and $p^*_{kB:n}(d_k)$ if the average reaming resource at the end of batch $k$ is sufficiently close to the initial average resource $d_0$. This difference should be small if $\norm{p_k-p^*(d_{k-1})}_2$, $\norm{p_{kB:n}^*-p^*(d_k)}_2$ and $\norm{d_{k-1}-d_k}_2$ are sufficiently small. Lemma \ref{Lemma: MultiRegretComp} provides a rigorous upper bound on the sum of $\mathbb{E}\left[\mathbbm{1}\{k<\bar\kappa\}u_{j,k}(p_k,d_k)\right]$'s.
\begin{lem}\label{Lemma: MultiRegretComp}
There exists a constant $C_{UB}$ such that,
    \begin{equation*}
        \begin{aligned}
            &\sum_{k=1}^{K-1}\sum_{j=t_{k-1}+1}^{t_k}\mathbb{E}\left[\mathbbm{1}\{k<\bar\kappa\}u_{j,k}(p_k,d_k)\right]
            \leq C_{UB}\log K.
        \end{aligned}
    \end{equation*}
\end{lem}
Summing the upper bounds in (\ref{eqn: RG1}), (\ref{eqn: RG2}) and Lemma \ref{Lemma: MultiRegretComp} completes the proof for the case $B>\frac{\bar a}{\underline{d}}$. When $B\leq \frac{\bar a}{\underline{d}}$, if $K\leq 1 + \left\lceil\frac{\bar a}{\underline{d}}\right\rceil$, there are only a constant number of orders and the regret is certainly bounded by a constant. If $K > 1 + \left\lceil\frac{\bar a}{\underline{d}}\right\rceil$, we can modify the definition of $\bar\kappa$ to $\inf\left\{1\leq k\leq K-1:d_k\notin\Omega(\delta_d) \right\}\wedge K - \left\lceil\frac{\bar a}{\underline{d}}\right\rceil$. Then, except for $\mathbb{P}\{\bar\kappa\leq k\}\leq \mathbb{P}\{d'_{k}\notin \Omega(\delta_d)\}$ is true only for $1\leq k\leq K - \left\lceil\frac{\bar a}{\underline{d}}\right\rceil-1$, all the analyses above including Lemma \ref{lem: DARes} and Lemma \ref{Lemma: MultiRegretComp} are not affected. Thus, we only need to modify the analysis for $\mathbb{E}[t_{K-1}-t_{\bar\kappa-1}]$. We now have
\begin{equation*}
    \mathbb{E}\left[K-\bar\kappa\right] \leq \sum_{k=1}^{K-1}\sum_{i=1}^{m}\frac{\mathbb{E}
        \left[\abs{{e_i}^Td'_k-{e_i}^Td_0}^2\right]}{\delta_d^2}+\left\lceil\frac{\bar a}{\underline{d}}\right\rceil+1
\end{equation*}
Then, by Lemma \ref{lem: DARes},
\begin{equation*}
    \mathbb{E}\left[t_{K-1}-t_{\bar{\kappa}-1}\right]\leq \frac{C_{Res}}{\delta_d^2}\log K + \left(\left\lceil\frac{\bar a}{\underline{d}}\right\rceil+1\right)B\leq \frac{C_{Res}}{\delta_d^2}\log K + \left\lceil\frac{\bar a}{\underline{d}}\right\rceil\left(\left\lceil\frac{\bar a}{\underline{d}}\right\rceil+1\right) 
\end{equation*}
Since there is only an extra constant added to the final regret upper bound, the theorem is still true. Proofs of Lemma \ref{lem: DARes} and Lemma \ref{Lemma: MultiRegretComp} can be found in the appendix.
\subsection{Leave-Some-Out Analysis of the Dual Price}\label{Section: LOO}
The proof of Lemma \ref{lem: DARes} and Lemma \ref{Lemma: MultiRegretComp} use the fact that $(r_{j_1},a_{j_1})$, $(r_{j_2},a_{j_2})$, and $p_k$ are mutually independent for any $1\leq k\leq K-1$ and $t_{k-1}<j_1<j_2\leq t_k$. This fact is no longer true for $k=1$ when we replace $p_1$ with $\tilde p_1$ in the modified algorithm given in Section \ref{Section: SimulatorNotNecessary}. To deal with this issue, we define the following leave-some-out price. Denote $\{1,\cdots,N\}$ as $[N]$. Let $\mathcal{J}$ be a subset of $[N]$ and $\abs{\mathcal{J}}$ be the number of the elements in $\mathcal{J}$. Define
\begin{equation*}
    p_{\mathcal{J},N}(d)=\arg\min_{p\geq 0} d^Tp+\frac{1}{N-\abs{\mathcal{J}}}\left(\sum_{j=1}^N\left(\Delta_j(p)\right)^+-\sum_{j\in J}\left(\Delta_j(p)\right)^+\right)
\end{equation*}
where we call $p_{\mathcal{J},N}(d)$ a leave-$\abs{\mathcal{J}}$-out price of $p^*_N(d)$. Since $\{(r_j,a_j):j\in\mathcal{J}\}$ are not used to compute $p_{\mathcal{J},N}(d)$, they and $p_{\mathcal{J},N}(d)$ are mutually independent. Let $\Omega(\delta_d)$ be defined the same as in Section \ref{Section: ProofUpperBound}, the following lemma says that the difference between decisions made by $p_N^*(d)$ and $p_{\mathcal{J},N}(d)$ is small enough. 
\begin{lem}\label{lem:Leave-One-Out}
    Under Assumptions \ref{Assumption: DistributionMulti}, \ref{Assumption: DistributionDensity}, \ref{Assumption: LOO}, there exists a constant $C_{LOO,J}$ such that for all $d\in\Omega(\delta_d)$, $N> J$, $j\in [N]$ and $\mathcal{J}\subseteq [N]$ with $\abs{\mathcal{J}}=J$ 
    \begin{equation}\label{eqn: LOO}
        \mathbb{P}\left\{\mathbbm{1}\left\{\Delta_j\left( p_N^*\right)>0\right\}\neq\mathbbm{1}\left\{\Delta_j\left(p_{\mathcal{J},N}(d)\right)>0\right\}\right\}\leq\frac{C_{LOO,J}}{N}
    \end{equation}
    where $C_{LOO,J}$ only depends on $J$, the distribution of $(r_1,a_1)$ and the initial average resource $d_0$.
\end{lem}
In the regret analysis, we replace $\tilde p_1$ with appropriate leave-some-out prices to create required independence. Lemma \ref{lem:Leave-One-Out} guarantees that extra terms show up in the regret upper bound due to the replacements are small enough. Technical details are left to the complete proof of Proposition \ref{Prop:RegretBoundDA_2} in the appendix. In fact, (\ref{eqn: LOO}) is the most general condition to guarantee a $O(\log K)$ regret upper bound when the joint distribution of $(r_1,a_1)$ is completely unknown. We believe (\ref{eqn: LOO}) holds under conditions that are weaker than Assumption \ref{Assumption: LOO}, though the proof is left for future works.
\section{Experiments and Discussions}\label{Section: NumExp}
\subsection{Numerical Results: Known Total Number of Orders}
In this section, we test $\pi_1$ (Algorithm \ref{alg: AhdLAMulti}) and $\tilde\pi_1$ (modified algorithm proposed in Section \ref{Section: SimulatorNotNecessary}). We consider the setting with a single type of resource ($m=1$) and the setting with four types of resources ($m=4$). When $m=1$, the initial average resource is set to $d_0=5$. The rewards and the resource consumptions $\{(r_j,a_j)\}_{j=1}^n$ is a sequence of i.i.d random pairs such that $r_1\sim Uniform(1,19)$, $a_1\sim Uniform(1,19)$, and $r_1$ and $a_1$ are independent. When $m=4$, initial average resource of type $i$ is set to $d_0(i)=5$ for $1\leq i\leq 4$. $\{(r_j,a_j)\}_{j=1}^n$ is a sequence of i.i.d random pairs such that $r_1\sim Uniform(1,19)$, and $a_1(i)\sim Uniform(1,19)$ for all $1\leq i\leq 4$. In addition, $r_1$ and each component of $a_1$ are mutually independent. For both of the settings, we run numerical experiments with number of batches $K(=2,8,32,64,128)$ and number of orders $n(=1280,6400,32000,64000,128000)$. The values of $n$ are selected to make them divisible by the largest value of $K$. For each numerical test instance, we run $1000$ simulation trials to estimate the regret. Table \ref{Table: KnownTotalCustomersMulti} and Table \ref{Table: KnownTotalCustomers} in the appendix summarize the estimated regrets of $\pi_1$ and $\tilde\pi_1$ respectively.

We use Figure \ref{Figure: DetArrival} to discuss the insights of the numerical results in Table \ref{Table: KnownTotalCustomersMulti} and Table \ref{Table: KnownTotalCustomers}. The plots in the first two columns of Figure \ref{Figure: DetArrival} demonstrates the impact of the number of batches on the regret. In both of the settings, these plots show that the regret of both $\pi_1$ and $\tilde\pi_1$ is logarithmic in $K$, which justifies the $O(\log K)$ regret upper bounds in Theorem \ref{Thm:RegretBoundDAMulti} and Proposition \ref{Prop:RegretBoundDA_2}. The plots in the last two columns of Figure \ref{Figure: DetArrival} demonstrates the impact of the total number of orders on the regret of $\pi_1$ and $\tilde\pi_1$. When $m=1$, these plots show that the regret of both $\pi_1$ and $\tilde\pi_1$ are bounded above by a constant that does not depend on $n$, which is consistent with Theorem \ref{Thm:RegretBoundDAMulti} and Proposition \ref{Prop:RegretBoundDA_2}. When $m=4$, these plots show that the regret of $\pi_1$ and $\tilde\pi_1$ increase in $n$ when $n$ is small, and larger $n$ is required to see the regret becomes constant for larger $K$, but these numerical results still justify the fact that the regret has a constant upper bound that is independent of $n$. 
\begin{figure}[H]
    \centering
    \includegraphics[width=0.8\textwidth]{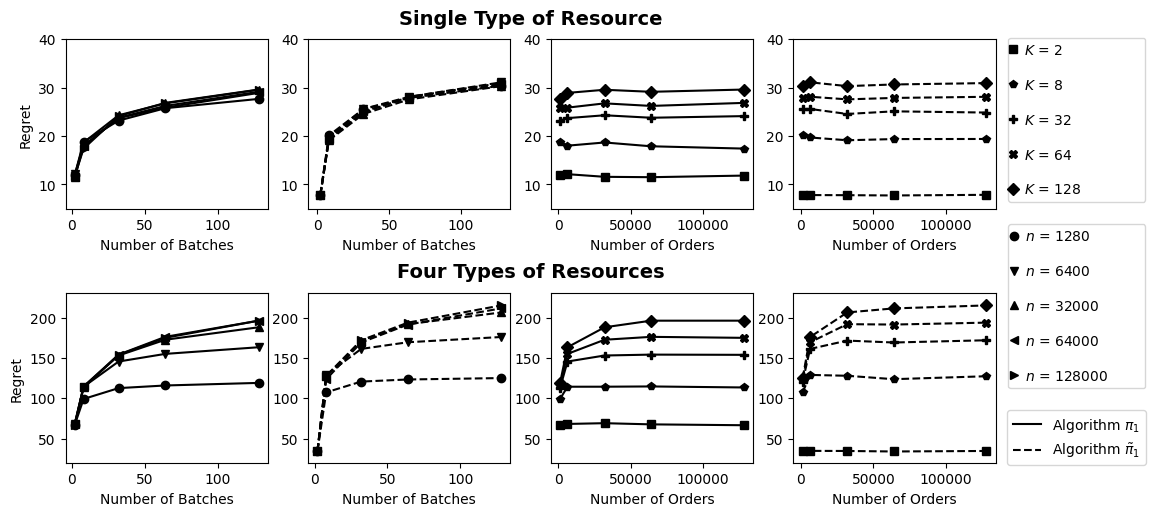}
    \caption{Regret of $\pi_1$ and $\tilde\pi_1$}
    \label{Figure: DetArrival}
\end{figure}
Note that the setting with $m=4$ does not satisfy Assumption \ref{Assumption: LOO}; however, the numerical results suggest that the regret of $\tilde\pi_1$ is still $O(\log K)$. These numerical results support our conjecture that Proposition \ref{Prop:RegretBoundDA_2} should still hold without Assumption \ref{Assumption: LOO}, and the proof of Proposition \ref{Prop:RegretBoundDA_2} under weaker assumption is an interesting future research direction.
\subsection{Numerical Results: Poisson Arrival Process}
In this section, we test $\pi_2$ (Algorithm \ref{alg: RaAhdLA}) and a variant of it in the setting with a single type of resource $m=1$. The initial average resource is set to $d_0 = 5$. The rewards and the resource consumptions $\{(r_j,a_j)\}_{j=1}^n$ is a sequence of i.i.d random pairs such that $r_1\sim Uniform(1,19)$, $a_1\sim Uniform(1,19)$, and $r_1$ and $a_1$ are independent.

The arrival rate of the Poisson process is assumed to be unknown in $\pi_2$ (Algorithm \ref{alg: RaAhdLA}), and only decisions on orders arriving in the first and the last batch are delayed. It is interesting to investigate how much improvement can be achieved if we assume the arrival rate is known at the beginning of the planning horizon and decisions on orders arriving in every batch are delayed. We make the following two modifications to Algorithm \ref{alg: RaAhdLA}. First, we replace the estimated arrival rate $\hat{\lambda}_k$ with the known arrival rate $\lambda$, i.e., $\hat{\lambda}_{k} = \lambda\;\forall k=0,\cdots,K-1.$ Secondly, we include the information from orders arriving in the current batch into the computation of dual prices used in batch $k=2,\cdots,K-1$, i.e., $p_k = \arg\min_{p\geq 0} N(t_{k})\Tilde{d}_{k-1}p+\sum_{j=1}^{N(t_{k})}\left(r_j-a_jp\right)^+$. We call the algorithm with the above two modifications the Modified Algorithm $\tilde\pi_2$. 

We run numerical experiments with arrival rate $\lambda(=10,50,100,500)$, length of planning horizon $T(=128,512,1024)$, and number of batches $K(=2,8,32,64,128)$. For each numerical test instance, we run 1000 simulation trials to estimate regrets. Table \ref{Table: UnknownRate} and Table \ref{Table: KnownRate} in the appendix summarize the estimated regrets of $\pi_2$ and the modified version. To illustrate the insights provided by these numerical results, we select several numerical test instances to make Figure \ref{Figure: Poisson}. The first two plots in Figure \ref{Figure: Poisson} justify Theorem \ref{Thm: RegretPA} by indicating that the regret of Algorithm \ref{alg: RaAhdLA} does not depend on the arrival rate $\lambda$ and the length of the planning horizon $T$. The third plot in \ref{Figure: Poisson} shows that, although knowing arrival rate and delaying decisions on orders arriving in each of the batch can improve the performance, modified $\tilde\pi_2$ does not have a smaller order of regret, and both $\pi_2$ and modified $\tilde\pi_2$ achieve $O(\log K)$ regret.
\begin{figure}[H]
    \centering
    \includegraphics[width=\textwidth]{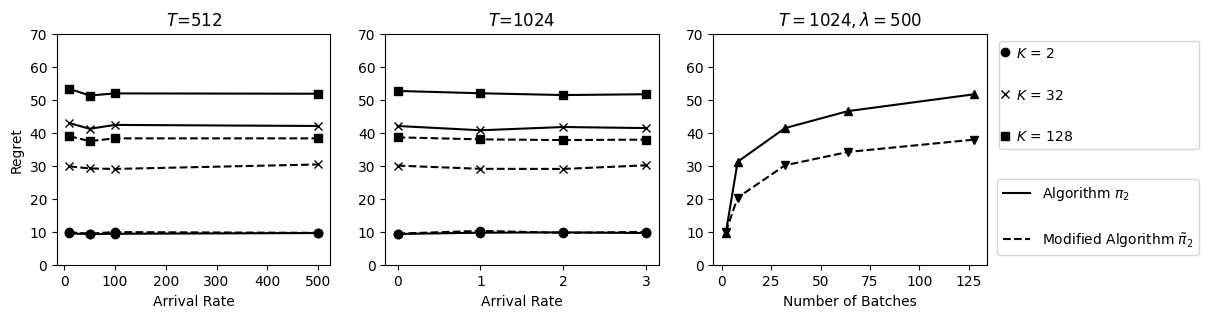}
    \caption{Regret of $\pi_2$ and Modified $\tilde\pi_2$}
    \label{Figure: Poisson}
\end{figure}
\subsection{Numerical Results: Large Support}
We test Algorithm \ref{alg: finite_conti} in the following multi-secretary problem. Total budget is $0.25n$ and i.i.d. rewards $r_j$'s follow
\begin{equation*}
        \mathbb{P}\left\{r_1=1+\frac{8\cdot (i -1)}{\lfloor 10\cdot n^{\delta}\rfloor-1}\right\} = \frac{1}{\lfloor10\cdot  n^{\delta}\rfloor}\quad\forall 1\leq i\leq \lfloor 10\cdot n^{\delta}\rfloor
    \end{equation*}
For each combination of $\alpha=[0,0.25,0.5,0.75,1]$, $n=[1280,6400,32000,64000,128000]$ and $K=[2,8,32,64,128]$, we run 1000 simulations. Table \ref{Table: Finit_Conti} summarizes the average regret. In the special case with $\alpha=0$, we observe the same pattern of regret as in \cite{Xie.et.al.2023}, i.e., the regret decreases rapidly when the batch size gets larger. To be more specific, when $K$ is fixed, the regret decreases as $n$ gets larger, and when $n$ is fixed, the regret decreases as $K$ gets smaller. For the cases with $\alpha\geq 0.5$, we can see the pattern of regret can be captured by the $O(\log K)$ regret bound as predicted by Proposition \ref{Proposition: Growing Support}. When $K$ is fixed, the regret is almost constant in $n$, and when $n$ is fixed, the regret is of order $O(\log K)$. These are visualized by the Figure \ref{Figure: Finit_Conti}.
\begin{figure}[H]
    \centering
    \includegraphics[width=0.5\linewidth]{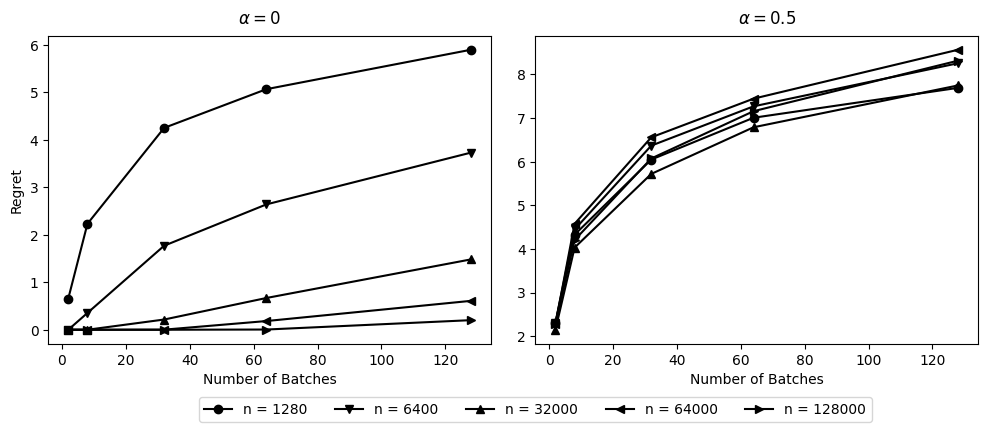}
    \caption{Finite support with large support: Regret of Algorithm \ref{alg: finite_conti}}
    \label{Figure: Finit_Conti}
\end{figure}
In addition, we see from Table \ref{Table: Finit_Conti} that in the case $\alpha=0.25$, the regret still has the pattern of $O(\log K)$, which shows the possibility of generalizing Proposition \ref{Proposition: Growing Support} beyond the case $\alpha\geq 0.5$
\section{Conclusion}\label{Sec: Conclusion}
In this paper, we introduced and analyzed the problem of Online Linear Programming (OLP) with batching, demonstrating that the batching operation fundamentally alters the regret behavior of online resource allocation problems. We proposed re-solving algorithms and provided rigorous analyses of their performance. When the total number of orders is known and the reward has a continuous support, we established that our algorithms achieve $O(\log K)$ regret. We showed that these regret upper bounds are tight by providing an $\Omega(\log K)$ regret lower bound in the single-resource case. Crucially, these regret bounds depend exclusively on the number of batches $K$ and is entirely independent of the total number of orders $n$. Furthermore, we extended our theoretical framework to accommodate environments with random Poisson arrivals, proving that the logarithmic regret bound still holds. Finally, we studied a more realistic large support setting, where the number of distinct order types is finite but scales as $S = \Theta(n^\alpha)$. We analyzed a multisecretary problem under this large support environment and proved that the $\Theta(\log K)$ regret bound established in the continuous support setting still holds when $\alpha \geq 1/2$. In other words, the regret remains independent of both the total number of orders and the number of order types, provided the support size is sufficiently large relative to $n$. This regret analysis provides strong evidence that the continuous support setting is an appropriate approximation of the more realistic large support setting, justifying it as a reasonable modeling assumption.

Beyond our theoretical contributions, this work yields actionable managerial insights for practical implementation. Because our proposed pricing algorithms only require re-solving an LP at the beginning of each batch, the batching operation drastically reduces computational burdens and naturally satisfies real-world constraints against frequent price fluctuations. We further demonstrated that the optimal order of regret can be achieved by delaying orders only in the first and last batches, ensuring that the vast majority of customers are processed immediately. This minimal delay requirement relies on a strategic re-solving schedule for the intermediate batches to prevent severe regret degradation. Consequently, the number of batches $K$ becomes the key parameter of the batching operation. When designing an appropriate batching operation, decision-makers are freed from managing the overall scale of arrivals $n$ and can instead focus on selecting the optimal $K$. With an estimate of the implicit business costs related to order delays, our quantitative logarithmic bounds can help managers formally balance the penalty of resource misallocation against the costs of delaying decisions when selecting the appropriate number of batches, $K$. 

Our study of OLP with batching opens several avenues for future research. First, extending the Leave-Some-Out analysis of dual prices to accommodate weaker assumptions remains an important challenge. Resolving this would allow the $O(\log K)$ regret upper bound and the $\Omega(\log K)$ lower bound to hold in more general settings. Second, future work should generalize the analysis of the large support setting beyond the multisecretary problem to broader OLP frameworks. Additionally, the exact regret behavior when the support size scales at a slower rate relative to $n$ (i.e., $\alpha < 1/2$) requires further investigation. Finally, identifying the optimal re-solving frequency required to guarantee the $O(\log K)$ regret upper bound stands as a key practical and theoretical question for future exploration.

\bibliographystyle{plainnat}
\bibliography{references}

\clearpage
\appendix
\section{Allocation of the Delay Budget}
All of our regret bounds based on the assumption of equal batch size. This assumption allow us to establish regret bounds that are only in term of number of batches $K$ and ease the comparison on the requirements of customer patience and delay budget between different algorithms. These regret bounds also imply that we should select the smallest $K$ that is allowed by the delay budget and satisfies the requirement on customer patience. However, decision makers usually have the flexibility to divide the planning horizon into batches with different sizes. One of the difficulties of analyzing OLP with batching without equal batch size assumption is that there are infinite many possible batching schemes. For different batching schemes, the corresponding appropriate algorithms and the regret bounds may also be different. Thus, if we relaxed the equal batch size assumption, we need an unified framework to study OLP with batching  to avoid endless analyses.

We can use $\left\{n_k\right\}_{k=1}^{K(n)}$ to label a batching scheme, where $K(n)$ is the number of batches when the total number of orders is $n$, and $n_k$ is the index of the last order in batch $k$. Instead of analyzing each of the possible batching scheme, we can view selecting a batch scheme as a decision additional to the resource allocation decision. Since the selection of batching scheme is subject to the delay budget and the customer patience, we can view the ability to delay decisions as an additional resource which need to be allocated together with the resources demanded by the customers. To be more precise, we can formulate the joint problem of deciding batch scheme and resource allocation as the following MDP. Let $n$ be the length of the planning horizon, and let $(b,D)$ be the state, where $b$ is the remaining resource and $D$ is the remaining delay budget. The actions to take at the end of each batch is $(B,\mathbf{x})$, where $B$ is the size of the next batch and $\mathbf{x}\in \{0,1\}^B$ is the resource allocation decision in the next batch. Let $\bar B$ be the maximum batch size allowed. The bellman equations are
\begin{equation*}
    \begin{aligned}
        &V_t(b,D) = \max_{B\leq \bar B,\;B(B-1)\leq 2D}\mathbb{E}\left[\max_{\mathbf{x}\in X_{t,B}(b)}\sum_{j=t+1}^{t+B}r_jx_j+V_{t+B}\left(h_{t,B}(b,D)\right)\right]\quad t=0,\cdots,n-1\\
        &V_n(b,D)=0
    \end{aligned}
\end{equation*}
where
\begin{equation*}
    \begin{aligned}
         &X_{t,B}(b)=\left\{\mathbf{x}=\left\{x_j\right\}_{j=t+1}^{t+B}:\sum_{j=t+1}^{t+B}a_jx_j\leq b,\;\mathbf{x}\in \{0,1\}^B\right\}\\
         &h_{t,B}(b,D) = \left(b-\sum_{j=t+1}^{t+B}a_jx_j,D-\frac{B(B-1)}{2}\right) 
    \end{aligned}
\end{equation*}
With the above MDP formulation, we at least can apply some reinforcement learning algorithms to find a near optimal policy although the computation complexity can be very high. Finding computationally efficient algorithms and analyzing their properties can be a future extension of this work.
\newpage
\section{Proofs}
\subsection{Proof of Lemma~\ref{Lemma: OptimalRS}}
\begin{proof}
Define
\begin{equation*}
    \begin{aligned}
        &V_{1,k}(d) = \max\left\{\sum_{j=1}^{t_k} r_jx_j:\sum_{j=1}^{t_k} a_jx_j \leq b_0 - (n-t_k)d,\;x_j\in [0,1]\;\forall 1\leq j\leq t_k\right\}\\
        &V_{2,k}(d) = \max\left\{\sum_{j=t_k+1}^{n} r_jx_j:\sum_{j=t_{k}+1}^{n} a_jx_j \leq (n-t_k) d,\;x_j\in [0,1]\;\forall t_k +1\leq j\leq n \right\}
    \end{aligned}
\end{equation*}
An useful conclusion from the proof of Lemma 9 in \cite{Bray2019} is that for all $\delta>0$, there exists $\bar{N}$ such that if $n>\bar{N}$ and $n^{\frac{3}{4}}\leq (K-k)B\leq \frac{n}{2}$,
\begin{equation*}
    \mathbb{P}\left\{R_n^*-\sup_{d\notin\Omega(\delta)}\left(V_{1,k}(d)+V_{2,k}(d)\right)\geq n^{\frac{2}{3}}\right\}\geq\frac{1}{2}
\end{equation*}
Because $R_n^*-R_n(\pi^*)\geq R_n^*-V_{1,k}(d_k)-V_{2,k}(d_k)$ with probability $1$, then
\begin{equation*}
    \mathbb{P}\left\{d_{k}\notin \Omega(\delta)\;\mathrm{and}\;R_n^*-\sup_{d\notin\Omega(\delta)}\left(V_{1,k}(d)+V_{2,k}(d)\right)\geq n^{\frac{2}{3}}\right\}\leq\mathbb{P}\left\{R_n^*-R_n(\pi^*)\geq n^{\frac{2}{3}}\right\}\leq \frac{C_{AhdLA}\log n}{n^{\frac{2}{3}}}
\end{equation*}
where the last inequality is from the fact that AhdLA is a feasible policy and its regret is smaller than $C_{AhdLA}\log n$ for some constant $C_{AhdLA}$ that only depends on the distribution of $(r_1,a_1)$ and the initial average resource $d_0$. Then,
\begin{equation*}
    \mathbb{P}\left\{d_k\in\Omega(\delta)\right\}\geq \frac{1}{2} - \frac{C_{AhdLA}\log n}{n^{\frac{2}{3}}}
\end{equation*}
We can assume $\bar N$ is large enough such that, when $n\geq\bar N$, $\mathbb{P}\left\{d_k\in\Omega(\delta)\right\}\geq\frac{1}{4}$.
\end{proof}
\subsection{Proof of Lemma \ref{Lemma: LBMyopic}}
Select $\bar{\delta}$ to be a positive real number satisfying (i) $[d-2\bar{\delta},d+2\bar{\delta}]\subseteq\Omega(\delta_d)$, (ii) $\bar{\delta}<\frac{\delta_p}{2\sqrt{L}}$, (iii) $\bar{\delta}<\frac{\min\{\delta_{4,1},\delta_{-4,1}\}}{2}$, where $\delta_p$, $\delta_{4,1}$, $\delta_{-4,1}$ are given in Assumption \ref{Assumption: DistributionDensity}(c) and Lemma \ref{Lemma: SqrtDualPrice}. Fix $\delta\in (0,\bar\delta)$. For any $d\in [d_0-\delta,d_0+\delta]$, $0\leq k\leq K-2$, and $t_k<j\leq t_{k+1}$, define $(\Delta_t(p))^+=\max \{\Delta_t(p),0\}$ and $(\Delta_t(p))^-=\min\{\Delta_t(p),0\}$
\begin{equation}\label{eqn: LBMyopic}
    \begin{aligned}
        u^{\pi^*,1}_{j,k+1}(d)=&\mathbb{E}\left[(1-x_j^{\pi^*}(d))\left(\Delta_j(p^*_{t_{k+1}:n}(d_{k+1}^*(d)))\right)^+\right]\\
        u^{\pi^*,2}_{j,k+1}(d)=&\mathbb{E}\left[-x_j^{\pi^*}(d)\left(\Delta_j(p^*_{t_{k+1}:n}(d_{k+1}^*(d)))\right)^-\right]
    \end{aligned}
\end{equation}
Then, $u^{\pi^*}_{j,k+1}(d)=u^{\pi^*,1}_{j,k+1}(d)+u^{\pi^*,2}_{j,k+1}(d)$. It is sufficient to only analyze the case $B\geq 2$ because $B=1$ is the case of standard OLP without batching, in which the lower bound is true and has been proved in \cite{Bray2019}. Define
\begin{equation*}
    \Tilde{p}_{j,t_{k}:n}(d) = \arg\min_{p\geq 0} (n-t_{k})dp - (\Delta_j(p))^+ + \sum_{t=t_k+1}^{n}(\Delta_t(p))^+ 
\end{equation*}
and
\begin{equation*}
    \Tilde{d}_{j,k+1}(d) = \frac{(n-t_k)d + (a_j\mathbbm{1}\{\Delta_j(\tilde p_{j,t_k:n})>0\}-\bar{a})-\sum_{t=t_k+1}^{t_{k+1}}a_t\mathbbm{1}\{\Delta_t(\tilde p_{j,t_k:n})>0\}}{n-t_{k+1}}
\end{equation*}
Then, $\Tilde{d}_{j,k+1}(d)\leq d^*_{k+1}(d)$ almost sure. Thus, $p^*_{t_{k+1}:n}(d_{k+1}^*(d))\leq p^*_{t_{k+1}:n}(\Tilde{d}_{j,k+1}(d))$ almost sure. Then, 
\begin{equation*}
    \begin{aligned}
        u^{\pi^*,1}_{j,k+1}(d)\geq \mathbb{E}\left[(1-x_j^{\pi^*}(d))\left(\Delta_j(p^*_{t_{k+1}:n}(\Tilde{d}_{j,k+1}(d)))\right)^+\right]    
    \end{aligned}
\end{equation*}
Define
\begin{equation*}
    \begin{aligned}
        &E_{1,1}(d,k,j) = \left\{\abs{\Tilde{d}_{j,k+1}(d)-d}\leq\frac{1}{\sqrt{L(n-t_{k+1})}}\right\}\\
        &E_{1,2}(d,k,j) = \left\{\abs{\sqrt{n-t_{k+1}}(p^*_{t_{k+1}:n}(\Tilde{d}_{j,k+1}(d))-p^*(\Tilde{d}_{j,k+1}(d))+4}\leq 1\right\}
    \end{aligned}
\end{equation*}
Then, on the event $E_{1,1}(d,k,j)\cap E_{1,2}(d,k,j)$
\begin{equation*}
    \begin{aligned}
        p^*_{t_{k+1}:n}(\Tilde{d}_{j,k+1}(d))-p^*(d) = (p^*_{t_{k+1}:n}(\Tilde{d}_{j,k+1}(d))-p^*(\Tilde{d}_{j,k+1}(d)))+(p^*(\Tilde{d}_{j,k+1}(d))-p^*(d))\leq -\frac{2}{\sqrt{n-t_{k+1}}}
    \end{aligned}
\end{equation*}
Then,
\begin{equation*}
    \begin{aligned}
        \mathbb{E}\left[(1-x_j^{\pi^*}(d))\left(\Delta_j(p^*_{t_{k+1}:n}(\Tilde{d}_{j,k+1}(d)))\right)^+\right] \geq& \mathbb{E}\left[\mathbbm{1}_{E_{1,1}(d,k,j)}\mathbbm{1}_{E_{1,2}(d,k,j)}(1-x_j^{\pi^*}(d))\left(\Delta_j\left(p^*(d)-\frac{2}{\sqrt{n-t_{k+1}}}\right)\right)^+\right]
    \end{aligned}
\end{equation*}
In addition,
\begin{equation*}
    \begin{aligned}
        &\mathbbm{1}\left\{\Delta_j\left(p^*(d)-\frac{1}{\sqrt{n-t_{k+1}}}\right)>0\right\}\left(\Delta_j\left(p^*(d)-\frac{2}{\sqrt{n-t_{k+1}}}\right)\right)^+\\
        \geq &\mathbbm{1}\left\{\Delta_j\left(p^*(d)-\frac{1}{\sqrt{n-t_{k+1}}}\right)>0\right\}\frac{a_j}{\sqrt{n-t_{k+1}}}\\
        \geq &\left(\mathbbm{1}\left\{\Delta_j\left(p^*(d)-\frac{1}{\sqrt{n-t_{k+1}}}\right)>0\right\}-\mathbbm{1}\left\{\Delta_j\left(p^*(d)+\frac{1}{\sqrt{n-t_{k+1}}}\right)>0\right\}\right)\frac{a_j}{\sqrt{n-t_{k+1}}}
    \end{aligned}
\end{equation*}
Thus,
\begin{equation*}
    \begin{aligned}
        u^{\pi^*,1}_{j,k+1}(d)   \geq&\mathbb{E}\left[\mathbbm{1}_{E_{1,1}(d,k,j)}\mathbbm{1}_{E_{1,2}(d,k,j)}(1-x_j^{\pi^*}(d))\left(\mathbbm{1}\left\{\Delta_j\left(p^*(d)-\frac{1}{\sqrt{(K-k-1)B}}\right)>0\right\}\right.\right.\\
        &\qquad\qquad\qquad\qquad\qquad-\left.\left.\mathbbm{1}\left\{\Delta_j\left(p^*(d)+\frac{1}{\sqrt{(K-k-1)B}}\right)>0\right\}\right)\frac{a_j}{\sqrt{(K-k-1)B}}\right]
    \end{aligned}
\end{equation*}
Define
\begin{equation*}
    \hat{d}_{j,k+1}(d) = \frac{(n-t_k)d+\left(a_j\mathbbm{1}\{\Delta_j(\tilde p_{j,t_k:n})>0\}+\frac{\bar{a}^2}{\underline{a}}\right)-\sum_{t=t_k+1}^{t_{k+1}}a_t\mathbbm{1}\{\Delta_t(\tilde p_{j,t_k:n})>0\}}{n-t_{k+1}}
\end{equation*}
Then, $\hat{d}_{j,k+1}(d)\geq d^*_{k+1}(d)$ almost sure. Thus, $p^*_{t_{k+1}:n}(d_{k+1}^*(d))\geq p^*_{t_{k+1}:n}(\hat{d}_{j,k+1}(d))$ almost sure. Define
\begin{equation*}
\begin{aligned}
    &E_{2,1}(d,k,j) = \left\{\abs{\hat{d}_{j,k+1}(d)-d}\leq\frac{1}{\sqrt{L(n-t_{k+1}))}}\right\}\\
    &E_{2,2}(d,k,j) = \left\{\abs{\sqrt{n-t_{k+1}}(p^*_{t_{k+1}:n}(\hat{d}_{j,k+1}(d))-p^*(\hat{d}_{j,k+1}(d))-4}\leq 1\right\}
\end{aligned}
\end{equation*}
Then, with similar arguments above,
\begin{equation*}
    \begin{aligned}
        u^{\pi^*,2}_{j,k+1}(d)\geq&\mathbb{E}\left[\mathbbm{1}_{E_{2,1}(d,k,j)}\mathbbm{1}_{E_{2,2}(d,k,j)}x_j^{\pi^*}(d)\left(\mathbbm{1}\left\{\Delta_j\left(p^*(d)-\frac{1}{\sqrt{(K-k-1)B}}\right)>0\right\}\right.\right.\\
        &\qquad\qquad\qquad\qquad\qquad-\left.\left.\mathbbm{1}\left\{\Delta_j\left(p^*(d)+\frac{1}{\sqrt{(K-k-1)B}}\right)>0\right\}\right)\frac{a_j}{\sqrt{(K-k-1)B}}\right]
    \end{aligned}
\end{equation*}
If $K-k>\max\left\{\frac{2(\bar{a}+\delta-d_0)}{\delta},\frac{2(d_0+\delta)}{\delta},\frac{2\bar{a}^2}{\underline{a}\delta}\right\}+1$, we have
\begin{equation*}
    \begin{aligned}
        \frac{1}{n-t_{k+1}}\left((n-t_k)d-B\bar{a}-\frac{\bar{a}^2}{\underline{a}}\right)\geq d_0-2\delta\;\;\mathrm{and}\;\;\frac{1}{n-t_{k+1}}\left((n-t_k)d+\frac{\bar{a}^2}{\underline{a}}\right)\leq d_0+2\delta
    \end{aligned}
\end{equation*}
which implies that
\begin{equation*}
    \Tilde{d}_{j,k+1}(d),\hat{d}_{j,k+1}(d)\in [d_0-2\delta,d_0+2\delta]\quad\forall t_k<j\leq t_{k+1}
\end{equation*}
Define
\begin{equation*}
    \begin{aligned}
        &E_{1,3}(k) = \left\{\sup_{d\in[d_0-2\delta,d_0+2\delta]}\abs{\sqrt{n-t_{k+1}}(p^*_{t_{k+1}:n}(d)-p^*(d)+4}\leq 1\right\}\\
        &E_{2,3}(k) = \left\{\sup_{d\in[d_0-2\delta,d_0+2\delta]}\abs{\sqrt{n-t_{k+1}}(p^*_{t_{k+1}:n}(d)-p^*(d)-4}\leq 1\right\}
    \end{aligned}
\end{equation*}
then $K-k>\max\left\{\frac{2(\bar{a}+\delta-d_0)}{\delta},\frac{2(d_0+\delta)}{\delta},\frac{2\bar{a}^2}{\underline{a}\delta}\right\}+1$ implies that $E_{1,3}(k)\subseteq E_{1,2}(d,k,j)$ and $E_{2,3}(k)\subseteq E_{2,2}(d,k,j)$. Let $F_{+}(d,k,j) = E_{1,3}(k)\cap E_{1,1}(d,k,j)\cap E_{2,1}(d,k,j)$, then
\begin{equation*}
    \begin{aligned}
        u^{\pi^*,1}_{j,k+1}(d)\geq&\mathbb{E}\left[\mathbbm{1}_{F_{+}(d,k,j)}(1-x_j^{\pi^*}(d))\left(\mathbbm{1}\left\{\Delta_j\left(p^*(d)-\frac{1}{\sqrt{n-t_{k+1}}}\right)>0\right\}\right.\right.\\
        &\qquad\qquad\qquad\qquad\qquad-\left.\left.\mathbbm{1}\left\{\Delta_j\left(p^*(d)+\frac{1}{\sqrt{n-t_{k+1}}}\right)>0\right\}\right)\frac{a_j}{\sqrt{n-t_{k+1}}}\right]
    \end{aligned}
\end{equation*}
Similarly, define $F_{-}(d,k,j) = E_{2,3}(k)\cap E_{1,1}(d,k,j)\cap E_{2,1}(d,k,j)$, then
\begin{equation*}
    \begin{aligned}
        u^{\pi^*,2}_{j,k+1}(d)\geq&\mathbb{E}\left[\mathbbm{1}_{F_{-}(d,k,j)}x_j^{\pi^*}(d)\left(\mathbbm{1}\left\{\Delta_j\left(p^*(d)-\frac{1}{\sqrt{n-t_{k+1}}}\right)>0\right\}\right.\right.\\
        &\qquad\qquad\qquad\qquad\qquad-\left.\left.\mathbbm{1}\left\{\Delta_j\left(p^*(d)+\frac{1}{\sqrt{n-t_{k+1}}}\right)>0\right\}\right)\frac{a_j}{\sqrt{n-t_{k+1}}}\right]
    \end{aligned}
\end{equation*}
Let $\mathcal{H}_{k,j}=\sigma(\{(r_l,a_l)\}_{l=t_k+1}^{t_{k+1}}\backslash (r_j,a_j))$ $q_{+}(d,k,j)=\mathbb{P}\{F_{+}(d,k,j)|\mathcal{H}_{k,j}\}$ and $q_{-}(d,k,j)=\mathbb{P}\{F_{-}(d,k,j)|\mathcal{H}_{k,j}\}$. By the fact $F_{+}(d,k,j)$ and $F_{-}(d,k,j)$ are both independent with $(r_j,a_j)$ and the i.i.d. assumptions on $(r_j,a_j)$'s, let
\begin{equation*}
    H(d,k,j) = \left(\mathbbm{1}\left\{\Delta_j\left(p^*(d)-\frac{1}{\sqrt{n-t_{k+1}}}\right)>0\right\}-\mathbbm{1}\left\{\Delta_j\left(p^*(d)+\frac{1}{\sqrt{n-t_{k+1}}}\right)>0\right\}\right)\frac{a_j}{\sqrt{n-t_{k+1}}}
\end{equation*}
then
\begin{equation*}
    \begin{aligned}
        u^{\pi^*}_{j,k+1}(d)\geq &\mathbb{E}\left[H(d,k,j)\cdot\left((1-x_j^{\pi^*}(d))q_+(d,k,j)+x_j^{\pi^*}(d)q_{-}(d,k,j)\right)\right]\\
        \geq &\mathbb{E}\left[H(d,k,j)\cdot\min\{q_+(d,k,j),q_-(d,k,j)\}\right]\\
        =&\mathbb{E}[H(d,k,j)]\mathbb{E}[\min\{q_+(d,k,j),q_-(d,k,j)\}]
    \end{aligned}
\end{equation*}
Define $C_P(d,k,j)=\mathbb{E}[\min\{q_+(d,k,j),q_-(d,k,j)\}]$. Furthermore, by Assumption \ref{Assumption: DistributionDensity}(c), there exists a constant $\sigma>0$ such that
\begin{equation*}
    \begin{aligned}
        u^{\pi^*}_{j,k+1}(d)\geq &\frac{C_P(d,k,j)}{\sqrt{(K-k-1)B}}\int_{-1}^1\frac{d}{ds}\mathbb{E}\left[a_j\mathbbm{1}\left\{r_j>a_j\left(p^*(d)-\frac{s}{\sqrt{(K-k-1)B}}\right)\right\}\right]ds\\
        \geq & \frac{C_p(d,k,j)}{\sqrt{(K-k-1)B}}\int_{-1}^1\frac{\sigma}{\sqrt{(K-k-1)B}}ds=\frac{2\sigma C_p(d,k,j)}{(K-k-1)B}
    \end{aligned}
\end{equation*}
By Lemma \ref{Lemma: SqrtDualPrice} and our selection of $\delta$, there exists constants $N_{4,1},\;C_{4,1}$ such that if $K-k\geq\max\{N_{4,1}\}+1$, then $\mathbb{P}\left\{E_{1,3}(k)\right\}\geq  C_{4,1}>0$. In addition, by Lemma \ref{Lemma: RemOffline},
\begin{equation*}
    \begin{aligned}
        \mathbb{P}\left\{\abs{\Tilde{d}_{j,k+1}(d)-d}\leq\frac{1}{\sqrt{L(K-k-1)B}}\right\}\geq& 1 - \mathbb{E}\left[\left(\Tilde{d}_{j,k+1}(d)-d\right)^2\right]L(K-k-1)B\\
        \geq & 1-\frac{LC_{off}}{(K-k-1)}
    \end{aligned}
\end{equation*}
and
\begin{equation*}
    \begin{aligned}
        \mathbb{P}\left\{\abs{\hat{d}_{j,k+1}(d)-d}\leq\frac{1}{\sqrt{L(K-k-1)B}}\right\}\geq& 1 - \mathbb{E}\left[\left(\hat{d}_{j,k+1}(d)-d\right)^2\right]L(K-k-1)B\\
        \geq & 1-\frac{LC_{off}}{(K-k-1)}
    \end{aligned}
\end{equation*}
Then,
\begin{equation*}
    \begin{aligned}
        &\mathbb{P}(E_{1,1}^c(d,k,j))+\mathbb{P}( E_{2,1}^c(d,k,j))\leq \frac{2LC_{off}}{(K-k-1)}\quad\forall t_k<j\leq t_{k+1}
    \end{aligned}
\end{equation*}
Let $C_{*}=\min\{C_{4,1},C_{-4,1}\}$ and define
\begin{equation*}
    G(d,k,j) = \left\{\mathbb{P}(E_{1,1}^c(d,k,j)|\mathcal{H}_{k,j}) + \mathbb{P}( E_{2,1}^c(d,k,j)|\mathcal{H}_{k,j})\leq\frac{C_{*}}{2}\right\}
\end{equation*}
Then, if $K-k>1 + \max\left\{N_{4,1},\frac{8LC_{off}}{C_{*}}\right\}$
\begin{equation*}
    \mathbb{P}(G(d,k,j))\geq 1 - \frac{4LC_{off}}{C^*(K-k-1)}\geq \frac{1}{2}
\end{equation*}
Thus, together with the fact that $\mathbb{P}\{E_{1,3}(k)|\mathcal{H}_{j,k}\}=\mathbb{P}\{E_{1,3}(k)\}$, on the event $G(d,k,j)$,
\begin{equation*}
    q_+(d,k,j)\geq \mathbb{P}\{E_{1,3}(k)|\mathcal{H}_{j,k}\} - \mathbb{P}(E_{1,1}^c(d,k,j)|\mathcal{H}_{k,j}) - \mathbb{P}( E_{2,1}^c(d,k,j)|\mathcal{H}_{k,j})\geq C_{4,1} - \frac{C_{*}}{2} \geq \frac{C_{*}}{2}
\end{equation*}
Similarly,
\begin{equation*}
    q_-(d,k,j)\geq\frac{C_{*}}{2}
\end{equation*}
Then,
\begin{equation*}
    C_P(d,k,j)\geq \frac{C_*}{2}\mathbb{P}\{G(d,k,j)\}\geq\frac{C_{*}}{4}
\end{equation*}
Note that $C_*$ does not depend on $d,\;k$ and $j$. To summarize, let $k_{LB}=\max\left\{\frac{2(\bar{a}+\delta-d_0)}{\delta},\frac{2(d_0+\delta)}{\delta},\frac{2\bar{a}^2}{\underline{a}\delta},N_{4,1},\frac{8LC_{off}}{C_{*}},N_{-4,1}\right\}+1$, and $C_{LB}=\frac{\sigma C_*}{2}$. Both $k_{LB}$ and $C_{LB}$ only depends on the distribution of $(r_1,a_1)$, $d_0$ and $\delta$. We have proved that there exists $\bar\delta$ such that for any $\delta\in (0,\bar{\delta})$, there exists $\underline{K}$ and $\underline{C}$ such that, if $K-k\geq k_{LB}$,
\begin{equation*}
    \inf_{d\in\Omega(\delta)}\mathbb{E}\left[u^{\pi^*}_{j,k+1}(d)\right]\geq\frac{C_{LB}}{(K-k-1)B}
\end{equation*}
which completes the proof.
\begin{lem}\label{Lemma: SqrtDualPrice}
    \;\\
    (Lemma 2 in Online Appendix of \cite{Bray2019}) For any $\gamma\in\mathbb{R}$ and $\epsilon>0$, there exists $\delta_{\gamma,\epsilon}$, $C_{\gamma,\epsilon}$, $N_{\gamma,\epsilon}>0$ such that, when $(K-k)B>N_{\gamma,\epsilon}$, 
    \begin{equation*}
        \mathbb{P}\left\{\sup_{d\in [d_0-\delta_{\gamma,\epsilon},d_0+\delta_{\gamma,\epsilon}]}\abs{\sqrt{(K-k)B}(p_{t_{k}:n}^*(d)-p^*(d))-\gamma}\leq\epsilon\right\}\geq C_{\gamma,\epsilon}
    \end{equation*}
\end{lem}
\begin{lem}\label{Lemma: RemOffline}
    For any $k<K-1$ and $\delta\in (0,\bar\delta)$, there exists a constant $C_{off}$ such that,
    \begin{equation*}
        \left\{
        \begin{aligned}
            &\mathbb{E}\left[\left(\Tilde{d}_{j,k+1}(d)-d\right)^2\right]\leq\frac{C_{off}}{(K-k-1)^2B}\quad\forall d\in [d_0-\delta,d_0+\delta]\; t_k<j\leq t_{k+1}\\
            &\mathbb{E}\left[\left(\hat{d}_{j,k+1}(d)-d\right)^2\right]\leq\frac{C_{off}}{(K-k-1)^2B}\quad\forall d\in [d_0-\delta,d_0+\delta]\; t_k<j\leq t_{k+1}
        \end{aligned}
        \right.
    \end{equation*}
\end{lem}
\begin{proof}
    To simplify the notation, we use $[j_1:j_2]$ to denote $\{j_1+1,\cdots,j_2\}$ and use $A_t(p)$ to denote $a_t\mathbbm{1}\{r_t>{a_t}p\}$. Given $d\in [d_0-\delta,d_0+\delta]$ and $k<K-1$, for any $t_k<j\leq t_{k+1}$,
    \begin{equation}\label{eqn: ResOff0}
        \begin{aligned}
            \left(\Tilde{d}_{j,k+1}(d)-d\right)^2 = &\left(\frac{d-\bar{a}}{n-t_{k+1}}\right)^2+\frac{1}{(n-t_{k+1})^2}\left(\sum_{t\in[t_k:t_{k+1}]\backslash j }d-A_t(\Tilde{p}_{j,t_{k}:n}(d))\right)^2\\
            &+\frac{2}{(n-t_{k+1})^2}\left((d-\bar{a})\left(\sum_{t\in[t_k:t_{k+1}]\backslash j }d-A_t(\Tilde{p}_{j,t_{k}:n}(d))\right)\right)
        \end{aligned}
    \end{equation}
    We have
    \begin{equation*}
        \begin{aligned}
            \left(\sum_{t\in[t_k:t_{k+1}]\backslash j }d-A_t(\Tilde{p}_{j,t_{k}:n}(d))\right)^2=&\left(\sum_{t\in[t_k:t_{k+1}]\backslash j }d-A_t(p^*(d))+\sum_{t\in[t_k:t_{k+1}]\backslash j }A_t(p^*(d))-A_t(\Tilde{p}_{j,t_{k}:n}(d))\right)^2
        \end{aligned}
    \end{equation*}
    where
    \begin{equation*}
        \begin{aligned}
            \mathbb{E}\left[\left(\sum_{t\in[t_k:t_{k+1}]\backslash j }d-A_t(p^*(d))\right)^2\right]=\sum_{t_1\in[t_k:t_{k+1}]\backslash j}\sum_{t_2\in\in[t_k:t_{k+1}]\backslash j}\mathbb{E}\left[(d-A_{t_1}(p^*(d)))(d-A_{t_2}(p^*(d)))\right]\leq B(d_0+\delta+\bar a)^2
        \end{aligned}
    \end{equation*}
    because $(r_{t_1},a_{t_1})$ and $(r_{t_2},a_{t_2})$ are independent, and $\mathbb{E}\left[A_t(p^*(d))\right]=d$ by Lemma \ref{lem: BoundedLipschitz}. Also,
    \begin{equation*}
        \begin{aligned}
            &\mathbb{E}\left[\left(\sum_{t\in[t_k:t_{k+1}]\backslash j }A_t(p^*(d))-A_t(\Tilde{p}_{j,t_{k}:n}(d))\right)^2\right]\\
            =&\sum_{t_1\in [t_k:t_{k+1}]\backslash j}\sum_{t_2\in [t_k:t_{k+1}]\backslash j}\mathbb{E}\left[\left(A_{t_1}(p^*(d))-A_{t_1}(\Tilde{p}_{j,t_{k}:n}(d))\right)\left(A_{t_2}(p^*(d))-A_{t_2}(\Tilde{p}_{j,t_{k}:n}(d))\right)\right]
        \end{aligned}
    \end{equation*}
    If $B\leq 3$, $\mathbb{E}\left[\left(\sum_{t\in[t_k:t_{k+1}]\backslash j }A_t(p^*(d))-A_t(\Tilde{p}_{j,t_{k}:n}(d))\right)^2\right]\leq B^2\bar{a}^2\leq B\cdot 9\bar a^2$. If $B\geq 4$, when $t_1=t_2$,
    \begin{equation*}
    \mathbb{E}\left[\left(A_{t_1}(p^*(d))-A_{t_1}(\Tilde{p}_{j,t_{k}:n}(d))\right)\left(A_{t_2}(p^*(d))-A_{t_2}(\Tilde{p}_{j,t_{k}:n}(d))\right)\right]\leq{\bar a}^2    
    \end{equation*}
    When $t_1\neq t_2$,  define
    \begin{equation*}
        \tilde{p}_{k,j,t_1,t_2}(d) = \arg\min_{p\geq 0} dp + \frac{1}{(K-k)B-3}\sum_{t\in [t_k:n]\backslash\{j,t_1,t_2\}}(r_t-a_tp)^+
    \end{equation*}
    Then, $(r_{t_1},a_{t_1})$, $r_{t_2},a_{t_2}$, and $\tilde{p}_{k,j,t_1,t_2}(d)$ are independent. In addition, by Lemma \ref{lem:Leave-One-Out}, there exists a constant $C_{off,1}$
    \begin{equation*}
        \mathbb{P}\left\{\mathbbm{1}\{r_{t_1}>a_{t_1}\tilde p_{j,t_k:n}\}\neq\mathbbm{1}\{r_{t_1}>a_{t_1}\tilde p_{k,j,t_1,t_2}\}\right\}\leq \frac{C_{off,1}}{(K-k)B} 
    \end{equation*}
    Same result also holds for $t_2$. Then,
    \begin{equation*}
        \begin{aligned}
            &\mathbb{E}\left[\left(A_{t_1}(p^*(d))-A_{t_1}(\Tilde{p}_{j,t_{k}:n}(d))\right)\left(A_{t_2}(p^*(d))-A_{t_2}(\Tilde{p}_{j,t_{k}:n}(d))\right)\right]\\
            \leq &\bar{a}^2\mathbb{P}\left\{\mathbbm{1}\{r_{t_1}>a_{t_1}\tilde p_{j,t_k:n}\}\neq\mathbbm{1}\{r_{t_1}>a_{t_1}\tilde p_{k,j,t_1,t_2}\}\right\}\\
            &+\bar{a}^2\mathbb{P}\left\{\mathbbm{1}\{r_{t_2}>a_{t_2}\tilde p_{j,t_k:n}\}\neq\mathbbm{1}\{r_{t_2}>a_{t_2}\tilde p_{k,j,t_1,t_2}\}\right\}\\
            &+\mathbb{E}\left[\left(A_{t_1}(p^*(d))-A_{t_1}(\Tilde{p}_{k,j,t_1,t_2}(d))\right)\left(A_{t_2}(p^*(d))-A_{t_2}(\Tilde{p}_{k,j,t_{1},t_2}(d))\right)\right]
        \end{aligned}
    \end{equation*}
    where
    \begin{equation*}
        \begin{aligned}
            &\mathbb{E}\left[\left(A_{t_1}(p^*(d))-A_{t_1}(\Tilde{p}_{k,j,t_1,t_2}(d))\right)\left(A_{t_2}(p^*(d))-A_{t_2}(\Tilde{p}_{k,j,t_{1},t_2}(d))\right)\right]\\
            =&\mathbb{E}\left[\mathbb{E}\left[A_{t_1}(p^*(d))-A_{t_1}(\Tilde{p}_{k,j,t_1,t_2}(d))\bigg\rvert a_1,a_2,\Tilde{p}_{k,j,t_1,t_2}(d)\right]\cdot\right.\\
            &\left.\qquad\mathbb{E}\left[A_{t_2}(p^*(d))-A_{t_2}(\Tilde{p}_{k,j,t_1,t_2}(d))\bigg\rvert a_1,a_2,\Tilde{p}_{k,j,t_1,t_2}(d)\right]\right]\\
            \leq & \bar{a}^4\beta^2\mathbb{E}\left[(p^*(d)-\tilde p_{k,j,t_1,t_2})^2\right]\leq \frac{\bar a^4\beta^2C_{Dual}}{(K-k)B-3}\leq \frac{4\bar a^4\beta^2C_{Dual}}{(K-k)B} 
        \end{aligned}
    \end{equation*}
    where the last line is by Assumption \ref{Assumption: DistributionDensity}(b) and Lemma \ref{lem: UniformDualConvergence}. Summarizing all the above,
    \begin{equation*}
        \begin{aligned}
            &\mathbb{E}\left[\left(\sum_{t\in[t_k:t_{k+1}]\backslash j }A_t(p^*(d))-A_t(\Tilde{p}_{j,t_{k}:n}(d))\right)^2\right]\\
            \leq & B\cdot 9\bar a^2+ B\bar a^2 + B^2\frac{4\bar a^4\beta^2C_{Dual}}{(K-k)B}  + 2B^2\frac{C_{off,1}}{(K-k)B}\leq B(10\bar a^2+4\bar a^4\beta^2C_{Dual}+2C_{off,1})
        \end{aligned}
    \end{equation*}
    Together with Cauchy-Schwarz Inequality, 
    \begin{equation*}
        \begin{aligned}
            &\mathbb{E}\left[\left(\sum_{t\in[t_k:t_{k+1}]\backslash j }d-A_t(\Tilde{p}_{j,t_{k}:n}(d))\right)^2\right]\\
            \leq &B(d_0+\delta+\bar a)^2 + B(10\bar a^2+4\bar a^4\beta^2C_{Dual}+2C_{off,1}) + 2\sqrt{B(d_0+\delta+\bar a)^2}\sqrt{B(10\bar a^2+4\bar a^4\beta^2C_{Dual}+2C_{off,1})}\\
            \equiv & B C_{off,2}
        \end{aligned}
    \end{equation*}
    Thus, there exists a constant $C_{off}$, $\mathbb{E}\left[\left(\Tilde{d}_{j,k+1}(d)-d\right)^2\right]\leq\frac{C_{off}}{(K-k-1)^2B}$. Using the same way, we can establish the result of $\hat{d}_{j,k+1}(d)$, and the proof is complete.
\end{proof}
\subsection{Proof of Lemma \ref{lem: DARes}}
\begin{proof}
For simplicity, define $A_t(p)=a_t\mathbbm{1}\{\Delta_t(p)>0\}$ throughout this proof. By definition,
\begin{align}
    d'_{k+1}-d_0&=d'_k-d_0+\frac{\sum_{t=t_{k}+1}^{t_{k+1}}d'_k-A_t(p_{k+1})}{n-t_{k+1}}\mathbbm{1}\{k<\bar{\kappa}\}\nonumber\\
    &=\underbrace{d'_k-d_0}_{\text{(i)}}+\underbrace{\frac{\sum_{t=t_{k}+1}^{t_{k+1}}d'_k-A_t(p^*(d'_k))}{n-t_{k+1}}\mathbbm{1}\{k<\bar{\kappa}\}}_{\text{(ii)}}+\underbrace{\frac{\sum_{t=t_{k}+1}^{t_{k+1}}A_t(p^*(d'_k))-A_t(p_{k+1})}{n-t_{k+1}}\mathbbm{1}\{k<\bar{\kappa}\}}_{\text{(iii)}}\nonumber
\end{align}
We are interested in $\mathbb{E}[(d'_{k+1}(i)-d_0(i))^2]$ for $1\leq i\leq m$, where $d'_{k+1}(i)$ and $d_0(i)$ are the $i$th component of the vectors $d'_{k+1}$ and $d_0$. Since the analysis for each dimension $i$ is the same, we can focus on $d'_k(1)$'s. For simplicity, we abuse the notation to write $d'_k(1)$ as $d'_k$. In the following, we analyze the second moment of (ii) and (iii), and the expectation of the crossing terms between each two of (i), (ii), and (iii). First, for the second moment of (ii),  
\begin{equation*}\begin{aligned}
    &\mathbb{E}\left[\left(\sum_{t=t_{k}+1}^{t_{k+1}}d'_k-A_t(p^*(d'_k))\right)^2\mathbbm{1}\{k<\bar{\kappa}\}\right]\\
    =&\mathbb{E}\left[\sum_{i=t_{k}+1}^{t_{k+1}}\sum_{j=t_{k}+1}^{t_{k+1}}\left(\mathbb{E}\left[\left(d'_k-A_i(p^*(d'_k))\right)\left(d'_k-A_j(p^*(d'_k))\right)\bigg\rvert d'_k\right]\mathbbm{1}\{k<\bar{\kappa}\}\right)\right]
\end{aligned}\end{equation*}
For $i\neq j$,
\begin{equation*}\begin{aligned}
    &\mathbb{E}\left[\left(d'_k-A_i(p^*(d'_k))\right)\left(d'_k-A_j(p^*(d'_k))\right)\bigg\rvert d'_k\right]\mathbbm{1}\{k<\bar{\kappa}\}\\
    =&\mathbb{E}\left[\left(d'_k-A_i(p^*(d'_k))\right)\bigg\rvert d'_k\right]\mathbb{E}\left[\left(d'_k-A_j(p^*(d'_k))\right)\bigg\rvert d'_k\right]\mathbbm{1}\{k<\bar{\kappa}\}=0
\end{aligned}\end{equation*}
Thus,
\begin{equation*}\begin{aligned}
    \mathbb{E}\left[\frac{\left(\sum_{t=t_{k}+1}^{t_{k+1}}d'_k-A_t(p^*(d'_k))\right)^2}{(n-t_{k+1})^2}\mathbbm{1}\{k<\bar{\kappa}\}\right]=&\frac{1}{(n-t_{k+1})^2}\mathbb{E}\left[\left(\sum_{t=t_{k}+1}^{t_{k+1}}\mathbb{E}\left[(d'_k-A_t(p^*(d'_k)))^2\bigg\rvert d'_k\right]\right)\mathbbm{1}\{k<\bar{\kappa}\}\right]\\
    \leq& \frac{(\bar{d}+\bar{a})^2}{(K-k-1)^2B}
\end{aligned}\end{equation*}
Secondly, for the second moment of (iii),
\begin{equation*}\begin{aligned}
    &\mathbb{E}\left[\left(\sum_{t=t_{k}+1}^{t_{k+1}}A_t(p^*(d'_k))-A_t(p_{k+1})\right)^2\mathbbm{1}\{k<\bar{\kappa}\}\right]\\
    =&\mathbb{E}\left[\left(\sum_{i=t_{k}+1}^{t_{k+1}}\sum_{j=t_{k}+1}^{t_{k+1}}\mathbb{E}\left[\left(A_i(p^*(d'_k))-A_i(p_{k+1})\right)\left(A_j(p^*(d'_k))-A_j(p_{k+1})\right)\bigg\rvert d'_k\right]\right)\mathbbm{1}\{k<\bar{\kappa}\}\right]
\end{aligned}\end{equation*}
When $i=j$,
\begin{equation*}
    \mathbb{E}\left[\left(A_i(p^*(d'_k))-A_i(p_{k+1})\right)\left(A_j(p^*(d'_k))-A_j(p_{k+1})\right)\bigg\rvert d'_k\right]\leq\bar{a}^2.
\end{equation*}
When $i\neq j$, we note that, if $(r_1,a_1)\ind (r_2,a_2)\ind (p,\Tilde{d})$, then
\begin{equation}\label{eqn: AuxThm1(e)}
    \begin{aligned}
        &\mathbb{E}\left[\left(a_1\left(\mathbbm{1}\{\Delta_1(p^*(\Tilde{d}))>0\}-\mathbbm{1}\{\Delta_1(p)>0\}\right)\right)\left(a_2\left(\mathbbm{1}\{\Delta_2(p^*(\Tilde{d}))>0\}-\mathbbm{1}\{\Delta_2(p)>0\}\right)\right)\bigg\rvert \Tilde{d}\right]\\
        =&\mathbb{E}\left[a_1\mathbb{E}\left[\left(\mathbbm{1}\{\Delta_1(p^*(\Tilde{d}))>0\}-\mathbbm{1}\{\Delta_1(p)>0\}\right)\bigg\rvert a_1,a_2,p,\Tilde{d}\right]a_2\mathbb{E}\left[\left(\mathbbm{1}\{\Delta_2(p^*(\Tilde{d}))>0\}-\mathbbm{1}\{\Delta_2(p)>0\}\right)\bigg\rvert a_1,a_2,p,\Tilde{d}\right]\bigg\rvert\Tilde{d}\right]\\
        \leq &\beta^2\bar{a}^4\mathbb{E}\left[\abs{p-p^*(\Tilde{d})}^2\bigg\rvert\Tilde{d}\right]
        \end{aligned}
\end{equation}
Thus,
\begin{equation*}\begin{aligned}
    &\mathbb{E}\left[\left(A_i(p^*(d'_k))-A_i(p_{k+1})\right)\left(A_j(p^*(d'_k))-A_j(p_{k+1})\right)\bigg\rvert d'_k\right]\leq \mathbb{E}\left[\beta^2\bar{a}^4\abs{p_{k+1}-p^*(d'_k)\}}^2\bigg\rvert d'_k\right]
\end{aligned}\end{equation*}
Then,
\begin{equation*}
    \begin{aligned}
        &\mathbb{E}\left[\left(\sum_{t=t_{k}+1}^{t_{k+1}}A_t(p^*(d'_k))-A_t(p_{k+1})\right)^2\mathbbm{1}\{k<\bar{\kappa}\}\right]\\
        \leq &\mathbb{E}\left[\left(B\bar{a}^2+B^2\mathbb{E}\left[\beta^2\bar{a}^4\abs{p_{k+1}-p^*(d'_k)\}}^2\bigg\rvert d'_k\right]\right)\mathbbm{1}\{k<\bar{\kappa}\}\right]\\
        \leq &B\bar{a}^2+B^2\mathbb{E}\left[\beta^2\bar{a}^4\abs{p_{k+1}-p^*(d'_k)\}}^2\mathbbm{1}\{k<\bar{\kappa}\}\right]\\
        \leq & B\bar{a}^2+B^2\mathbb{E}\left[\beta^2\bar{a}^4\sup_{d\in\Omega_d}\abs{p_{(k+1)B}^*(d)-p^*(d)\}}^2\right]\leq  B\bar{a}^2+\beta^2\bar{a}^4B^2\frac{C_{Dual}}{(k+1)B}\leq B\left(\bar{a}^2+\beta^2\bar{a}^4C_{Dual}\right)
    \end{aligned}
\end{equation*}
Thus, 
\begin{equation*}
    \begin{aligned}
        \mathbb{E}\left[\frac{\left(\sum_{t=t_{k}+1}^{t_{k+1}}A_t(p^*(d'_k))-A_t(p_{k+1})\right)^2}{(n-t_{k+1})^2}\mathbbm{1}\{k<\bar{\kappa}\}\right]\leq \frac{\bar{a}^2+\beta^2\bar{a}^4C_{Dual}}{(K-k-1)^2B}
    \end{aligned}
\end{equation*}
Thirdly, for the expectation of the crossing term between (i) and (ii)
\begin{equation*}\begin{aligned}
    &2\mathbb{E}\left[(d'_k-d_0)\left(\frac{\sum_{t=t_{k}+1}^{t_{k+1}}d'_k-A_t(p^*(d'_k))}{n-t_{k+1}}\right)\mathbbm{1}\{k<\bar{\kappa}\}\right]= 2\mathbb{E}\left[\frac{d'_k-d_0}{n-t_{k+1}}\sum_{t=t_{k}+1}^{t_{k+1}}\mathbb{E}\left[\left(d'_k-A_t(p^*(d'_k))\right)\mathbbm{1}\{k<\bar{\kappa}\}\bigg\rvert d'_k\right]\right]= 0
\end{aligned}\end{equation*}
Fourthly, for the expectation of the crossing term between (i) and (iii). When $k=0$, since $d'_0 = d_0$, $\mathbb{E}\left[(d'_k-d_0)\left(\frac{\sum_{t=t_{k}+1}^{t_{k+1}}A_t(p^*(d'_k))-A_t(p_{k+1})}{n-t_{k+1}}\right)\mathbbm{1}\{k<\bar{\kappa}\}\right]=0$. When $k\geq 1$,
\begin{equation*}\begin{aligned}
    &\mathbb{E}\left[(d'_k-d_0)\left(\sum_{t=t_{k}+1}^{t_{k+1}}A_t(p^*(d'_k))-A_t(p_{k+1})\right)\mathbbm{1}\{k<\bar{\kappa}\}\right]= \sum_{t=t_{k}+1}^{t_{k+1}}\mathbb{E}\left[(d'_k-d_0)\mathbbm{1}\{k<\bar{\kappa}\}\left(A_t(p^*(d'_k))-A_t(p_{k+1})\right)\right]
\end{aligned}\end{equation*}
and
\begin{equation*}\begin{aligned}
    &\mathbb{E}\left[(d'_k-d_0)\mathbbm{1}\{k<\bar{\kappa}\}\left(A_t(p^*(d'_k))-A_t(p_{k+1})\right)\right]\\
    =&\mathbb{E}\left[(d'_k-d_0)\mathbbm{1}\{k<\bar{\kappa}\}a_t\mathbb{E}\left[\left(\mathbbm{1}\{\Delta_t(p^*(d'_k))>0\}-\mathbbm{1}\{\Delta_t(p_{k+1})>0\}\right)\bigg\rvert a_t,p_{k+1},d'_k\right]\right]\\
   \leq &\beta\bar{a}^2\mathbb{E}\left[(d'_k-d)\mathbbm{1}\{k<\bar{\kappa}\}\abs{p_{k+1}-p^*(d'_k)}\right]\\
   \leq &\beta\bar{a}^2\sqrt{\mathbb{E}\left[(d'_k-d)^2\right]}\sqrt{\mathbb{E}\left[\abs{p_{k+1}-p^*(d'_k)}^2\mathbbm{1}\{k<\bar{\kappa}\}\right]}\leq \frac{\sqrt{2}\beta\bar{a}^2\sqrt{C_{Dual}}}{\sqrt{(k+1)B}}\sqrt{\mathbb{E}\left[(d'_k-d)^2\right]}
\end{aligned}\end{equation*}
Thus, 
\begin{equation*}\begin{aligned}
    2\mathbb{E}\left[(d'_k-d_0)\left(\frac{\sum_{t=t_{k}+1}^{t_{k+1}}A_t(p^*(d'_k))-A_t(p_{k+1})}{n-t_{k+1}}\right)\mathbbm{1}\{k<\bar{\kappa}\}\right]\leq \frac{2\sqrt{2}\beta\bar{a}^2\sqrt{C_{Dual}}}{(K-k-1)\sqrt{k+1}\sqrt{B}}\sqrt{\mathbb{E}[(d'_k-d_0)^2]}
\end{aligned}\end{equation*}
Fifthly, by Cauchy-Schwarz Inequality and the upper bounds for the second moment of (i) and (ii), there exists a constant $C_{Res,3}$ such that
\begin{equation*}\begin{aligned}
    &\mathbb{E}\left[\left(\frac{\sum_{t=t_{k}+1}^{t_{k+1}}d'_k-A_t(p^*(d'_k))}{n-t_{k+1}}\mathbbm{1}\{k<\bar{\kappa}\}\right)\left(\frac{\sum_{t=t_{k}+1}^{t_{k+1}}A_t(p^*(d'_k))-A_t(p_{k+1})}{n-t_{k+1}}\right)\mathbbm{1}\{k<\bar{\kappa}\}\right]\\
    \leq &\sqrt{\frac{(\bar{d}+\bar{a})^2}{(K-k-1)^2B}}\sqrt{\frac{\bar{a}^2+\beta^2\bar{a}^4C_{Dual}}{(K-k-1)^2B}}= \frac{\sqrt{(\bar{d}+\bar{a})^2(\bar{a}^2+\beta^2\bar{a}^4C_{Dual})}}{(K-k-1)^2B}
\end{aligned}\end{equation*}
To summarize, there exists a constant $C_{Res,0}$ such that
\begin{equation*}
    \mathbb{E}\left[\abs{d'_{k+1}-d_0}^2\right]\leq \mathbb{E}\left[\abs{d'_k-d_0}^2\right] + \frac{C_{Res,0}}{(K-k-1)^2B}+\frac{\sqrt{C_{Res,0}}}{(K-k-1)\sqrt{k+1}\sqrt{B}}\sqrt{\mathbb{E}\left[\abs{d'_k-d_0}^2\right]}\quad\forall k=0,\cdots,K-2
\end{equation*}
By Lemma \ref{lem:Induction}, there exists a constant $C_{Res}=8C_{Res,0}$ such that
\begin{equation*}
    \sum_{k=1}^{K-1}\mathbb{E}\left[\abs{d'_k-d_0}^2\right]\leq\frac{C_{Res}\log K}{B}
\end{equation*}
\end{proof}
\begin{lem}\label{lem:Induction}
    If $z_0=0$, and
    \begin{equation*}
        z_{k+1}\leq z_k + \frac{C}{(K-k-1)^2}+\frac{\sqrt{C}}{(K-k-1)\sqrt{k+1}}\sqrt{z_k}\quad\forall k=0,\cdots,K-2
    \end{equation*}
    Then,
    \begin{equation*}
        \sum_{k=1}^{K-1}z_k\leq 8C\log K
    \end{equation*}
\end{lem}
\begin{proof}
    It suffices to show the result under the following assumption.
    \begin{equation*}
        z_{k+1}=z_k+C\left(\frac{1}{(K-k-1)^2}\right)+\sqrt{C}\frac{\sqrt{z_k}}{(K-k-1)\sqrt{k+1}}
    \end{equation*}
    Then,
    \begin{equation*}\begin{aligned}
        \sum_{k=0}^{K-2}(K-k-1)(z_{k+1}-z_k)=C\sum_{k=0}^{K-2}\left(\frac{1}{K-k-1}\right)+\sqrt{C}\sum_{k=0}^{K-2}\frac{\sqrt{z_k}}{\sqrt{k+1}}\leq 4C\log(K)+\sqrt{C}\sum_{k=0}^{K-2}\frac{\sqrt{z_k}}{\sqrt{k+1}}
    \end{aligned}\end{equation*}
    and 
    \begin{equation*}\begin{aligned}
        \sum_{k=0}^{K-2}(K-k-1)(z_{k+1}-z_k)=\sum_{k=1}^{K-1}z_k
    \end{aligned}\end{equation*}
    In addition,
    \begin{equation*}\begin{aligned}
        \left(\sqrt{C}\sum_{k=0}^{K-2}\frac{\sqrt{z_k}}{\sqrt{k+1}}\right)^2\leq & C\left(\sum_{k=0}^{K-2}z_k\right)\left(\sum_{k=0}^{K-2}\frac{1}{k+1}\right)\leq 2C\log(K)\left(\sum_{k=1}^{K-1}z_k\right)\leq 2C\log(K)\left(4C\log(K)+\sqrt{C}\sum_{k=0}^{K-2}\frac{\sqrt{z_k}}{\sqrt{k+1}}\right)
    \end{aligned}\end{equation*}
    This implies that
    \begin{equation*}
        \sqrt{C}\sum_{k=0}^{K-2}\frac{\sqrt{z_k}}{\sqrt{k+1}}\leq 4C\log(K)
    \end{equation*}
    Thus,
    \begin{equation*}
        \sum_{k=1}^{K-1}z_k\leq 8C\log(K)
    \end{equation*}
\end{proof}
\subsection{Proof of Lemma~\ref{Lemma: MultiRegretComp}}
\begin{proof}
Fix $j$ and $k$, for each $1\leq i\leq m$, define $\tilde{d}_k(i)=\min\{\max\{d'_k(i),d_0(i)-\delta_d\},d_0(i)+\delta_d\}$, and $\hat{p}_k=\arg\min_{p\geq 0}{\tilde d_{k-1}}^Tp+\frac{1}{kB}\sum_{j=-B+1}^{t_{k-1}}(\Delta_j(p))^+$. Then, $\tilde{d}_k\mathbbm{1}\{k<\bar\kappa\}=d'_k\mathbbm{1}\{k<\bar\kappa\}=d_k\mathbbm{1}\{k<\bar\kappa\}$. Then, $\mathbb{E}\left[\mathbbm{1}\{k<\bar\kappa\}u_{j,k}(p_k,d_k)\right]\leq \mathbb{E}\left[u_{j,k}(\hat{p}_k,\tilde d_k)\right]$ by the fact that $u_{j,k}(\hat{p}_k,\tilde d_k)\geq 0$ almost surely. Define
\begin{equation*}
    \begin{aligned}
        &u^{1}_{j,k}(p,d)=\Delta_j(p_{kB:n}^*(d)\left(\mathbbm{1}\{{a_j}^Tp_{kB:n}^*(d)<r_j\leq {a_j}^Tp\}\right)\\
        &u^{2}_{j,k}(p,d)=-\Delta_j(p_{kB:n}^*(d)\left(\mathbbm{1}\{{a_j}^Tp<r_j\leq {a_j}^Tp_{kB:n}^*(d)\}\right)
    \end{aligned}
\end{equation*}
Then, $u_{j,k}(\hat{p}_k,\tilde d_k)=u^{1}_{j,k}(\hat{p}_k,\tilde d_k)+u^{2}_{j,k}(\hat{p}_k,\tilde d_k)$. Note that there is a weak dependence between $(r_j,a_j)$ and $\Tilde{d}_{k}$. To deal with this weak dependence, we define the Leave-One-Out average remaining resource $\hat{d}_{j,k}$ such that for all $1\leq i\leq m$,
\begin{equation*}\begin{aligned}
    \hat{d}_{j,k}(i) =\mathbbm{1}\{k\leq\bar{\kappa}\}\min\left\{\max\left\{\left(d'_{k}+\frac{a_j\mathbbm{1}\{r_j>{a_j}^T\hat{p}_k\}}{(K-k)B}\right)(i),d_0(i)-\delta_d\right\},{e_i}^Td_0+\delta_d\right\}+\Tilde{d}_{k-1}(i)\mathbbm{1}\{k>\bar{\kappa}\} 
\end{aligned}\end{equation*}
Then, $\abs{\hat{d}_{j,k}(i)-{\Tilde{d}_{k}}(i)}\leq\frac{\bar{a}}{(K-k)B}$ almost surely, and by Lemma \ref{lem: BoundedLipschitz}, $\norm{p^*(\hat{d}_{j,k})-p^*(\Tilde{d}_{k})}_2\leq \frac{\bar a\sqrt{mL}}{(K-k)B}$ almost surely. Thus,
\begin{equation*}
    \begin{aligned}
        {a_j}^Tp_{kB:n}^*(\Tilde{d}_{k})=&{a_j}^Tp_{kB:n}^*(\hat{d}_{j,k})-{a_j}^T(p_{kB:n}^*(\hat{d}_{j,k})-p^*(\hat{d}_{j,k}))-{a_j}^T(p^*(\hat{d}_{j,k})-p^*(\Tilde{d}_{k}))-{a_j}^T(p^*(\Tilde{d}_{k})-p_{(kB:n}^*(\Tilde{d}_{k}))\\
        \geq &{a_j}^Tp_{kB:n}^*(\hat{d}_{j,k})-2\bar{a}\left(\sup_{d\in\Omega_d}\norm{p_{kB:n}^*(d)-p^*(d)}_2\right)-\frac{\bar{a}^2\sqrt{mL}}{(K-k)B}
    \end{aligned}
\end{equation*}
Thus,
\begin{equation*}
    \begin{aligned}
        \mathbb{E}\left[u^{1}_{j,k}(\hat{p}_k,\tilde d_k)\right]\leq &\mathbb{E}\left[\left({a_j}^T(\hat{p}_k-p_{kB:n}^*(\hat{d}_{j,k}))+2\bar{a}\left(\sup_{d\in\Omega_d}\norm{p_{kB:n}^*(d)-p^*(d)}_2\right)+\frac{\bar{a}^2\sqrt{mL}}{(K-k)B}\right)\cdot\right.\\
        &\left.\mathbbm{1}\left\{{a_j}^Tp_{kB:n}^*(\hat{d}_{j,k})-2\bar{a}\left(\sup_{d\in\Omega_d}\norm{p_{kB:n}^*(d)-p^*(d)}_2\right)-\frac{\bar{a}^2\sqrt{mL}}{(K-k)B}\leq r_j\leq {a_j}^T\hat{p}_k\right\}\right]
    \end{aligned}
\end{equation*}
By our construction of $\hat{p}_k$ and $\hat{d}_{j,k}$, we have $(r_j,a_j)\ind \left(\hat{p}_k, p_{kB:n}^*(\hat{d}_{j,k}),\sup_{d\in\Omega_d}\norm{p_{kB:n}^*(d)-p^*(d)}_2\right)$. Then, by Assumption \ref{Assumption: DistributionDensity}(b),
\begin{equation}\label{eqn: MultiMypoic}
    \begin{aligned}
        \mathbb{E}\left[u^{1}_{t,k}(\hat{p}_k)\right]\leq &\mathbb{E}\left[\beta\left(\bar{a}\norm{\hat{p}_k-p^*_{kB:n}(\hat{d}_{j,k})}_2+2\bar{a}\left(\sup_{d\in\Omega_d}\norm{p_{kB:n}^*(d)-p^*(d)}_2\right)+\frac{\bar{a}^2\sqrt{mL}}{(K-k)B}\right)^2\right]
    \end{aligned}
\end{equation}
By Lemma \ref{lem: UniformDualConvergence}, $\mathbb{E}\left[\sup_{d\in\Omega_d}\norm{p_{kB:n}^*(d)-p^*(d)}_2^2\right]\leq \frac{C_{Dual}}{(K-k)B}$. In addition,
\begin{equation*}
    \begin{aligned}
        \norm{\hat{p}_k-p^*_{kB:n}(\hat{d}_{j,k})}_2^2\leq \left(\norm{\hat{p}_k-p^*(\Tilde{d}_{k-1})}_2+\norm{p^*(\Tilde{d}_{k-1})-p^*(\Tilde{d}_{k})}_2+\norm{p^*(\Tilde{d}_{k})-p^*(\hat{d}_{j,k})}_2+\norm{p^*(\hat{d}_{j,k})-p^*_{kB:n}(\hat{d}_{j,k})}_2\right)^2
    \end{aligned}
\end{equation*}
By Lemma~\ref{lem: BoundedLipschitz} and Lemma \ref{lem: UniformDualConvergence}, $\mathbb{E}\left[\norm{\hat{p}_k-p^*(\Tilde{d}_{k-1})}_2^2\right]\leq \frac{C_{Dual}}{kB}$, $\mathbb{E}\left[\norm{p^*(\hat{d}_{j,k})-p^*_{kB:n}(\hat{d}_{j,k})}_2^2\right]\leq\frac{C_{Dual}}{(K-k)B}$, and $\mathbb{E}\left[\norm{p^*(\Tilde{d}_{k})-p^*(\hat{d}_{j,k})}_2^2\right]\leq \frac{Lm\bar{a}^2}{(K-k)^2B^2}$. Furthermore,
\begin{equation*}
    \mathbb{E}\left[\norm{p^*(\Tilde{d}_{k-1})-p^*(\Tilde{d}_{k})}_2^2\right]\leq L\mathbb{E}\left[\norm{\Tilde{d}_{k}-\Tilde{d}_{k-1}}_2^2\right]\leq L\mathbb{E}\left[\norm{d'_{k}-d'_{k-1}}_2^2\right] 
\end{equation*}
By definition,
\begin{equation*}
    d'_{k}-d'_{k-1} = \frac{\sum_{j=t_{k-1}+1}^{t_{k}}d'_{k-1}-a_j\mathbbm{1}\{r_j>{a_j}^Tp^*(d'_{k-1})\}+\sum_{j=t_{k-1}+1}^{t_{k}}a_j(\mathbbm{1}\{r_j>{a_j}^Tp^*(d'_{k-1})\}-\mathbbm{1}\{r_j>{a_j}^T\hat{p}_k\})}{(K-k)B}\mathbbm{1}\{k\leq \bar{\kappa}\}
\end{equation*}
Following the analysis related component (ii) and (iii) in Lemma \ref{lem: DARes}, there exists a constant $C_{UB,1}$ such that $\mathbb{E}\left[\norm{d'_{k}-d'_{k-1}}_2^2\right]\leq\frac{C_{UB,1}}{(K-k)^2B}$. Thus, there exists a constant $C_{UB,2}$ such that
\begin{equation*}
    \mathbb{E}\left[\norm{\hat{p}_k-p^*_{kB:n}(\hat{d}_{j,k})}_2^2\right]\leq \frac{C_{UB,2}}{kB} + \frac{C_{UB,2}}{(K-k)B}
\end{equation*}
Put all these upper bounds into (\ref{eqn: MultiMypoic}), together with Cauchy-Schwarz Inequality, we can conclude that there exists a constant $C_{UB,3}$ such that,
\begin{equation*}
    \mathbb{E}\left[u^{1}_{t,k}(\hat{p}_k)\right]\leq \frac{C_{UB,3}}{kB} + \frac{C_{UB,3}}{(K-k)B}
\end{equation*}
The same conclusion can be made to $\mathbb{E}\left[u^{2}_{t,k}(\hat{p}_k)\right]$. Then, the proof is complete.
\end{proof}
\subsection{Proof of Lemma~\ref{lem:Leave-One-Out}}
\begin{proof}
    Given $d\in\Omega(\delta_d)$, $N>J$ and $\mathcal{J}\subseteq [N]$ with $\abs{\mathcal{J}}=J$. By Assumption \ref{Assumption: LOO}, for each order $j$, there exists $i$ and a scalar $c_j$ such that $a_j=c_j\alpha_i$, and we set $g_j=i$, where $g_j$ represents the type of order $j$. For all $i$, let $N_0=0$ and $N_i=N_{i-1}+\sum_{j=1}^N\mathbbm{1}\{g_j=i\}$ for $1\leq i\leq I$. Then, let $\pi(j)$ be a permutation of $[N]$ such that $g_{\pi(j)}=i$ if $N_{i-1}<j\leq N_i$ and $\frac{r_{\pi(j_1)}}{c_{\pi(j_1)}}\geq\frac{r_{\pi(j_2)}}{c_{\pi(j_2)}}$ if $g_{\pi(j_1)}=g_{\pi(j_2)}$. In addition, Let $\tilde{\mathcal{J}}=\{j:\pi(j)\in \mathcal{J}\}$ and let $\left\{x_j^*\right\}_{j\in [N]}$ and $\left\{\tilde x_j^*\right\}_{j\in [N]\backslash\mathcal{J}}$ be the optimal solution of the primal problem corresponds to the problem defining $p_N^*(d)$ and $p_{\mathcal{J},N}(d)$. Define
    \begin{equation*}
        \begin{aligned}
            &\tau_i = \max\left\{N_{i-1}\leq t\leq N_i:\alpha_i\sum_{j=N_{i-1}+1}^{t}c_{\pi(j)}\leq\alpha_i \sum_{j\in[N]}c_j x_j^*\mathbbm{1}\{g_j=i\}\right\}\quad\forall 1\leq i\leq I\\
            &\tilde\tau_i=\max\left\{N_{i-1}\leq t\leq N_i:\alpha_i\sum_{j=N_{i-1}+1}^{t}c_{\pi(j)}\mathbbm{1}\{j\notin\tilde{\mathcal{J}}\}\leq \alpha_i\sum_{j\in[N]\backslash \mathcal{J}} \tilde c_j x_j^*\mathbbm{1}\{g_j=i\}\right\}\quad\forall 1\leq i\leq I
        \end{aligned}
    \end{equation*}
    Then, for any $j\in [N]$, $\mathbbm{1}\{r_j>{a_j^Tp_N^*(d)}\}\neq \mathbbm{1}\{r_j>{a_j^T\tilde p_{\mathcal{J},N}(d)}\}$ implies that $\pi^{-1}(j)\in \cup_{i=1}^I (\tau_i\wedge\tilde\tau_i,\tau_i\vee\tilde\tau_i]$. For all $1\leq i\leq I$, we have
    \begin{equation*}
        \begin{aligned}
            &\norm{\alpha_i\sum_{j=N_{i-1}+1}^{\tau_i}c_{\pi(j)}-\alpha_i\sum_{j=N_{i-1}+1}^{\tilde\tau_i}c_{\pi(j)}\mathbbm{1}\{j\notin\tilde{\mathcal{J}}\}}_{\infty}\\
            \leq&\norm{\alpha_i\left(\sum_{j\in[N]} c_jx_j^*\mathbbm{1}\{g_j=i\}-\sum_{j\in[N]\backslash \mathcal{J}} c_j\tilde x_j^*\mathbbm{1}\{g_j=i\}\right)}_{\infty}+2\bar{a}
        \end{aligned}
    \end{equation*}
    In addition, on the event $E=\left\{p_N^*(d)>0\right\}\cap \left\{\tilde p_{\mathcal{J},N}>0\right\}$, i.e., all the components of $p_N^*(d)$ and $p_{\mathcal{J},N}(d)$ are positive,   
    \begin{equation*}
        \sum_{i=1}^I\alpha_i\left(\sum_{j\in[N]} c_jx_j^*\mathbbm{1}\{g_j=i\}-\sum_{j\in[N]\backslash \mathcal{J}} c_j\tilde x_j^*\mathbbm{1}\{g_j=i\}\right)=Jd
    \end{equation*}
    Because $\{\alpha_i\}_{i=1}^I$ are linear independent and $\Omega(\delta_d)$ is a bounded set, there exists a constant $C$ that only depends on $J$, the joint distribution of $(r_1,a_1)$ and $d_0$ such that
    \begin{equation*}
        \abs{\sum_{j=N_{i-1}+1}^{\tau_i}c_{\pi(j)}-\sum_{j=N_{i-1}+1}^{\tilde\tau_i}c_{\pi(j)}\mathbbm{1}\{j\notin\tilde{\mathcal{J}}\}}\leq C\quad\forall i
    \end{equation*}
    By Assumption \ref{Assumption: DistributionMulti}, there exists $\underline{{c}}>0$ that only depends on the distribution of $a_1$ such that $c_j\geq \underline{c}>0$ almost surely for all $j$. Thus, there exists a constant $\hat{C}_{LOO,J}$ that only depends on $J$, the joint distribution of $(r_1,a_1)$ and $d_0$ such that
    \begin{equation*}
        \tau_i\vee\tilde\tau_i-\tau_i\wedge\tilde\tau_i\leq \hat{C}_{LOO,J}\quad\forall i
    \end{equation*}
    Then,
    \begin{equation*}
        \begin{aligned}
            \mathbb{P}\left\{\mathbbm{1}\{r_j>{a_j^Tp_N^*(d)}\}\neq \mathbbm{1}\{r_j>{a_j^T\tilde p_{\mathcal{J},N}(d)}\}\right\}\leq \mathbb{P}\{\pi^{-1}(j)\in \cup_{i=1}^I(\tau_i-\hat{C}_{LOO,J},\tau_i+\hat{C}_{LOO,J}]\} + \mathbb{P}(E^c)
        \end{aligned}
    \end{equation*}
    We have
    \begin{equation*}
        \begin{aligned}
            &\mathbb{P}\{\pi^{-1}(j)\in \cup_{i=1}^I(\tau_i-\hat{C}_{LOO,J},\tau_i+\hat{C}_{LOO,J}]\}\\ = &\mathbb{E}\left[\mathbb{P}\left\{\pi^{-1}(j)\in \cup_{i=1}^I(\tau_i-\hat{C}_{LOO,J},\tau_i+\hat{C}_{LOO,J}]\bigg\rvert \left(r_{\pi(1)},a_{\pi(1)}\right),\cdots,\left(r_{\pi(N)},a_{\pi(N)}\right)\right\}\right]\leq \frac{2I\hat{C}_{LOO,J}}{N}
        \end{aligned}
    \end{equation*}
    In addition, We can assume $\delta_d$ is selected small enough such that there exists $\epsilon_{\delta_d}>0$ such that every component of $p$ is positive if $\norm{p-p^*(d_0)}<\epsilon_{\delta_d}$ and $\norm{p^*(d)-p^*(d_0)}_2<\frac{\epsilon_{\delta_d}}{2}$. Then,
    \begin{equation}\label{eqn: priceallpositive}
        \mathbb{P}\left\{p_N^*(d)>0\right\}\geq 1 - \mathbb{P}\left\{\sup_{d\in\Omega(\delta_d)}\norm{p_N^*(d)-p^*(d)}_2>\frac{\epsilon_{\delta_d}}{2}\right\}\geq 1-\frac{4C_{Dual}}{N\epsilon_{\delta_d}^2}
    \end{equation}
    Thus, $\mathbb{P}(E^c)\leq \frac{C_{J}}{N}$, where $C_j$ only depends on $J$, the joint distribution of $(r_1,a_1)$ and $d_0$. Thus, 
    \begin{equation*}
        \begin{aligned}
            \mathbb{P}\left\{\mathbbm{1}\{r_j>{a_j^Tp_N^*(d)}\}\neq \mathbbm{1}\{r_j>{a_j^T\tilde p_{\mathcal{J},N}(d)}\}\right\}\leq \frac{2I\hat{C}_{LOO,J}}{N} + \frac{C_J}{N}
        \end{aligned}
    \end{equation*}
    which completes the proof.
\end{proof}
\subsection{Proof of Proposition~\ref{Prop:RegretBoundDA_2}}
We only need to show that Lemma \ref{lem: DARes} and Lemma \ref{Lemma: MultiRegretComp} also hold for the modified algorithm $\tilde\pi_1$. All the analyses related to Batch $2,\cdots,K-1$ in the proof of Lemma \ref{lem: DARes} and Lemma \ref{Lemma: MultiRegretComp} for Algorithm \ref{alg: AhdLAMulti} can be directly extended to the analyses for the modified algorithm $\tilde\pi_1$. The only difference is that the sample size of $\tilde p_k$ in $\tilde\pi_1$ is $(k-1)B$ but the sample size of $p_k$ in Algorithm \ref{alg: AhdLAMulti} is $kB$; however, this difference will only lead to a larger constant in the final regret bound. For argument related to the first batch, we only need to focus on the case $B\geq 3$ as when $B\leq 2$ the contribution of the first batch to the final regret is at most a constant. We apply Lemma \ref{lem:Leave-One-Out} to deal with the dependence between $\tilde p_1$ and $\{(r_j,a_j)\}_{j=1}^B$. For Lemma \ref{lem: DARes}, we only need deal with is the second moment of (iii) when $k=0$. When $i=j$, we still have
\begin{equation*}
    \mathbb{E}\left[\left(A_i(p^*(d'_k))-A_i(p_{k+1})\right)\left(A_j(p^*(d'_k))-A_j(p_{k+1})\right)\bigg\rvert d'_k\right]\leq\bar{a}^2.
\end{equation*}
When $i\neq j$, define
\begin{equation*}
    \begin{aligned}
        \hat{p}_{1,i,j} = \arg\min_{p\geq 0} {d_0}^Tp+\frac{1}{B-2}\left(\sum_{t=1}^B(\Delta_t(p))^+-(\Delta_i(p))^+-(\Delta_j(p))^+\right)
    \end{aligned}
\end{equation*}
Then,
\begin{equation*}\begin{aligned}
    &\mathbb{E}\left[\left(A_i(p^*(d_0))-A_i(\tilde p_{1})\right)\left(A_j(p^*(d_0))-A_j(\tilde p_{1})\right)\right]\\
    \leq &\bar{a}^2\mathbb{P}\left\{\mathbbm{1}\{\Delta_i(\tilde p_1)>0\}\neq \mathbbm{1}\{\Delta_i(\hat p_{1,i,j})>0\}\right\} + \bar{a}^2\mathbb{P}\left\{\mathbbm{1}\{\Delta_j(\tilde p_1)>0\}\neq \mathbbm{1}\{\Delta_j(\hat p_{1,i,j})>0\}\right\}\\
    &+\mathbb{E}\left[\abs{\left(A_i(p^*(d_0))-A_i(\hat p_{1,i,j})\right)}\abs{\left(A_j(p^*(d_0))-A_j(\hat p_{1,i,j})\right)}\right]
\end{aligned}\end{equation*}
For terms in the first line of the right-hand-side of the above inequality, they can be upper bounded using Lemma \ref{lem:Leave-One-Out}. For terms in the second line, because $(r_i,a_i)\ind (r_j,a_j)\ind \hat{p}_{1,i,j}$, the same analysis for Algorithm \ref{alg: AhdLAMulti} applies. Thus,
\begin{equation*}
    \begin{aligned}
        &\mathbb{E}\left[\left(A_i(p^*(d_0))-A_i(\tilde p_{1})\right)\left(A_j(p^*(d_0))-A_j(\tilde p_{1})\right)\right]\\
        \leq & \frac{2\bar a^2 C_{LOO,2}}{B}+\mathbb{E}\left[\beta^2\bar a^4\norm{\hat p_{1,i,j}- p^*(d_0)}^2\right]\leq \frac{2\bar a^2 C_{LOO,2}+3\beta^2\bar a^4 C_{Dual}}{B} 
    \end{aligned}
\end{equation*}
Thus, $\mathbb{E}\left[\frac{\left(\sum_{t=1}^{t_{1}}A_t(p^*(d_0))-A_t(\tilde p_{1})\right)^2}{(n-t_1)^2}\right]\leq \frac{\bar a^2+2\bar a^2 C_{LOO,2}+3\beta^2\bar a^4 C_{Dual}}{(K-1)^2B}$. Although we have a larger constant, the dependence of the second moment of (iii) on $K,\;k$ and $B$ does not change when $k=0$. Thus, Lemma \ref{lem: DARes} still holds for the modified algorithm $\tilde\pi_1$. For Lemma \ref{Lemma: MultiRegretComp}, on the event $\{\tilde p_1>0\}$, i.e., every component of $\tilde p_1$ is positive, then by LP Duality Theory, $\norm{Bd_0-\sum_{j=1}^{t_1}a_j\mathbbm{1}\{\Delta_j(\tilde p_1)>0\}}_2\leq m\bar a$. Then, for any $j$,
\begin{equation*}
    \begin{aligned}
        {a_j}^Tp_{B:n}^*(\Tilde{d}_{1})\mathbbm{1}\{\tilde p_1>0\}
        \geq \left({a_j}^Tp_{B:n}^*(d_0)-2\bar{a}\left(\sup_{d\in\Omega_d}\norm{p_{B:n}^*(d)-p^*(d)}_2\right)-\frac{\bar{a}^2m\sqrt{L}}{(K-1)B}\right)\mathbbm{1}\{\tilde p_1>0\}
    \end{aligned}
\end{equation*}
In addition, define 
\begin{equation*}
    \tilde p_{1,j} = \arg\min_{p\geq 0} {d_0}^Tp+\frac{1}{B-1}\left(\sum_{t=1}^B(\Delta_t(p))^+-(\Delta_j(p))^+\right)
\end{equation*}
Then, 
\begin{equation*}
    \begin{aligned}
        \mathbb{E}\left[u^{1}_{j,1}(\tilde {p}_1, d_0)\right]\leq& \mathbb{E}\left[u^{1}_{j,1}(\tilde {p}_{1,j}, d_0)\mathbbm{1}\left\{\mathbbm{1}\{\Delta_j(\tilde p_1)>0\}= \mathbbm{1}\{\Delta_j(\tilde p_{1,j})>0\}\right\}\mathbbm{1}\{\tilde p_1>0\}\right]\\
        &+ \left(\bar r+\frac{\bar r\bar a}{\underline{a}}\right)\mathbb{P}\left\{\mathbbm{1}\{\Delta_j(\tilde p_1)>0\}\neq \mathbbm{1}\{\Delta_j(\tilde p_{1,j})>0\}\right\}+\left(\bar r+\frac{\bar r\bar a}{\underline{ a}}\right)\left(1-\mathbb{P}\{\tilde p_1 > 0\}\right)\\
        \leq &\mathbb{E}\left[\left({a_j}^T(\tilde {p}_{1,j}-p_{B:n}^*({d}_{0}))+2\bar{a}\left(\sup_{d\in\Omega_d}\norm{p_{B:n}^*(d)-p^*(d)}_2\right)+\frac{\bar{a}^2m\sqrt{L}}{(K-1)B}\right)\cdot\right.\\
        &\left.\mathbbm{1}\left\{{a_j}^Tp_{B:n}^*({d}_{0})-2\bar{a}\left(\sup_{d\in\Omega_d}\norm{p_{B:n}^*(d)-p^*(d)}_2\right)-\frac{\bar{a}^2m\sqrt{L}}{(K-1)B}\leq r_j\leq {a_j}^T\tilde{p}_{1,j}\right\}\right]\\
        &+ \left(\bar r+\frac{\bar r\bar a}{\underline{a}}\right)\mathbb{P}\left\{\mathbbm{1}\{\Delta_j(\tilde p_1)>0\}\neq \mathbbm{1}\{\Delta_j(\tilde p_{1,j})>0\}\right\}+\left(\bar r+\frac{\bar r\bar a}{\underline{ a}}\right)\left(1-\mathbb{P}\{\tilde p_1 > 0\}\right)
    \end{aligned}
\end{equation*}
Because $\tilde p_{1,j}\ind (r_j,a_j)$, the same analysis for Algorithm \ref{alg: AhdLAMulti} can be applied to the first term on the right-hand-side of the above inequality. The second term can be handled by Lemma \ref{lem:Leave-One-Out}. For the last term, by (\ref{eqn: priceallpositive}), $1-\mathbb{P}\{\tilde p_1 > 0\}\leq \frac{4C_{Dual}}{B\epsilon_{\delta_d}^2}$. Thus, there exists a constant $\tilde C_{UB}$ such that  
\begin{equation*}
    \begin{aligned}
        \mathbb{E}\left[u^{1}_{j,1}(\tilde {p}_1, d_0)\right]\leq \frac{\tilde C_{UB}}{B} + \frac{\tilde C_{UB}}{(K-1)B} 
    \end{aligned}
\end{equation*}
Thus, Lemma \ref{Lemma: MultiRegretComp} still holds for the modified algorithm $\tilde\pi_1$.
\subsection{Proof of Proposition \ref{Proposition: Infrequent1}}
We use the way proposed in \cite{li2019olp} to upper bound the regret. Let $b_0=nd_0$ be the total resources available at the beginning of the planning horizon. Since $p^*(d_0)$ is a feasible solution to the offline dual problem, 
\begin{equation*}
    \mathbb{E}[R_n^*]\leq {b_0}^Tp^* + \mathbb{E}\left[\sum_{j=1}^n(\Delta_j(p^*))^+\right]
\end{equation*}
Let $b_n=b_0-\sum_{j=1}^na_jx_j^{\hat\pi_1}$ be the remaining resource of $\hat\pi_1$ at the end of the planning horizon. Then,
\begin{equation*}
    \mathbb{E}[R_n(\hat\pi_1)]=\mathbb{E}\left[{b_0}^Tp^*(d_0)+\sum_{j=1}^n\Delta_j(p^*)x_j^{\hat\pi_1}\right]-\mathbb{E}[{b_n}^Tp^*(d_0)]
\end{equation*}
Define $d'_j(\hat\pi_1)=\frac{b_0-\sum_{t=1}^ja_j\mathbbm{1}\{\Delta_t(\tilde p)>0\}}{n-j}$ and $\tau=\min\{j\geq 1:d'_j(\hat\pi_1)\notin\Omega(\delta_d)\}\wedge n-B$ and $u^{\hat\pi_1}_{j}(p) = \Delta_j\left(p\right)\left(\mathbbm{1}\left\{\Delta_j\left(p\right)>0\right\}-\mathbbm{1}\{\Delta_j(\tilde p)>0\}\right)$. We start our discussion from the case $B>\left\lceil\frac{\bar a}{\underline{ d}}\right\rceil$,
\begin{equation}\label{eqn: ModifyGUB}
    \begin{aligned}
        Regret_n(\hat\pi_1)\leq& \mathbb{E}\left[\sum_{j=1}^{n-B}\mathbbm{1}\{j\leq\tau\}u^{\hat\pi_1}_j(p^*)\right] + \left(\bar r + \frac{\bar r\bar a}{\underline{a}}\right)\mathbb{E}\left[n-B-\tau\right]\\
        &+\mathbb{E}\left[\sum_{j=n-B+1}^{n}\Delta_j(p^*)\left(\mathbbm{1}\{\Delta_j(p^*)>0\}-x_j^{\hat\pi_1}\right)\right] + \mathbb{E}\left[{b_n}^Tp^*(d_0)\right]
    \end{aligned}
\end{equation}
For the first term on the right-hand-side of (\ref{eqn: ModifyGUB}), since $u^{\hat\pi_1}_j\geq 0$ almost surely,
\begin{equation}\label{eqn: ModifyUB1}
    \begin{aligned}
        \mathbb{E}\left[\sum_{j=1}^{n-B}\mathbbm{1}\{j\leq\tau\}u^{\hat\pi_1}_j(p^*)\right]\leq & \mathbb{E}\left[\sum_{j=1}^{n-B}u^{\hat\pi_1}_j(p^*)\right]\leq(K-1)B\beta\bar a^2\mathbb{E}\left[\norm{\tilde p-p^*(d_0)}_2^2\right]\leq \beta\bar a^2 C_{Dual}
    \end{aligned}
\end{equation}
For the second term on the right-hand-side of (\ref{eqn: ModifyGUB}), define $A_j(p)=a_j{\bf 1}\{\Delta_j(p)>0\}$ and $\mu(p)=E[A_1(p)]$. In addition, define $M_{\ell}=\sum_{t=1}^{\ell}\left(A_t(\tilde p)-\mu(\tilde p)\right)$. Conditional on \(\tilde p\), \(\{M_{\ell,i}\}_{\ell=1}^n\) is a martingale adapt to the natural filtration generated by $\{(r_j,a_j)\}_{j=1}^n$. In addition, since its increments are bounded by $2\bar a$ almost surely, $E[M^2_{\ell}|\tilde p]\leq l\bar a^2$. Furthermore,
\begin{equation*}
    \sum_{t=1}^{\ell} A_t(\tilde p) - ld_0 = M_l + l(\mu(\tilde p)-d_0)
\end{equation*}
Let $d_0(i)$, $M_{l,i}$ and $\mu_i(\tilde p)$ be the $i$th component of $d_0$, $M_l$ and $\mu(\tilde p)$, then
\begin{equation*}
    \mathbb{P}\{\tau\leq j\}\leq \sum_{i=1}^m \mathbb{P}\left\{\sup_{\ell\leq j}|M_{\ell,i}|\geq\frac{(n-j)\delta_d}{2}\right\} + \sum_{i=1}^m \mathbb{P}\left\{j|\mu_i(\tilde p)-d_0(i)|\geq\frac{(n-j)\delta_d}{2}\right\}
\end{equation*}
By Doob's inequality,
\begin{equation*}
    \begin{aligned}
        \mathbb{P}\left\{\sup_{\ell\leq j}|M_{\ell,i}|\geq\frac{(n-j)\delta_d}{2}\right\} = \mathbb{E}\left[\mathbb{P}\left\{\sup_{\ell\leq j}|M_{\ell,i}|\geq\frac{(n-j)\delta_d}{2}\bigg\rvert \tilde p\right\}\right]\leq \frac{4j\bar a^2}{(n-j)^2\delta_d^2}
    \end{aligned}
\end{equation*}
In addition,
\begin{equation*}
    \begin{aligned}
        \mathbb{P}\left\{j|\mu_i(\tilde p)-d_0(i)|\geq\frac{(n-j)\delta_d}{2}\right\}&\leq \frac{4j^2\mathbb{E}[|\mu_i(\tilde p)-d_0(i)|^2]}{(n-j)^2\delta_d^2}\leq \frac{4\beta^2\bar a^4j^2\mathbb{E}[\norm{\tilde p-p^*(d_0)}_2^2]}{(n-j)^2\delta_d^2}\leq \frac{4C_{Dual}\beta^2\bar a^4n}{(n-j)^2\delta_d^2}
    \end{aligned}
\end{equation*}
Thus, there exists a constant $C_{\tau}$ that only depends on the joint distribution of $(r_1,a_1)$ and $d_0$ such that $\mathbb{P}(\tau\leq j)\leq \min\left\{1,\frac{C_{\tau}n}{(n-j)^2}\right\}$. Then, if $B\geq\sqrt{C_{\tau}n}$,
\begin{equation}\label{eqn: ModifyUB2}
    \mathbb{E}[n-B-\tau]\leq C_\tau n\sum_{h=B+1}^{+\infty}\frac{1}{h^2}\leq \frac{C_{\tau}n}{B} = C_{\tau}K
\end{equation}
If $B<\sqrt{C_{\tau}n}$,
\begin{equation}\label{eqn: ModifyUB3}
    \mathbb{E}[n-B-\tau]\leq\sum_{h=B+1}^{\lceil\sqrt{C_\tau n}\rceil}1
+
C_\tau n
\sum_{h=\lceil\sqrt{C_\tau n}\rceil+1}^{\infty}
\frac1{h^2}\leq 2\sqrt{C_{\tau}n}
\end{equation}
For the third term on the right-hand-side of (\ref{eqn: ModifyGUB}), let $\tilde d_{t_{K-1}}(\hat\pi_1)=d'_{t_{K-1}}(\hat\pi_1)\mathbbm{1}\{d'_{t_{K-1}}(\hat\pi_1)\in\Omega(\delta_d)\}+d_0\mathbbm{1}\{d'_{t_{K-1}}(\hat\pi_1)\notin\Omega(\delta_d)\}$. If $B\geq\sqrt{C_{\tau}n}$, define $p'_K=\arg\min_{p\geq 0} \tilde d_{t_{K-1}}(\hat\pi_1)^Tp+\frac{1}{B}\sum_{j=t_{K-1}+1}^{n}(\Delta_j(p))^+$, then
\begin{equation*}
    \begin{aligned}
        &\mathbb{E}\left[\sum_{j=n-B+1}^{n}\Delta_j(p^*)\left(\mathbbm{1}\{\Delta_j(p^*)>0\}-x_j^{\hat\pi_1}\right)\right]\\
        \leq & \left(\bar r+\frac{\bar r\bar a}{\underline{ a}}\right)B\mathbb{P}\left\{d'_{t_{K-1}}\notin\Omega(\delta_d)\right\}+m\bar r+\sum_{j=n-B+1}^n\mathbb{E}\left[\Delta_j(p^*)\left(\mathbbm{1}\{\Delta_j(p^*)>0\}-\mathbbm{1}\{\Delta_j(p'_K)>0\}\right)\right]
    \end{aligned}
\end{equation*}
By similar arguments in Lemma \ref{lem: DARes}, there exists a constant $\tilde C_{Res}$ such that $\mathbb{E}\left[\norm{d'_{t_{K-1}}-d_0}_2^2\right]\leq\frac{K\tilde C_{Res}}{B}$. Thus, $\left(\bar r+\frac{\bar r\bar a}{\underline{ a}}\right)B\mathbb{P}\left\{d'_{t_{K-1}}\notin\Omega(\delta_d)\right\}\leq \tilde C_{Res}\left(\bar r+\frac{\bar r\bar a}{\underline{ a}}\right)K$. Since $p'_K$ and $\{(r_j,a_j)\}_{j=n-B+1}^n$ are not independent, we apply the leave-some-out analysis here. Let $p'_{K,j} = \arg\min_{p\geq 0} \tilde d_{t_{K-1}}(\hat\pi_1)^Tp+\frac{1}{B-1}(\sum_{t=t_{K-1}+1}^{n}(\Delta_t(p))^+-(\Delta_j(p))^+)$. Then,
\begin{equation*}
    \begin{aligned}
        &\mathbb{E}\left[\Delta_j(p^*)\left(\mathbbm{1}\{\Delta_j(p^*)>0\}-\mathbbm{1}\{\Delta_j(p'_K)>0\}\right)\right]\\
        \leq & \left(\bar r+\frac{\bar r\bar a}{\underline{ a}}\right)\mathbb{P}\left\{\mathbbm{1}\{\Delta_j(p'_K)>0\}\neq\mathbbm{1}\{\Delta_j(p'_{K,j})>0\}\right\}\\
        &+\mathbb{E}\left[\Delta_j(p^*)\left(\mathbbm{1}\{\Delta_j(p^*)>0\}-\mathbbm{1}\{\Delta_j(p'_{K,j})>0\}\right)\right]
    \end{aligned}
\end{equation*}
By Lemma \ref{lem:Leave-One-Out} and similar arguments in the proof of Lemma \ref{Lemma: MultiRegretComp}, there exists a constant $\tilde C_{UB,1}$ such that $\mathbb{E}\left[\Delta_j(p^*)\left(\mathbbm{1}\{\Delta_j(p^*)>0\}-\mathbbm{1}\{\Delta_j(p'_K)>0\}\right)\right]\leq\frac{\tilde C_{UB,1} K}{B}$. Summarizing the above, If $B\geq\sqrt {C_{\tau}n}$, there exists a constant $\tilde C_{UB}$ such that
\begin{equation}\label{eqn: ModifyUB4}
    \mathbb{E}\left[\sum_{j=n-B+1}^{n}\Delta_j(p^*)\left(\mathbbm{1}\{\Delta_j(p^*)>0\}-x_j^{\hat\pi_1}\right)\right]\leq\tilde C_{UB}K
\end{equation}
If $B<\sqrt {C_{\tau}n}$,
\begin{equation}\label{eqn: ModifyUB5}
    \begin{aligned}
        \mathbb{E}\left[\sum_{j=n-B+1}^{n}\Delta_j(p^*)\left(\mathbbm{1}\{\Delta_j(p^*)>0\}-x_j^{\hat\pi_1}\right)\right]\leq\left(\bar r+\frac{\bar r\bar a}{\underline{ a}}\right)\sqrt {C_{\tau}n}
    \end{aligned}
\end{equation}
For the final term in the right-hand-side of (\ref{eqn: ModifyGUB}), if $B\geq\sqrt {C_{\tau}n}$, by complementary slackness, 
\begin{equation*}
 \mathbb{E}\left[b_n(i)\mathbbm{1}\{d'_{t_{K-1}}\in\Omega_d\}\mathbbm{1}\{p_K(i)>0\}\right]\leq m\bar a   
\end{equation*}
and by (\ref{eqn: priceallpositive}), $\mathbb{E}\left[b_n(i)\mathbbm{1}\{d'_{t_{K-1}}\in\Omega_d\}\mathbbm{1}\{p_K(i)=0\}\right]\leq nd_0(i)\frac{C_{Dual}}{B\epsilon_{\delta_d}^2}=\frac{C_{Dual}d_0(i)}{\epsilon_{\delta_d}^2}K$. In addition,
\begin{equation*}
    \begin{aligned}
        \mathbb{E}\left[b_n(i)\mathbbm{1}\{d'_{t_{K-1}}\notin\Omega_d\}\right]&\leq \mathbb{E}\left[(d_0(i)+\delta_d)(n-\tau+1)\mathbbm{1}\{d'_{t_{K-1}}\notin\Omega_d\}\right]\\
        &\leq (d_0(i)+\delta_d)\left(1+\mathbb{E}\left[n-\tau-B\right]\right)+B(d_0(i)+\delta_d)\mathbb{P}\left\{d'_{t_{K-1}}\notin\Omega(\delta_d)\right\}\\
        \leq &(d_0(i)+\delta_d)\left(1+\frac{2m\bar a^2}{\delta_d}K+\tilde C_{Res}K\right)
    \end{aligned} 
\end{equation*}
Thus, if $B\geq\sqrt {C_{\tau}n}$, there exists a constant $C_{Rem}$ such that
\begin{equation}\label{eqn: ModifyUB6}
    \mathbb{E}\left[{b_n}^Tp^*\right]\leq C_{Rem}K
\end{equation}
If $B<\sqrt {C_{\tau}n}$,
\begin{equation*}
    \begin{aligned}
        &\mathbb{E}\left[\abs{nd_0(i)-\sum_{j=1}^{n-B}a_j(i)\mathbbm{1}\{\Delta_j(\tilde p)\}}\right]\\
        \leq &Bd_0(i) + \sqrt{\mathbb{E}\left[\abs{(n-B)d_0(i)-\sum_{j=1}^{n-B}a_j(i)\mathbbm{1}\{\Delta_j(\tilde p)\}}^2\right]}\\
        \leq &Bd_0(i) + \sqrt{2n^2\mathbb{E}\left[(d_0(i)-\mu_i(\tilde p))^2\right]+2\mathbb{E}\left[\abs{\sum_{j=1}^{n-B}(a_j(i)\mathbbm{1}\{\Delta_j(\tilde p)\}-\mu_i(\tilde p))}^2\right]}\\
        \leq &\sqrt {C_{\tau}n}d_0(i) + \sqrt{2\beta^2\bar a^4 C_{Dual}n+2\bar a^2 n}
    \end{aligned}
\end{equation*}
Thus, if $B<\sqrt {C_{\tau}n}$, there exists a constant $\tilde C_{Rem}$ such that
\begin{equation}\label{eqn: ModifyUB7}
    \mathbb{E}\left[{b_n}^Tp^*\right]\leq \tilde C_{Rem}\sqrt{n}
\end{equation}
Combining (\ref{eqn: ModifyUB1}), (\ref{eqn: ModifyUB2}), (\ref{eqn: ModifyUB3}), (\ref{eqn: ModifyUB4}), (\ref{eqn: ModifyUB5}), (\ref{eqn: ModifyUB6}) and (\ref{eqn: ModifyUB7}) the proof is complete for the case $B>\left\lceil\frac{\bar a}{\underline{ d}}\right\rceil$. For the case $B\leq\left\lceil\frac{\bar a}{\underline{ d}}\right\rceil$, we modify the definition of $\tau$ to $\min\{j\geq 1:d'_j(\hat\pi_1)\notin\Omega(\delta_d)\}\wedge n-B-\left\lceil\frac{\bar a}{\underline{ d}}\right\rceil$ and all the arguments for the case $B<\sqrt{C_{\tau}n}$ still holds. Thus, the proof is complete. In addition, $\mathbb{E}\left[\norm{d'_{K-1}-d_0}_2^2\right]\leq\frac{K\tilde C_{Res}}{B}$ is the main reason why we get a $O(K^2)$ upper bound if we use the proof method for Theorem \ref{Thm:RegretBoundDAMulti}. 
\subsection{Proof of Theorem~\ref{Thm: RegretPA}}
Recall $\Omega(\delta_d)$ is the neighborhood given in Lemma \ref{lem: BoundedLipschitz} and Lemma \ref{lem: UniformDualConvergence}. Fix the length of the planning horizon $T$, and we start our discussion from the case $\lambda B>\frac{\bar{a}}{\underline{d}}$. To simplify the notation, we write $x^{\pi_2}_j$ as $x_j$ and $p^*$ as $p^*(d_0)$. For $k=0,\cdots,K-1$, define $d_k^{\pi_2} = \frac{b_{N(t_k)}}{\lambda(T-t_k)}$ which is the average remaining resource with respect to the expected number of customers to come. In addition, we define $\bar{\tau} = \inf\{t\geq 0:b_{N(t)}<\bar{a}\}$ and $\bar{\kappa} = \min\{k\geq 1: d_k^{\pi_2}\notin \Omega(\frac{\delta_d}{2})\}\wedge K$. We also define $d'_0=d_0^{\pi_2}$ and, for $k=0,\cdots,K-1$,
\begin{equation*}
    d'_{k+1} = \frac{\lambda(T-t_k)d'_k-\sum_{j=N(t_k)+1}^{N(t_{k+1})}a_j\mathbbm{1}\{r_j>a_jp_{k+1}\}}{\lambda(T-t_{k+1})}\mathbbm{1}\{k<\bar{\kappa}\} + d'_k\mathbbm{1}\{k\geq\bar{\kappa}\}
\end{equation*}
Then, $d'_k\mathbbm{1}\{k<\bar{\kappa}\} = d_k^{\pi_2}\mathbbm{1}\{k<\bar{\kappa}\}$. By definition of $\Tilde{d}_k$ in Algorithm \ref{alg: RaAhdLA}, $\Tilde{d}_k\mathbbm{1}\{k<\bar{\kappa}\} = \frac{\lambda}{\hat{\lambda}_k}d'_k\mathbbm{1}\{k<\bar{\kappa}\}$. By similar arguments in the proof of Proposition \ref{Proposition: Infrequent1}, define $u^{\pi_2}_j(p)=\Delta_j(p^*)\left(\mathbbm{1}\left\{\Delta_j(p^*)>0\right\}-\mathbbm{1}\left\{\Delta_j(p)>0\right\}\right)$ and $\tau = \min\{\bar{\tau},t_{\bar{\kappa}},t_{K-1}\}$ then
\begin{equation*}
    \begin{aligned}
        Regret_T(\pi_2)\leq& \underbrace{\mathbb{E}\left[\sum_{k=1}^{K-1}\sum_{j=N(t_{k-1})}^{N(t_k)}\mathbbm{1}\{j<N(\tau)\}u^{\pi_2}_{j}(p_{k})\right]+\mathbb{E}\left[\sum_{j=N(t_{K-1})+1}^{N(T)}\Delta_j(p^*)\left(\mathbbm{1}\{\Delta_j(p^*)>0\}-x_j\right)\right]}_{\text{(i)}}\\
        &+\underbrace{\left(\bar{r}+\frac{\bar{r}\bar{a}}{\underline{{d}}}\right)\mathbb{E}\left[N(t_{K-1})-N(\tau)\right]}_{\text{(ii)}}+\underbrace{\mathbb{E}\left[b_{N(T)}p^*\right]}_{\text{(iii)}}
    \end{aligned}
\end{equation*}
We provide upper bounds on these three terms on the right-hand-side in the next. For (i), for $k\leq K-1$, define
\begin{equation*}
    p'_k = p_k\mathbbm{1}\{k\leq\bar\kappa\} + \left(\arg\min_{p\geq 0} N(t_{k-1}){d_0}^Tp+\sum_{j=1}^{N(t_{k-1})}(\Delta_j(p))^+\right)\mathbbm{1}\{k>\bar\kappa\}
\end{equation*}
Together with the fact that $\tau\leq t_{K-1}\wedge t_{\bar\kappa}$ and $u_j^{\pi_2}(p'_k)\geq 0$ almost surely for all $j$ and k, 
\begin{equation*}
    \mathbb{E}\left[\sum_{k=1}^{K-1}\sum_{j=N(t_{k-1})}^{N(t_k)}\mathbbm{1}\{j<N(\tau)\}u^{\pi_2}_{j}(p_{k})\right]\leq\sum_{k=1}^{K-1}\mathbb{E}\left[\sum_{j=N(t_{K-1})}^{N(t_k)} u_j^{\pi_2}(p'_k)\right]
\end{equation*}
If $2\leq k\leq K-1$, 
\begin{equation*}
    \begin{aligned}
        \mathbb{E}\left[\sum_{j=N(t_{K-1})}^{N(t_k)+1} u_j^{\pi_2}(p'_k)\right]=&\mathbb{E}\left[\sum_{j=N(t_{K-1})+1}^{N(t_k)}\mathbb{E}\left[\mathbb{E}\left[u_j^{\pi_2}(p'_k)\bigg\rvert a_j, p'_k, N(t_{k-1}),N(t_k)\right]\bigg\rvert N(t_{k-1}),N(t_k)\right]\right]\\
        \leq & \beta\bar a^2\mathbb{E}\left[(N(t_k)-N(t_{k-1}))\norm{p'_k-p^*}_2^2\right]=\lambda B\beta\bar a^2 \mathbb{E}\left[\norm{p'_k-p^*}_2^2\right]
    \end{aligned}
\end{equation*}
where the last inequality is by the independence condition given in Assumption \ref{Assumption: Random Arrival} and similar arguments in the proof of Lemma \ref{Lemma: MultiRegretComp}. If $k=1$, define the leave-one-out price 
\begin{equation*}
    p'_{1,j} = \left(\arg\min_{p\geq 0} (N(t_1)-1){\hat{d_0}}^Tp + \left(\sum_{t=1}^{N(t_1)}(\Delta_t(p))^+ - (\Delta_j(p))^+\mathbbm{1}\{j\leq N(t_1)\}\right)\right)\mathbbm{1}\{N(t_1)\geq 2\}
\end{equation*}
Then,
\begin{equation*}
    \begin{aligned}
        \mathbb{E}\left[\sum_{j=1}^{N(t_1)} u_j^{\pi_2}(p'_k)\right]
        \leq & \beta\bar a^2\mathbb{E}\left[\sum_{j=1}^{N(t_1)}\norm{p'_{1,j}-p^*}_2^2\right]+\left(\bar r+\frac{\bar r\bar a}{\underline{ a}}\right)\mathbb{E}\left[\sum_{j=1}^{N(t_1)}\mathbbm{1}\left\{\mathbbm{1}\{\Delta_j(p'_1)>0\}\neq\mathbbm{1}\{\Delta_j(p'_{1,j})>0\}\right\}\right]
    \end{aligned}
\end{equation*}
To make the notation consistent, let $p'_K = p_K$. Then, define the leave-one-out price
\begin{equation*}
    p'_{K,j} = \left(\arg\min_{p\geq 0} {b_{N(t_{K-1})}}^Tp + \left(\sum_{t=N(t_{K-1})+1}^{N(T)}(\Delta_t(p))^+ - (\Delta_j(p))^+\mathbbm{1}\{N(t_{K-1})<j\leq N(T)\}\right)\right)\mathbbm{1}\{N(T)-N(t_{K-1})\geq 2\}
\end{equation*}
Then,
\begin{equation*}
    \begin{aligned}
        &\mathbb{E}\left[\sum_{j=N(t_{K-1})+1}^{N(T)}\Delta_j(p^*)\left(\mathbbm{1}\{\Delta_j(p^*)>0\}-x_j\right)\right]\\
        \leq &\beta\bar a^2\mathbb{E}\left[\sum_{j=N(t_{K-1})+1}^{N(T)}\norm{p'_{K,j}-p^*}_2^2\right]+\left(\bar r+\frac{\bar r\bar a}{\underline{ a}}\right)\mathbb{E}\left[\sum_{j=N(t_{K-1})+1}^{N(T)}\mathbbm{1}\left\{\mathbbm{1}\{\Delta_j(p'_K)>0\}\neq\mathbbm{1}\{\Delta_j(p'_{K,j})>0\}\right\}\right]
    \end{aligned}
\end{equation*}
To summarize,
\begin{equation}\label{eqn: PARegretComp1Summary}
    \begin{aligned}
        \text{(i)}
        \leq& \beta\bar{a}^2\left(\mathbb{E}\left[\sum_{j=1}^{N(t_1)}\norm{p'_{1,j}-p^*}_2^2\right]+\lambda B\sum_{k=2}^{K-1}\mathbb{E}\left[\norm{p'_k-p^*}_2^2\right]+\mathbb{E}\left[\sum_{j=N(t_{K-1})+1}^{N(T)}\norm{p'_{K,j}-p^*}_2^2\right]\right)\\
        &+\left(\bar r+\frac{\bar r\bar a}{\underline{ a}}\right)\mathbb{E}\left[\sum_{j=1}^{N(t_1)}\mathbbm{1}\left\{\mathbbm{1}\{\Delta_j(p'_1)>0\}\neq\mathbbm{1}\{\Delta_j(p'_{1,j})>0\}\right\}\right]\\
        &+\left(\bar r+\frac{\bar r\bar a}{\underline{ a}}\right)\mathbb{E}\left[\sum_{j=N(t_{K-1})+1}^{N(T)}\mathbbm{1}\left\{\mathbbm{1}\{\Delta_j(p'_K)>0\}\neq\mathbbm{1}\{\Delta_j(p'_{K,j})>0\}\right\}\right]
    \end{aligned}
\end{equation}
For the first term on the right-hand-side of (\ref{eqn: PARegretComp1Summary}), define $\hat d_k = \tilde d_k\mathbbm{1}\{k<\bar\kappa\} + d_0\mathbbm{1}\{k\geq\bar\kappa\}$ for $0\leq k < K-1$. Then
\begin{equation*}
    \begin{aligned}
    \mathbb{E}\left[\sum_{j=1}^{N(t_1)}\norm{p'_{1,j}-p^*}_2^2\right]\leq &\mathbb{E}\left[\sum_{j=1}^{N(t_1)}\norm{p'_{1,j}-p^*}_2^2\left(\mathbbm{1}\left\{\hat{d_0}\notin\Omega(\delta_d)\right\}+\mathbbm{1}\{N(t_1)\leq 1\}\right)\right]\\
    &+\mathbb{E}\left[\sum_{j=1}^{N(t_1)}\norm{\left(p'_{1,j}-p^*(\hat d_0)\right)+\left(p^*(\hat d_0)-p^*\right)}_2^2\mathbbm{1}\left\{\hat{d_0}\notin\Omega(\delta_d)\right\}\mathbbm{1}\{N(t_1)\geq 2\}\right]
\end{aligned}\end{equation*}
By Lemma \ref{lem: PAAux}(a), there exists a constant $C_0$ that only depends on the distribution of $(r_1,a_1)$ and $d_0$ such that
\begin{equation*}
    \begin{aligned}
        &\mathbb{E}\left[\sum_{j=1}^{N(t_1)}\norm{p'_{1,j}-p^*}_2^2\left(\mathbbm{1}\left\{\hat{d_0}\notin\Omega(\delta_d)\right\}+\mathbbm{1}\{N(t_1)\leq 1\}\right)\right]\leq \frac{\bar{r}^2}{\underline{a}^2}\left(\mathbb{E}\left[N(t_1)\mathbbm{1}\left\{\hat{d_0}\notin\Omega(\delta_d)\right\}+1\right]\right)\leq \frac{4\bar{r}^2}{\underline{a}^2}\left(C_0+1\right)
    \end{aligned}
\end{equation*}
By Lemma \ref{lem: UniformDualConvergence} and Assumption \ref{Assumption: Random Arrival},
\begin{equation*}
    \begin{aligned}
    &\mathbb{E}\left[\sum_{j=1}^{N(t_1)}\left(p'_{1,j}-p^*(\hat d_0)\right)^2\mathbbm{1}\left\{\hat{d_0}\notin\Omega(\delta_d)\right\}\mathbbm{1}\{N(t_1)\geq 2\}\right]\\
    \leq &\mathbb{E}\left[\sum_{j=1}^{N(t_1)}\mathbb{E}\left[\sup_{d\in\Omega_d}\left(p^*_{N(t_1)-1}(d)-p^*(d)\right)^2\bigg\rvert N(t_1)\right]\mathbbm{1}\{N(t_1)\geq 2\}\right]\\
    \leq & \mathbb{E}\left[\frac{C_{Dual}N(t_1)}{\max\{N(t_1)-1,1\}}I_1(j)\right]\leq 2C_{Dual}
\end{aligned}\end{equation*}
By Lemma~\ref{lem: BoundedLipschitz},
\begin{equation*}
    \begin{aligned}
    &\mathbb{E}\left[\sum_{j=1}^{N(t_1)}\left(p^*(\hat d_0)-p^*\right)^2\mathbbm{1}\left\{\hat{d_0}\notin\Omega(\delta_d)\right\}\mathbbm{1}\{N(t_1)\geq 2\}\right]\leq L\mathbb{E}\left[N(t_1)\norm{\hat d_0-d_0}_2^2\right]
\end{aligned}\end{equation*}
The above inequalities together with Cauchy-Schwarz Inequality show that there exists a constant $C_{1,1}$ that only depends on the distribution of $(r_1,a_1)$ and $d_0$ such that
\begin{equation}\label{eqn: PARegretComp1_1}
    \begin{aligned}
        \mathbb{E}\left[\sum_{j=1}^{N(t_1)}\norm{p'_{1,j}-p^*}_2^2\right]&\leq C_{1,1}+L\mathbb{E}\left[N(t_1)\norm{\hat d_0-d_0}_2^2\right]+2\sqrt{C_{1,1}}\sqrt{L\mathbb{E}\left[N(t_1)\norm{\hat d_0-d_0}_2^2\right]}
    \end{aligned}
\end{equation}
Similarly, for $2\leq k\leq K-1$,
\begin{equation*}
    \begin{aligned}
        \mathbb{E}\left[\norm{p'_k-p^*}_2^2\right]=&\mathbb{E}\left[\norm{p'_k-p^*}_2^2\left(\mathbbm{1}\left\{\hat{d}_{k-1}\notin\Omega(\delta_d)\right\}+\mathbbm{1}\{N(t_{k-1})= 0\}\right)\right]\\
        &+\mathbb{E}\left[\norm{\left(p'_{k}-p^*(\hat d_{k-1})\right)+\left(p^*(\hat d_{k-1})-p^*\right)}_2^2\mathbbm{1}\left\{\hat d_{k-1}\in\Omega(\delta_d)\right\}\mathbbm{1}\{N(t_{k-1})\geq 1\}\right]
    \end{aligned}
\end{equation*}
For the first term on the right-hand-side,
\begin{equation*}
    \begin{aligned}
        &\mathbb{E}\left[\norm{p'_k-p^*}_2^2\left(\mathbbm{1}\left\{\hat{d}_{k-1}\notin\Omega(\delta_d)\right\}+\mathbbm{1}\{N(t_{k-1})= 0\}\right)\right]\\
        \leq&\frac{4\bar r^2}{\underline{a}^2}\left(\mathbb{E}\left[\mathbbm{1}\left\{\norm{\Tilde{d}_{k-1}-d_{k-1}'}_{\infty}>\frac{\delta_d}{2}\right\}\mathbbm{1}\{k\leq\bar{\kappa}\}\right]+\mathbb{P}\{N(t_{k-1})=0\}\right)\\
        \leq & \frac{4\bar r^2}{\underline{a}^2}\left(\frac{4\bar{d}^2}{\delta_d^2}\mathbb{E}\left[\abs{\frac{\lambda}{\hat{\lambda}_{k-1}}-1}^2\right] + \mathbb{P}\{N(t_{k-1})=0\}\right)\leq  \frac{4\bar r^2}{\underline{a}^2}\left(\frac{4\bar{d}^2}{\delta_d^2}\frac{C_0}{\lambda B(k-1)} + \exp(-\lambda(k-1) B)\right)
    \end{aligned}
\end{equation*}
where the last inequality is by Lemma \ref{lem: PAAux}(b). For the second term on the right-hand-side, we apply the same arguments for the case $k=1$.
\begin{equation*}
    \begin{aligned}
        \mathbb{E}\left[\norm{p'_{k}-p^*(\hat d_{k-1})}_2^2\mathbbm{1}\left\{\hat d_{k-1}\in\Omega(\delta_d)\right\}\mathbbm{1}\{N(t_{k-1})\geq 1\}\right]\leq\mathbbm{E}\left[\frac{C_{Dual}}{\max\{N(t_{k-1}),1\}}\right]\leq 2C_{Dual}\mathbbm{E}\left[\frac{1}{N(t_{k-1})+1}\right]\leq \frac{2C_{Dual}}{\lambda B(k-1)}
    \end{aligned}
\end{equation*}
and 
\begin{equation*}
    \begin{aligned}
        \mathbb{E}\left[\norm{p^*(\hat d_{k-1})-p^*}_2^2\mathbbm{1}\left\{\hat d_{k-1}\in\Omega(\delta_d)\right\}\mathbbm{1}\{N(t_{k-1})\geq 1\}\right]\leq L\mathbb{E}\left[\norm{\hat d_{k-1}-d_0}_2^2\right]
    \end{aligned}
\end{equation*}
Then, there exists a constant $C_{1,2}$ that only depends on the distribution of $(r_1,a_1)$ and $d_0$ such that, for $2\leq k\leq K-1$,
\begin{equation}\label{eqn: PARegretComp1_2}
\begin{aligned}
    \lambda B\mathbb{E}\left[\norm{p'_k-p^*}_2^2\right]\leq \frac{C_{1,2}}{k}+\lambda B L\mathbb{E}\left[\norm{\hat d_{k-1}-d_0}_2^2\right]+2\sqrt{\frac{C_{1,2}}{k}}\sqrt{\lambda B L\mathbb{E}\left[\norm{\hat d_{k-1}-d_0}_2^2\right]}
\end{aligned}\end{equation}
For $k=K$, define $\hat d_{K-1}= \frac{b_{N(t_{K-1})}}{\max\{N(T)-N(t_{K-1}),1\}}$ and event $E_K = \left\{\bar{\kappa}=K\right\}\cap\left\{\hat d_{K-1}\in\Omega(\delta_d)\right\}\cap\left\{N(T)-N(t_{K-1})\geq 2\right\}$. Then,
\begin{equation*}
    \begin{aligned}
        \mathbb{E}\left[\sum_{j=N(t_{K-1})+1}^{N(T)}\left(p'_{K,j}-p^*\right)^2\right]\leq &\frac{4\bar{r}^2}{\underline{a}^2}\mathbb{E}\left[\left(N(T)-N(t_{K-1})\right)\mathbbm{1}_{E_K^c}\right]+\mathbb{E}\left[\sum_{j=N(t_{K-1})+1}^{N(T)}\mathbb{E}\left[\norm{p'_{K,j}-p^*}_2^2\mathbbm{1}_{E_K}\bigg\rvert N(t_{K-1}),N(T)\right]\right]
    \end{aligned}
\end{equation*}
By Lemma \ref{lem: PARes},
\begin{equation*}
    \begin{aligned}
        \mathbb{E}\left[(N(T)-N(t_{K-1}))\mathbbm{1}\{\bar{\kappa}<K\}\right]
        \leq\mathbb{E}\left[(N(T)-N(t_{K-1}))\right]\mathbb{E}\left[\mathbbm{1}\left\{d'_{K-1}\notin\Omega(\frac{\delta_d}{2})\right\}\right]
        \leq \frac{4M_{Res}\log K}{\delta_d^2}
    \end{aligned}
\end{equation*}
By the same proof of Lemma~\ref{lem: PAAux}(a), there exists a constant $M$ such that
\begin{equation*}
    \begin{aligned}
        &\mathbb{E}\left[(N(T)-N(t_{K-1}))\mathbbm{1}\{\bar{\kappa}=K\}\mathbbm{1}\{\hat d_{K-1}\notin\Omega(\delta_d)\}\right]\leq \mathbb{E}\left[(N(T)-N(t_{K-1}))\mathbbm{1}\{\bar{\kappa}=K\}\mathbbm{1}\left\{\norm{\hat{d}_{K-1}-d_{K-1}'}_{\infty}>\frac{\delta_d}{2}\right\}\right]\leq  M
    \end{aligned}
\end{equation*}
Thus, 
\begin{equation*}
    \begin{aligned}
        \frac{4\bar{r}^2}{\underline{a}^2}\mathbb{E}\left[\left(N(T)-N(t_{K-1})\right)\mathbbm{1}_{E_K^c}\right]\leq &\frac{4\bar{r}^2}{\underline{a}^2}\left(\frac{4M_{Res}\log K}{\delta_d^2}+M+1\right)
    \end{aligned}
\end{equation*}
By Lemma \ref{lem: UniformDualConvergence},
\begin{equation*}
    \begin{aligned}
        \mathbb{E}\left[\norm{p'_{K,j}-p^*(\hat{d}_{K-1})}_2^2\mathbbm{1}_{E_K}\bigg\rvert N(t_{K-1}),N(T)\right]\leq\frac{C_{Dual}}{\max\{N(T)-N(t_{K-1})-1,1\}}
    \end{aligned}
\end{equation*}
By Lemma \ref{lem: BoundedLipschitz},
\begin{equation*}
    \begin{aligned}
        \mathbb{E}\left[\norm{p^*(\hat{d}_{K-1})-p^*}_2^2\mathbbm{1}_{E_K}\bigg\rvert N(t_{K-1}),N(T)\right]\leq &L\mathbb{E}\left[\norm{\hat{d}_{K-1}-d_0}_2^2\mathbbm{1}_{E_K}\bigg\rvert N(t_{K-1}),N(T)\right]
    \end{aligned}
\end{equation*}
Thus, there exists a constant $C_{1,3}$ that only depends on the distribution of $(r_1,a_1)$ and $d_0$ such that
\begin{equation}\label{eqn: PARegretComp1_3}
    \begin{aligned}
        &\mathbb{E}\left[\sum_{j=N(t_{K-1})+1}^{N(T)}\norm{p'_{K,j}-p^*}_2^2\right]\\
        \leq &C_{1,3}\log K + C_{1,3} + L\mathbb{E}\left[(N(T)-N(t_{K-1}))(\hat{d}_{K-1}-d_0)^2\mathbbm{1}_{E_K}\right] + 2\sqrt{C_{1,3}}\sqrt{L\mathbb{E}\left[(N(T)-N(t_{K-1}))(\hat{d}_{K-1}-d_0)^2\mathbbm{1}_{E_K}\right]}
    \end{aligned}
\end{equation}
In addition,
\begin{equation*}\begin{aligned}
    &\mathbb{E}\left[\sum_{j=1}^{N(t_1)}\mathbbm{1}\left\{\mathbbm{1}\{\Delta_j(p'_1)>0\}\neq\mathbbm{1}\{\Delta_j(p'_{1,j})>0\}\right\}\right]\\
    \leq &\mathbb{E}\left[N(t_1)\mathbbm{1}\left\{\hat d_0\notin\Omega(\delta_d)\right\}+N(t_1)\mathbbm{1}\{N(t_1)\leq 1\}\right]\\
    &+\mathbb{E}\left[\sum_{j=1}^{N(t_1)}\mathbbm{1}\left\{\mathbbm{1}\{\Delta_j(p'_1)>0\}\neq\mathbbm{1}\{\Delta_j(p'_{1,j})>0\}\right\}\mathbbm{1}\left\{\hat d_0\notin\Omega(\delta_d)\right\}\mathbbm{1}\{N(t_1)\geq 2\}\right]
\end{aligned}\end{equation*}
By Lemma \ref{lem: PAAux}(a), $\mathbb{E}\left[N(t_1)\mathbbm{1}\left\{\hat d_0\notin\Omega(\delta_d)\right\}+N(t_1)\mathbbm{1}\{N(t_1)\leq 1\}\right]\leq C_0+1$. By Lemma \ref{lem:Leave-One-Out},
\begin{equation*}
    \begin{aligned}
        &\mathbb{E}\left[\sum_{j=1}^{N(t_1)}\mathbbm{1}\left\{\mathbbm{1}\{\Delta_j(p'_1)>0\}\neq\mathbbm{1}\{\Delta_j(p'_{1,j})>0\}\right\}\mathbbm{1}\left\{\hat d_0\notin\Omega(\delta_d)\right\}\mathbbm{1}\{N(t_1)\geq 2\}\right]\\
        =&\mathbb{E}\left[\sum_{j=1}^{N(t_1)}\mathbb{E}\left[\mathbbm{1}\left\{\mathbbm{1}\{\Delta_j(p'_1)>0\}\neq\mathbbm{1}\{\Delta_j(p'_{1,j})>0\}\right\}\right]\mathbbm{1}\left\{\hat d_0\notin\Omega(\delta_d)\right\}\mathbbm{1}\{N(t_1)\geq 2\}\bigg\rvert N(t_1)\right]\\
        \leq &C_{LOO,1}\mathbb{E}\left[\frac{N(t_1)}{\max\{N(t_1),1\}}\right]\leq C_{LOO,1}
    \end{aligned}
\end{equation*}
Similarly,
\begin{equation*}
    \begin{aligned}
        &\mathbb{E}\left[\sum_{j=N(t_{K-1})+1}^{N(T)}\mathbbm{1}\left\{\mathbbm{1}\{\Delta_j(p'_K)>0\}\neq\mathbbm{1}\{\Delta_j(p'_{K,j})>0\}\right\}\right]\\
        \leq &\mathbb{E}\left[\left(N(T)-N(t_{K-1})\right)\mathbbm{1}_{E_K^c}\right]\\
        &+\mathbb{E}\left[\sum_{j=N(t_{K-1})+1}^{N(T)}\mathbb{E}\left[\mathbbm{1}\left\{\mathbbm{1}\{\Delta_j(p'_K)>0\}\neq\mathbbm{1}\{\Delta_j(p'_{K,j})>0\}\right\}\bigg\rvert N(t_{K-1}),N(T),b_{N(t_{K-1})}\right]\mathbbm{1}_{E_K}\right]
    \end{aligned}
\end{equation*}
By the same proof for (\ref{eqn: PARegretComp1_3}), $\mathbb{E}\left[\left(N(T)-N(t_{K-1})\right)\mathbbm{1}_{E_K^c}\right]\leq \frac{4M_{Res}\log K}{\delta_d^2}+M+1$, and by Lemma \ref{lem:Leave-One-Out},
\begin{equation*}
    \begin{aligned}
        &\mathbb{E}\left[\sum_{j=N(t_{K-1})+1}^{N(T)}\mathbb{E}\left[\mathbbm{1}\left\{\mathbbm{1}\{\Delta_j(p'_K)>0\}\neq\mathbbm{1}\{\Delta_j(p'_{K,j})>0\}\right\}\bigg\rvert N(t_{K-1}),N(T),b_{N(t_{K-1})}\right]\mathbbm{1}_{E_K}\right]\leq C_{LOO,1}
    \end{aligned}
\end{equation*}
Then, there exists a constant $C_{1,4}$ that only depends on the distribution of $(r_1,a_1)$ and $d_0$ such that the sum last two terms on the right-hand-side of (\ref{eqn: PARegretComp1Summary}) is smaller than $C_{1,4}$. Together with (\ref{eqn: PARegretComp1_1}), (\ref{eqn: PARegretComp1_2}), (\ref{eqn: PARegretComp1_3}) and Cauchy Inequality, to show there exists a constant $C_1$ that only depends on the distribution of $(r_1,a_1)$ and $d_0$ such that $\text{(i)}\leq C_1\log K$, it is sufficient to show there exists a constant $C_{1,5}$ that only depends on the distribution of $(r_1,a_1)$ and $d_0$ such that
\begin{equation}\label{eqn: PARegretComp1_Final}
    \mathbb{E}\left[N(t_1)\norm{\hat d_0-d_0}_2^2\right]+\lambda B \sum_{k=2}^{K-1}\mathbb{E}\left[\norm{\hat d_{k-1}-d_0}_2^2\right]+\mathbb{E}\left[(N(T)-N(t_{K-1}))(\hat{d}_{K-1}-d_0)^2\mathbbm{1}_{E_K}\right]\leq C_{1,5}\log K
\end{equation}
By Lemma~\ref{lem: PAAux} (c), $\mathbb{E}\left[N(t_1)\norm{\hat d_0-d_0}_2^2\right]\leq C_0$. Furthermore,
\begin{equation*}
    \begin{aligned}
        &\sum_{k=2}^{K-1}\mathbb{E}\left[\norm{\hat d_{k-1}-d_0}_2^2\right]\\
        \leq & \sum_{k=2}^{K-1}\mathbb{E}\left[\norm{\hat d_{k-1}-d'_{k-1}}_2^2\mathbbm{1}\{k\leq\bar\kappa\}\right] + \sum_{k=2}^{K-1}\mathbb{E}\left[\norm{d'_{k-1}-d_0}_2^2\right]+2\sum_{k=2}^K\sqrt{\mathbb{E}\left[\norm{\hat d_{k-1}-d'_{k-1}}_2^2\mathbbm{1}\{k\leq\bar\kappa\}\right]}\sqrt{\mathbb{E}\left[\norm{d'_{k-1}-d_0}_2^2\right]}
    \end{aligned}
\end{equation*}
By Lemma~\ref{lem: PAAux}(b),
\begin{equation*}
    \begin{aligned}
        \sum_{k=2}^{K-1}\mathbb{E}\left[\norm{\hat d_{k-1}-d'_{k-1}}_2^2\mathbbm{1}\{k\leq\bar\kappa\}\right]\leq m \bar{d}^2\sum_{k=2}^{K-1}\mathbb{E}\left[\abs{\frac{\lambda}{\hat{\lambda}_{k-1}}-1}^2\right]\leq \frac{2m\bar d^2 C_0\log K}{\lambda B} 
    \end{aligned}
\end{equation*}
Together with Lemma \ref{lem: PARes},
\begin{equation*}
    \begin{aligned}
        \lambda B \sum_{k=2}^{K-1}\mathbb{E}\left[\norm{\hat d_{k-1}-d_0}_2^2\right]\leq \left(2m\bar d^2 C_0+M_{Res}+2\sqrt{2m\bar d^2C_0}\sqrt{M_{Res}}\right)\log K
    \end{aligned}
\end{equation*}
In addition,
\begin{equation*}
    \begin{aligned}
        &\mathbb{E}\left[(N(T)-N(t_{K-1}))(\hat{d}_{K-1}-d_0)^2\mathbbm{1}_{E_K}\right]\\
        \leq  & \mathbb{E}\left[(N(T)-N(t_{K-1}))(\hat{d}_{K-1}-d'_{K-1})^2\mathbbm{1}_{E_K}\right]+\mathbb{E}\left[(N(T)-N(t_{K-1}))(d'_{K-1}-d_0)^2\mathbbm{1}_{E_K}\right]\\
        &+2\sqrt{\mathbb{E}\left[(N(T)-N(t_{K-1}))(\hat{d}_{K-1}-d'_{K-1})^2\mathbbm{1}_{E_K}\right]}\sqrt{\mathbb{E}\left[(N(T)-N(t_{K-1}))(d'_{K-1}-d_0)^2\mathbbm{1}_{E_K}\right]}
    \end{aligned}
\end{equation*}
By Lemma \ref{lem: PARes},
\begin{equation*}
    \begin{aligned}
        \mathbb{E}\left[(N(T)-N(t_{K-1}))\left(d'_{K-1}-d_0\right)^2\mathbbm{1}_{E_K}\right]\leq \mathbb{E}\left[N(T)-N(t_{K-1})\right]\mathbb{E}\left[\left(d'_{K-1}-d_0\right)^2\right]\leq M_{Res}\log K
    \end{aligned}
\end{equation*}
By the same proof of Lemma \ref{lem: PAAux}(c), $\mathbb{E}\left[(N(T)-N(t_{K-1}))\left(\hat{d}_{K-1}-d'_{K-1}\right)^2\mathbbm{1}_{E_K}\right]\leq C_0$. Thus,
\begin{equation*}
    \begin{aligned}
        \mathbb{E}\left[(N(T)-N(t_{K-1}))(\hat{d}_{K-1}-d_0)^2\mathbbm{1}_{E_K}\right]\leq& M_{Res}\log K + C_0 + 2\sqrt{C_0}\sqrt{M_{Res}\log K}
    \end{aligned}
\end{equation*}
Thus, we have proved (\ref{eqn: PARegretComp1_Final}). Therefore, $\text{(i)}\leq C_1\log K$. For part (ii), since $\tau$ is a stopping time adapted to the natural filtration of $((N(t),r_{N(t)},a_{N(t)}),t\geq 0)$, by the memoryless property of Poisson process, 
\begin{equation*}\begin{aligned}
    &\mathbb{E}\left[N(t_{K-1})-N(\tau)\right]\leq \lambda\mathbb{E}[t_{K-1}-\tau]\leq \lambda\mathbb{E}[t_{K-1}-t_{\bar{\kappa}-1}]= \lambda B\mathbb{E}[K-\kappa]
\end{aligned}\end{equation*}
where the second inequality comes from the fact that $t_{\bar{\kappa}-1}\leq\tau$ almost surely. Thus,
\begin{equation*}
    \mathbb{E}\left[N(t_{K-1})-N(\tau)\right]\leq \lambda B\mathbb{E}[K-\kappa]
\end{equation*}
By Lemma \ref{lem: PARes} and the same analysis of $\mathbb{E}[K-\kappa]$ in Theorem \ref{Thm:RegretBoundDAMulti}, there exists a constant $C_2$ that only depends on the distribution of $(r_1,a_1)$ and $d_0$ such that $\text{(ii)}\leq C_2\log K$. For part (iii),
\begin{equation*}
    \begin{aligned}
        \mathbb{E}[b_{N(T)}]\leq \mathbb{E}[d^{\pi_2}_{\bar{\kappa}}\lambda (T-t_{\bar{\kappa}})\mathbbm{1}\{\bar{\kappa}\leq K-1\}+b_{N(T)}\mathbbm{1}\{\bar{\kappa}= K\}]
    \end{aligned}
\end{equation*}
and there exists a constant $C_{3,1}$ that only depends on the distribution of $(r_1,a_1)$ and $d_0$ such that
\begin{equation*}
    \begin{aligned}
        \mathbb{E}[d^{\pi_2}_{\bar{\kappa}}\lambda (T-t_{\bar{\kappa}})\mathbbm{1}\{\bar{\kappa}\leq K-1\}]\leq &\mathbb{E}\left[d^{\pi_2}_{\bar{\kappa}-1}\frac{\lambda (T-t_{\bar{\kappa}-1})}{\lambda (T-t_{\bar{\kappa}})}\lambda (T-t_{\bar{\kappa}})\mathbbm{1}\{\bar{\kappa}\leq K-1\}\right]\\
        \leq &2\bar{d}\lambda B\mathbb{E}[K-\bar{\kappa}]\cdot e\leq C_{3,1}\log K\cdot e
    \end{aligned}
\end{equation*}
Recall $e$ is the vector in $\mathbb{R}^m$ with every components being $1$. In addition,
\begin{equation*}
    \begin{aligned}
        \mathbb{E}[b_{N(T)}\mathbbm{1}\{\bar{\kappa}=K\}] &\leq \mathbb{E}\left[b_{N(T)}\mathbbm{1}\{\bar{\kappa}=K\}\left(\mathbbm{1}\left\{\norm{\hat d_{K-1}-d'_{K-1}}_{+\infty}>\frac{\delta_d}{2}\right\}+\mathbbm{1}\{N(T)-N(t_{K-1})=0\}\right)\right]\\
        &+\mathbb{E}\left[b_{N(T)}\mathbbm{1}\{\bar{\kappa}=K\}\left(\mathbbm{1}\left\{\norm{\hat d_{K-1}-d'_{K-1}}_{+\infty}\leq\frac{\delta_d}{2}\right\}\mathbbm{1}\{N(T)-N(t_{K-1})\geq 1\}\right)\right]
    \end{aligned}
\end{equation*}
and
\begin{equation*}
    \begin{aligned}
        &\mathbb{E}\left[b_{N(T)}\mathbbm{1}\{\bar{\kappa}=K\}\left(\mathbbm{1}\left\{\norm{\hat d_{K-1}-d'_{K-1}}_{+\infty}>\frac{\delta_d}{2}\right\}+\mathbbm{1}\{N(T)-N(t_{K-1})=0\}\right)\right]\\
        \leq &\bar d\lambda B\left(m\mathbb{P}\left\{\abs{\frac{\lambda B}{\max\{N(T)-N(t_{K-1}),1\}}-1}>\frac{\delta_d}{2}\right\}+\mathbb{P}\left\{N(T)-N(t_{K-1})=0\right\}\right)\cdot e\\
        \leq &\bar d\lambda B\left(m\frac{4C_0}{\lambda B\delta_d^2}+\exp(-\lambda B)\right)\cdot e\leq \left(\frac{4mC_0\bar d}{\delta_d^2} + \bar d\exp(-1)\right)\cdot e
    \end{aligned}
\end{equation*}
where the last line is by the proof of Lemma \ref{lem: PAAux}(b) and select $\epsilon_p$ such that if $\norm{p-p^*}_2\leq \epsilon_p$ then every component of $p$ is positive, then by complementary slackness,
\begin{equation*}
    \mathbb{E}\left[b_{N(T)}\mathbbm{1}\{\bar{\kappa}=K\}\left(\mathbbm{1}\left\{\norm{\hat d_{K-1}-d'_{K-1}}_{+\infty}\leq\frac{\delta_d}{2}\right\}\mathbbm{1}\{N(T)-N(t_{K-1})\geq 1\}\mathbbm{1}\left\{\norm{p'_K-p*}_2\leq \epsilon_p\right\}\right)\right]\leq m\bar a\cdot e
\end{equation*}
and by Lemma \ref{lem: UniformDualConvergence},
\begin{equation*}
    \begin{aligned}
        &\mathbb{E}\left[b_{N(T)}\mathbbm{1}\{\bar{\kappa}=K\}\left(\mathbbm{1}\left\{\norm{\hat d_{K-1}-d'_{K-1}}_{+\infty}\leq\frac{\delta_d}{2}\right\}\mathbbm{1}\{N(T)-N(t_{K-1})\geq 1\}\mathbbm{1}\left\{\norm{p'_K-p*}_2> \epsilon_p\right\}\right)\right]\\
        \leq & \bar d\lambda B \mathbb{E}\left[\mathbb{E}\left[\mathbbm{1}\left\{\norm{p'_K-p*}_2> \epsilon_p\right\}\bigg\rvert b_{N(t_{K-1})},N(t_{K-1},N(t_K))\right]\mathbbm{1}\{\hat d_{K-1}\in\Omega(\delta_d)\}\mathbbm{1}\{N(T)-N(t_{K-1})\geq 1\}\right]\cdot e\\
        \leq &\frac{\bar d\lambda B}{\epsilon_p^2} \mathbb{E}\left[\mathbb{E}\left[\sup_d\in\Omega(d)\norm{p^*_{N(T)-N(t_{K-1})}(d)-p^*(d)}_2^2\bigg\rvert N(t_{K-1},N(t_K))\right]\mathbbm{1}\{N(T)-N(t_{K-1})\geq 1\}\right]\cdot e\\
        \leq & \frac{2\bar d\lambda B}{\epsilon_p^2}\mathbb{E}\left[\frac{1}{N(T)-N(t_{K-1})+1}\right]\cdot e = \frac{2\bar d}{\epsilon_p^2}\cdot e
    \end{aligned}
\end{equation*}
Summarize the above analyses, there exists a constant $C_{3,2}$ that only depends on the distribution of $(r_1,a_1)$ and $d_0$ such that
\begin{equation*}
    \mathbb{E}[b_{N(T)}\mathbbm{1}\{\bar{\kappa}=K\}]\leq C_{3,2}\cdot e
\end{equation*}
Together with Lemma \ref{lem: BoundedLipschitz}, we can conclude that there exists a constant $C_{3}$ that only depends on the distribution of $(r_1,a_1)$ and $d_0$ such that $\text{(iii)}\leq C_3\log K$. Our analyses on (i), (ii), (iii) complete the proof of Theorem~\ref{Thm: RegretPA} for the case $\lambda B>\frac{\bar a}{\underline{ d}}$. For the case $\lambda B\leq \frac{\bar a}{\underline{ d}}$, we can do the same analysis in Theorem \ref{Thm:RegretBoundDAMulti} and we omit the details.
\begin{lem}\label{lem: PARes}
    There exists a constant $M_{Res}$ that only depends on the distribution of $(r_1,a_1)$ and $d_0$ such that
    \begin{equation*}
        \sum_{k=1}^{K-1}\mathbb{E}\left[\abs{d'_k-d'_0}^2\right]\leq\frac{M_{Res}\log K}{\lambda B}
    \end{equation*}
\end{lem}
\begin{proof}
For simplicity, we only prove the case with a single type of resource. The proof for the case with multiple types of resources is almost the same. Throughout this section, define $A_j(p)=a_j\mathbbm{1}\{r_j>a_jp\}$. By the definition of $d'_k$, we have for $k\geq 0$,
\begin{equation}\label{eqn: PARex0}
\begin{aligned}
    d'_{k+1}-d'_0&=\left(d'_k-d'_0\right)+\underbrace{\frac{\lambda (t_{k+1}-t_k)-(N(t_{k+1})-N(t_k))}{\lambda (T-t_{k+1})}d'_k\mathbbm{1}\{k<\bar{\kappa}\}}_{\text{(i)}}\\
    &+\underbrace{\frac{\sum_{j=N(t_k)+1}^{N(t_{k+1})}d'_k-A_j(p^*(d'_{k}))}{\lambda (T-t_{k+1})}\mathbbm{1}\{k<\bar{\kappa}\}}_{\text{(ii)}}+\underbrace{\frac{\sum_{j=N(t_k)+1}^{N(t_{k+1})}A_j(p^*(d'_{k}))-A_j(p_{k+1})}{\lambda (T-t_{k+1})}\mathbbm{1}\{k<\bar{\kappa}\}}_{(iii)}
\end{aligned}\end{equation}
For the second moment of (i),
\begin{equation}\label{eqn: PARex1}
    \begin{aligned}
    &\mathbb{E}\left[\left(\frac{\lambda (t_{k+1}-t_k)-(N(t_{k+1})-N(t_k))}{\lambda (T-t_{k+1})}d'_k\mathbbm{1}\{k<\bar{\kappa}\}\right)^2\right]\\
    =&\mathbb{E}\left[\left(d'_k\mathbbm{1}\{k<\bar{\kappa}\}\right)^2\right]\mathbb{E}\left[\left(\frac{\lambda (t_{k+1}-t_k)-(N(t_{k+1})-N(t_k))}{\lambda (T-t_{k+1})}\right)^2\right]\\
    \leq&\frac{\bar{d}^2}{(\lambda B)^2(K-k-1)^2}Var(N(t_{k+1})-N(t_k))=\frac{\bar{d}^2}{(\lambda B)(K-k-1)^2}
\end{aligned}\end{equation}
For the crossing term between $d'_k-d'_0$ and (i),
\begin{equation}\label{eqn: PARex4}
\begin{aligned}
    &\mathbb{E}\left[(d'_k-d'_0)\left(\frac{\lambda (t_{k+1}-t_k)-(N(t_{k+1})-N(t_k))}{\lambda (T-t_{k+1})}d'_k\mathbbm{1}\{k<\bar{\kappa}\}\right)\right]\\
    =&\mathbb{E}\left[(d'_k-d'_0)d'_k\mathbbm{1}\{k<\bar{\kappa}\}\right]\mathbb{E}\left[\frac{\lambda (t_{k+1}-t_k)-(N(t_{k+1})-N(t_k))}{\lambda (T-t_{k+1})}\right]=0
\end{aligned}\end{equation}
The analysis the second moment of (ii), the second moment of (iii), and the crossing term among $d'_k-d'_0$, (ii), and (iii) is very similar to those in the proof of Lemma~\ref{lem: DARes} and its modification in the proof of Proposition \ref{Prop:RegretBoundDA_2} to deal with the dependence between price and reward and resource consumption in the first batch. The key point is to condition on $N(t_k),N(t_{k+1})$ and use the independence given by Assumption~\ref{Assumption: Random Arrival}. We only state the conclusions here and omit the details of the proof. For the second moment of (ii),
\begin{equation}\label{eqn: PARex2}
\begin{aligned}
    \mathbb{E}\left[\left(\frac{\lambda (t_{k+1}-t_k)-(N(t_{k+1})-N(t_k))}{\lambda (T-t_{k+1})}d'_k\mathbbm{1}\{k<\bar{\kappa}\}\right)^2\right]&\leq \frac{(\bar{d}+\bar{a})^2}{(\lambda (T-t_{k+1}))^2}\mathbb{E}[N(t_{k+1})-N(t_k)]=\frac{(\bar{d}+\bar{a})^2}{(\lambda B)(K-k-1)^2}
\end{aligned}\end{equation}
For the second moment of (iii), there exists a constant $C_{res,1}$ such that
\begin{equation}\label{eqn: PARex3}
\begin{aligned}
    &\mathbb{E}\left[\left(\frac{\sum_{j=N(t_k)}^{N(t_{k+1})}a_j(\mathbbm{1}\{r_j>a_jp^*(d'_k)\}-\mathbbm{1}\{r_j>a_jp_{k+1}\})}{\lambda (T-t_{k+1})}\mathbbm{1}\{k<\bar{\kappa}\}\right)^2\right]\leq \frac{C_{res,1}}{(\lambda B)(K-k-1)^2}
\end{aligned}\end{equation}
For the crossing term between $d'_k-d'_0$ and (ii),
\begin{equation}\label{eqn: PARex5}
\begin{aligned}
    &\mathbb{E}\left[(d'_k-d'_0)\left(\frac{\sum_{j=N(t_k)}^{N(t_{k+1})}d'_k-A_j(p^*(d'_{k}))\}}{\lambda (T-t_{k+1})}\mathbbm{1}\{k<\bar{\kappa}\}\right)\right]=0
\end{aligned}\end{equation}
For the crossing term between $d'_k-d'_0$ and (iii), there exists a constant $C_{res,2}$ such that
\begin{equation}\label{eqn: PARex6}
\begin{aligned}
    &\mathbb{E}\left[(d'_k-d'_0)\left(\frac{\sum_{j=N(t_k)}^{N(t_{k+1})}A_j(p^*(d'_{k}))-A_j(p_{k+1})}{\lambda (T-t_{k+1})}\mathbbm{1}\{k<\bar{\kappa}\}\right)\right]\\
    \leq &\frac{\sqrt{C_{res,2}}}{\sqrt{\lambda B}(K-k-1)\sqrt{k+1}}\sqrt{\mathbb{E}\left[(d'_k-d'_0)^2\right]}
\end{aligned}\end{equation}
For the rest of the crossing terms, we apply Cauchy-Schwarz Inequality to upper bound them, and by (\ref{eqn: PARex1}), (\ref{eqn: PARex4}), (\ref{eqn: PARex2}), (\ref{eqn: PARex3}), (\ref{eqn: PARex5}), (\ref{eqn: PARex6}), there exists a constant $C_{res}$ such that
\begin{equation*}
    \begin{aligned}
    \mathbb{E}\left[(d'_{k+1}-d'_0)^2\right]\leq\frac{C_{res}}{(\lambda B)(K-k-1)^2}+\frac{\sqrt{C_{res}}}{\sqrt{\lambda B}(K-k-1)\sqrt{k+1}}\sqrt{\mathbb{E}\left[(d'_{k}-d'_0)^2\right]}\quad\forall 0\leq k\leq K-2
\end{aligned}
\end{equation*}
Thus, by Lemma \ref{lem:Induction}, there exists a constant $M_{Res}=8C_{res}$ such that
\begin{equation*}
    \sum_{k=1}^{K-1}\mathbb{E}\left[(d'_k-d'_0)^2\right]\leq\frac{M_{Res}}{\lambda B}\log(K)
\end{equation*}
\end{proof}
\begin{lem}\label{lem: PAAux}
    There exists a constant $C_0$ that only depends on the distribution of $(r_1,a_1)$ and the initial average resource $d_0$ such that\\
    (a) $\mathbb{E}\left[N(t_1)\mathbbm{1}\left\{\hat d_0\notin\Omega(\delta_d)\right\}\right]\leq C_0$.\\
    (b) For any $k\geq 2$, $\mathbb{E}\left[\abs{\frac{\lambda}{\hat{\lambda}_{k-1}}-1}^2\right]\leq \frac{C_0}{\lambda B(k-1)}$.\\
    (c) $\mathbb{E}\left[N(t_{1})\norm{\Tilde{d}_{0}-d_{0}}_2^2\right]\leq C_0$.
\end{lem}
\begin{proof}
For part (a),
\begin{equation*}\begin{aligned}
    \mathbb{E}\left[N(t_1)\mathbbm{1}\left\{\hat{d_0}\notin\Omega(\delta_d)\right\}\right]&\leq\sqrt{\mathbb{E}\left[N(t_1)^2\right]}\sqrt{\mathbb{P}\left\{\hat{d_0}\notin\Omega(\delta_d)\right\}}\leq \sqrt{\lambda B+(\lambda B)^2}\sqrt{m\mathbb{P}\left\{\abs{\frac{\lambda}{\hat{\lambda}_0}-1}>\frac{\delta_d}{\bar{d}}\right\}}
\end{aligned}\end{equation*}
and
\begin{equation*}\begin{aligned}
    \mathbb{P}\left\{\abs{\frac{\lambda}{\hat{\lambda}_0}-1}>\frac{\delta_d}{\bar{d}}\right\}\leq&\mathbb{P}\left\{N(t_1)-\lambda B\leq -\frac{\delta_d}{\bar{d}+\delta_d}\lambda B\right\} + \mathbb{P}\left\{N(t_1)-\lambda B\geq \frac{\delta_d}{\bar{d}-\delta_d}\lambda B\right\}+\mathbb{P}\{N(t_1)=0\}\\
    \leq &\exp\left(-\frac{\delta_d^2}{2(\bar{d}+\delta_d)(\bar{d}+2\delta_d)}\lambda B\right) + \exp\left(-\frac{\delta_d^2}{2\bar{d}(\bar{d}-\delta_d)}\lambda B\right)+\exp(-\lambda B)
\end{aligned}\end{equation*}
The proof is complete by combining the above two inequalities. For part (b), Let $N$ be a Poisson random variable with mean $\lambda$, then there exists a constant $C_0$
\begin{equation*}\begin{aligned}
    \lambda\mathbb{E}\left[\left(\frac{\lambda}{\max\{N,1\}}-1\right)^2\right]\leq&\sum_{x=0}^{+\infty}\frac{\lambda^2}{(\max\{x,1\})^2}\frac{\lambda^{x+1}\exp(-\lambda)}{x!}-\frac{2\lambda}{x+1}\frac{\lambda^{x+1}\exp(-\lambda)}{x!}+\frac{\lambda^{x+1}\exp(-\lambda)}{x!}\\
    =&\lambda^3\exp(-\lambda)+\sum_{x=4}^{+\infty}\frac{(x-2)(x-1)x}{(x-3)^2}\frac{\lambda^{x}\exp(-\lambda)}{x!}-2\sum_{x=2}^{+\infty}x\frac{\lambda^{x}\exp(-\lambda)}{x!}+\lambda\\
    =&\lambda^3\exp(-\lambda)+\mathbb{E}\left[\frac{(N-2)(N-1)N}{(N-3)^2}\mathbbm{1}\{N\geq 4\}-2N\mathbbm{1}\{N\geq 2\}\right]+\lambda\\
    = &\lambda^3\exp(-\lambda)+\mathbb{E}\left[N\mathbbm{1}\{N\geq 4\}+\frac{(N-2)(N-1)N-N(N-3)^2}{(N-3)^2}\mathbbm{1}\{N\geq 4\}-2N\mathbbm{1}\{N\geq 2\}\right] +\lambda\\
    \leq &\lambda^3\exp(-\lambda)+\lambda+20-2(\lambda-1)+\lambda\leq C_0
\end{aligned}\end{equation*}
Since $N(t_{k-1})$ is a Poisson random variable with mean $\lambda B(k-1)$, the proof is complete. For part (c),
\begin{equation*}\begin{aligned}
    \mathbb{E}\left[N(t_{1})\norm{\Tilde{d}_{0}-d_{0}}_2^2\right]\leq&m\bar{d}^2\mathbb{E}\left[N(t_1)\abs{\frac{\lambda}{\hat{\lambda}_0}-1}^2\right]\\
    \leq &m\bar{d}^2\mathbb{E}\left[N(t_1)\left(\frac{\lambda B}{\max\{N(t_1),1\}}-1\right)^2\mathbbm{1}\{N(t_1\geq 1)\}\right]\\
    =&m\bar{d}^2\sum_{x=1}^{+\infty}\frac{(\lambda B-x)^2}{x}\frac{(\lambda B)^x\exp(-\lambda B)}{x!}\\
    \leq &2m\bar{d}^2\left(\sum_{x=1}^{+\infty}\frac{(\lambda B)^{x+2}\exp(-\lambda B)}{(x+1)!}-2\sum_{x=1}^{+\infty}x\frac{(\lambda B)^{x+1}\exp(-\lambda B)}{(x+1)!}+\sum_{x=1}^{+\infty}x^2\frac{(\lambda B)^{x}\exp(-\lambda B)}{(x+1)!}\right)\\
    \leq&2m\bar{d}^2\left(\lambda B\sum_{x=1}^{+\infty}\frac{(\lambda B)^{x+1}\exp(-\lambda B)}{(x+1)!}-2\sum_{x=1}^{+\infty}(x+1)\frac{(\lambda B)^{x+1}\exp(-\lambda B)}{(x+1)!}\right.\\
    &\left.+2\sum_{x=1}^{+\infty}\frac{(\lambda B)^{x+1}\exp(-\lambda B)}{(x+1)!}+\sum_{x=0}^{+\infty}(x+1)\frac{(\lambda B)^{x}\exp(-\lambda B)}{x!}\right)\\
    \leq &2m\bar{d}^2(\lambda B-2\lambda B + \lambda B\exp(-\lambda B)+2+\lambda B + 1)=2m\bar{d}^2(\lambda B\exp(-\lambda B)+3)\leq 2m\bar{d}^2(\exp(-1)+3)
\end{aligned}\end{equation*}
\end{proof}
\subsection{Proof of Proposition \ref{Proposition: Growing Support}}
Fix $n$ and set $S = \lfloor n^{\alpha}\rfloor$. Let $U_j = \bar U_j/S - 1/2S$, and let $w_j(n) = r_j(n) + U_j$. For simplicity of notation, we write $w_j(n)$ and $r_j(n)$ as $w_j$ and $r_j$ in the following since $n$ is fixed. Also, we write $x_j^{\pi}$ as $x_j$, and which algorithm $\pi$ we consider should be clear for the context. For any algorithm $\pi$, we have the following regret decomposition,
\begin{equation}\label{eqn: FiniteContiRegretDecomp}
    \begin{aligned}
        Regret_n(\pi)= &\mathbb{E}\left[\sum_{j=n-\lfloor nd_0\rfloor+1}^{n} r_{(j)} - \sum_{j=n-\lfloor nd_0\rfloor+1}^{n} w_{(j)}\right]\\
        &+\mathbb{E}\left[\sum_{j=n-\lfloor nd_0\rfloor+1}^{n} w_{(j)} - \sum_{j=1}^nw_jx_j \right] +\mathbb{E}\left[\sum_{j=1}^n w_jx_j - \sum_{j=1}^nr_jx_j \right]
    \end{aligned} 
\end{equation}
Lemma \ref{Lemma: OfflineDiff} provides a detailed characterization of the first term on the right-hand-side of (\ref{eqn: FiniteContiRegretDecomp}).
\begin{lem}\label{Lemma: OfflineDiff}
    \begin{equation*}
        \mathbb{E}\left[\sum_{j=n-\lfloor nd_0\rfloor+1}^{n} r_{(j)} - \sum_{j=n-\lfloor nd_0\rfloor+1}^{n} w_{(j)}\right] = -\frac{n}{12 S^2} + \frac{n}{2S^2}\int_{0}^1\bar{B}_2(Sw)f_{\text{Beta}(n-\lfloor nd_0\rfloor, \lfloor nd_0\rfloor)}(w)dw
    \end{equation*}
    where $\bar{B}_2(w) = \{w\}^2 - \{w\} + \frac{1}{6}$, $\{w\}$ is the fraction part of $w$, and $f_{Beta(n-\lfloor nd_0\rfloor, \lfloor nd_0\rfloor)}(w)$ is the probability density function of $Beta(n-\lfloor nd_0\rfloor, \lfloor nd_0\rfloor)$.
\end{lem}
\begin{proof}
By our construction of $w_j$,
\begin{equation*}
    r_j = \frac{\left\lceil S w_j\right\rceil - 0.5}{S}
\end{equation*}
Since the right-hand-side is non-decreasing in $w_j$,
\begin{equation*}
    r_{(j)} = \frac{\left\lceil S w_{(j)}\right\rceil - 0.5}{S}
\end{equation*}
Recall $r_{(j)}$'s and $w_{(j)}$'s are order statistics of $\{r_j\}_{j=1}^n$ and $\{w_j\}_{j=1}^n$. Then,
\begin{equation*}
    \mathbb{E}\left[\sum_{j=n-\lfloor nd_0\rfloor+1}^n r_{(j)}\right] - \mathbb{E}\left[\sum_{j=n-\lfloor nd_0\rfloor+1}^n w_{(j)}\right] = \int_0^1 D(w) f_{total}(w) \, dw
\end{equation*}
where $D(w) = \frac{\lceil S w \rceil - 0.5}{S} - w$, and $f_{total}(w) = \sum_{j=n-\lfloor nd_0\rfloor+1}^n f_{w_{(j)}}(w)$ is the sum of the probability density functions of the uniform order statistics. In fact,
\begin{equation*}
    \begin{aligned}
        f_{total}(w) &= \sum_{j=n-\lfloor nd_0\rfloor+1}^n n \binom{n-1}{j-1} w^{j-1} (1-w)^{n-j}\\
        &= n\sum_{j=n-\lfloor nd_0\rfloor}^{n-1} \binom{n-1}{j} y^j (1-y)^{n-1-j}
    \end{aligned}
\end{equation*}
and
\begin{equation*}
    f_{total}'(w) = n \sum_{j=n-\lfloor nd_0\rfloor}^{n-1} \left[ \underbrace{\binom{n-1}{j} j w^{j-1} (1-w)^{n-1-j}}_{A_j} - \underbrace{\binom{n-1}{j} (n-1-j) w^j (1-w)^{n-2-j}}_{B_j} \right]
\end{equation*}
where $B_j = A_{j+1}$ for $j\leq n-2$. Thus,
\begin{equation*}
    \begin{aligned}
        f_{total}'(w) &= n \binom{n-1}{n-\lfloor nd_0\rfloor} (n-\lfloor nd_0\rfloor) w^{n-\lfloor nd_0\rfloor-1} (1-w)^{\lfloor nd_0\rfloor-1} \\
    &= n \frac{(n-1)!}{(n-\lfloor nd_0\rfloor-1)! (\lfloor nd_0\rfloor-1)!} w^{n-\lfloor nd_0\rfloor-1} (1-w)^{\lfloor nd_0\rfloor-1} \\
    &= n \frac{1}{B(n-\lfloor nd_0\rfloor, \lfloor nd_0\rfloor)} w^{n-\lfloor nd_0\rfloor-1} (1-w)^{\lfloor nd_0\rfloor-1} = n \cdot f_{\text{Beta}(n-\lfloor nd_0\rfloor, \lfloor nd_0\rfloor)}(w)
    \end{aligned}
\end{equation*}
In addition,
\begin{equation*}
    \begin{aligned}
        \int_0^{w}D(s)ds = \frac{\{Sw\}(1-\{Sw\})}{2S^2}
    \end{aligned}
\end{equation*}
Recall $\{Sw\}$ is the fraction part of $Sw$. Then, by integration by parts,
\begin{equation*}
    \begin{aligned}
        \mathbb{E}\left[\sum_{j=n-\lfloor nd_0\rfloor+1}^{n} r_{(j)} - \sum_{j=n-\lfloor nd_0\rfloor+1}^{n} w_{(j)}\right] = -\frac{1}{2S^2}\int_{0}^1\{Sw\}(1-\{Sw\})f_{total}'(w)dw
    \end{aligned}
\end{equation*}
Let $\bar{B}_2(w) = \{w\}^2 - \{w\} + \frac{1}{6}$. Then,
\begin{equation*}
    \begin{aligned}
        \mathbb{E}\left[\sum_{j=n-\lfloor nd_0\rfloor+1}^{n} r_{(j)} - \sum_{j=n-\lfloor nd_0\rfloor+1}^{n} w_{(j)}\right] &= -\frac{1}{2S^2}\int_{0}^1\left(\frac{1}{6}-\bar{B}_2(Sw)\right)f_{total}'(w)dw\\
        &= -\frac{n}{12 S^2} + \frac{n}{2S^2}\int_{0}^1\bar{B}_2(Sw)f_{\text{Beta}(n-\lfloor nd_0\rfloor, \lfloor nd_0\rfloor)}(w)dw
    \end{aligned}
\end{equation*}
\end{proof}
Lemma \ref{Lemma: OfflineDiff} is the main reason on why we require $\alpha\geq 1/2$. Lemma \ref{Lemma: OfflineDiff} implies that, for all large enough $n$ and $\alpha\geq 1/2$, the first term on the right-hand-side of (\ref{eqn: FiniteContiRegretDecomp}) falls $[-1/4,0]$. Furthermore, for the last term in (\ref{eqn: FiniteContiRegretDecomp}), $\mathbb{E}\left[\sum_{j=1}^n w_jx_j - \sum_{j=1}^nr_jx_j \right] =\mathbb{E}\left[\sum_{j=1}^n U_jx_j \right]$. In part (a), we can assume $U_j$'s are independent with the decisions made by the optimal policy. Thus, $\mathbb{E}\left[\sum_{j=1}^n U_jx_j \right]=0$. In part (b),  $U_j$ and $x_j$ are independent for the orders in the last batch. For an order $j$ in Batch $k$ with $k<K$, $U_j$, $\theta_j$ and $(s_k,q_{s_k,k},b_{j-1})$ are independent with each other, then
\begin{equation*}
    \begin{aligned}
        \mathbb{E}\left[U_jx_j\right]&= \mathbb{E}\left[U_jx_j\mathbbm{1}\{\theta_j=s_k\}\right]\\
        &=\mathbb{E}[\mathbb{E}[U_j\mathbbm{1}{\bar U_j>1-q_{s_k,k}}|q_{s_k,k}]\mathbbm{1}\{b_{j-1}\geq 1\}\mathbbm{1}\{\theta_j=s_k\}]\\
        \leq &\frac{1}{2S}\mathbb{E}[q_{s_k,k}(1-q_{s_k,k})\mathbbm{1}\{b_{j-1}\geq 1\}\mathbbm{1}\{\theta_j=s_k\}]\\
        \leq &\frac{1}{8S}\mathbb{E}[\mathbbm{1}\{\theta_j=s_k\}]\leq \frac{1}{8S^2}
    \end{aligned}
\end{equation*}
Then, $\mathbb{E}\left[\sum_{j=1}^n U_jx_j \right]\leq \frac{n}{8S^2}$ which is $O(1)$ when $\alpha\geq 1/2$. Therefore, for the proof of both part (a) and part (b), we only need to focus on the second term in (\ref{eqn: FiniteContiRegretDecomp}). 

For part (a), under the optimal online policy, $\{x_j\}_{j=1}^n$ is a feasible online policy for the multisecretary problem with $\{w_j\}_{j=1}^n$ as the rewards under the batching operation. By our construction, $\{w_j\}_{j=1}^n$ is a sequence of i.i.d. $Uniform(0,1)$ random variables. Then, by Theorem \ref{thm: RegretLowerBound}, there exists $\underline{C}$ and $\underline{K}$ that only depends on the initial average resource $d_0$ and the CDF of $Uniform(0,1)$ such that
\begin{equation*}
    \mathbb{E}\left[\sum_{j=n-\lfloor nd_0\rfloor+1}^{n} w_{(j)} - \sum_{j=1}^nw_jx_j \right]\geq \underline{C}\log K
\end{equation*}
when $K\geq\underline{K}$. This complete the proof of part (a).

For part (b), the general regret upper bound (\ref{eqn: GeneralRegretUpperBound}) derived in the proof of Theorem \ref{Thm:RegretBoundDAMulti} still holds. For the term related to the last batch, since it is the difference between the optimal values of two offline problems, the same proof of Lemma \ref{Lemma: OfflineDiff} implies that it is $O(B/S^2)\leq O(1)$ given $\alpha\geq 1/2$. For the terms related to the first $K-1$ batches, by our construction of $w_j$, we have $x_j = \mathbbm{1}\{w_j>1-(\sum_{s=1}^Sq_{s,k}/S)\}\mathbbm{1}\{b_{j-1}\geq 1\}$. Thus, $1-(\sum_{s=1}^Sq_{s,k}/S)\}$ plays the role of $p_k$ in Algorithm \ref{alg: AhdLAMulti}. Then, if we can show the distance between $1-(\sum_{s=1}^Sq_{s,k}/S)\}$ and the population price $p^*(b_{t_{k-1}}/(n-t_{k-1}))$ (w.r.t $Uniform(0,1)$ reward) has similar upper bound as in Lemma \ref{lem: UniformDualConvergence}, the proof of part (b) is complete. This result is formally summarized in the following definitions and Lemma \ref{Lemma: finite_dual_convergence}. Define $s^*=\min\{s\in\mathbb{N}_{+}:\sum_{j=1}^n\mathbbm{1}\{\theta_j\leq s\}> n(1-d)\}$, $\eta^*=nd-\sum_{j=1}^n\mathbbm{1}\{\theta_j>s^*\}$, and $q^*_s(n,d) = \mathbbm{1}\{s>s^*\}+\frac{\eta^*\mathbbm{1}\{s=s^*\}}{\sum_{j=1}^n\mathbbm{1}\{\theta_j=s\}}$ for all $1\leq s\leq S$. Note that $n$ here is the sample size but not the one used to define the support size $S$.
\begin{lem}\label{Lemma: finite_dual_convergence}
    \begin{equation*}
        \mathbb{E}\left[\sup_{d\in\Omega(\delta_d)} \abs{1 - \frac{\sum_{s=1}^Sq_s^*(n,d)}{S}-p^*(d)}^2\right]\leq \frac{2}{n} + \frac{8}{S^2}
    \end{equation*}
\end{lem}
\begin{proof}
Let $F_S(s)=\frac{\sum_{j=1}^S\mathbbm{1}\{j\leq s\}}{S}$. Let $\hat{F}_n(s)=\frac{\sum_{j=1}^n\mathbbm{1}\{\theta_j\leq s\}}{n}$, and define $Q_n(d) = \min\{s\in\mathbb{N}_{+}:\hat{F}_n(s)> 1-d\}$. By definition of $p^*(d)$, we have $p^*(d) = 1-d$ for all $d\in\Omega(\delta_d)$. Then, 
\begin{equation*}
    \abs{1 - \frac{\sum_{s=1}^Sq_s^*(n,d)}{S}-p^*(d)}\leq \frac{\abs{Q_n(d)-(1-d)S} + 1}{S}\quad\forall d\in\Omega(\delta_d)
\end{equation*}
In addition, define
\begin{equation*}
    D_n = \sup_{s\in\mathbb{R}}\abs{\hat{F}_n(s)-F_S(s)}
\end{equation*}
By definition of $Q_n(d)$, we have $\hat{F}_n(Q_n(d))>1-d$ and $\hat{F}_n(Q_n(d)-1)\leq1-d$. Then, fix $d\in\Omega(\delta_d)$,
\begin{equation*}
    D_n\geq F_S(Q_n(d)-1) - \hat{F}_n(Q_n(d)-1)\geq\frac{Q_n(d)-1}{S}-(1-d)
\end{equation*}
which implies
\begin{equation*}
    Q_n(d) - (1-d)S \leq SD_n + 1
\end{equation*}
and
\begin{equation*}
    D_n\geq \hat{F}_n(Q_n(d)) - F_S(Q_n(d))>(1-d)-\frac{Q_n(d)}{S}
\end{equation*}
which implies that
\begin{equation*}
    S(1-d) - Q_n(d)<SD_n
\end{equation*}
Thus,
\begin{equation*}
    \abs{Q_n(d)-(1-d)S}\leq SD_n+1
\end{equation*}
Then,
\begin{equation*}
    \sup_{d\in\Omega(\delta_d)} \abs{1 - \frac{\sum_{s=1}^Sq_s^*(n,d)}{S}-p^*(d)}^2\leq 2D_n^2+\frac{8}{S^2}
\end{equation*}
By DKW inequality,
\begin{equation*}
    \begin{aligned}
        \mathbb{E}[D_n^2] &= \int_0^\infty \mathbb{P}(D_n^2 > s) \, ds = \int_0^\infty \mathbb{P}(D_n > \sqrt{s}) \, ds \leq \int_0^\infty 2\exp(-2ns) \, ds = \frac{1}{n}
    \end{aligned}
\end{equation*}
which completes the proof.
\end{proof}

\section{Tables}
Codes associated with the following numerical results can be found in the supplementary materials.
\begin{table}[H]
    \centering
    \begin{tabular}{l c c c c c}
    \hline
       \multicolumn{1}{l}{}  & \multicolumn{1}{c}{$K=2$} & \multicolumn{1}{c}{$K=8$} & \multicolumn{1}{c}{$K=32$} & \multicolumn{1}{c}{$K=64$} & \multicolumn{1}{c}{$K=128$}\\
       \hline
         $m=1,\;n=1280$  &12.02 &18.80 &23.19 &25.74 &27.67   \\
         $m=1,\;n=6400$  &12.12 &18.01 &23.70 &25.86 &28.91   \\
         $m=1,\;n=32000$  &11.57 &18.67 &24.29 &26.75 &29.56   \\
         $m=1,\;n=64000$  &11.49 &17.88 &23.78 &26.22 &29.13   \\
         $m=1,\;n=128000$  &11.82 &17.40 &24.11 &26.85 &29.60 \\  
         $m=4,\;n=1280$  &67.37 &99.04 &112.49 &115.75 &118.95   \\
         $m=4,\;n=6400$  &68.07 &114.21 &145.46 &154.92 &163.16   \\
         $m=4,\;n=32000$  &69.01 &114.23 &152.91 &172.45 &187.90   \\
         $m=4,\;n=64000$  &67.55 &114.56 &154.02 &176.04 &196.08   \\
         $m=4,\;n=128000$  &66.47 &113.30 &153.75 &174.80 &196.11 \\  
        \end{tabular}
    \caption{Regret of $\pi_1$}
    \label{Table: KnownTotalCustomersMulti}
\end{table}
\begin{table}[H]
    \centering
    \begin{tabular}{l c c c c c}
    \hline
       \multicolumn{1}{l}{}  & \multicolumn{1}{c}{$K=2$} & \multicolumn{1}{c}{$K=8$} & \multicolumn{1}{c}{$K=32$} & \multicolumn{1}{c}{$K=64$} & \multicolumn{1}{c}{$K=128$}\\
       \hline
         $m=1,\;n=1280$  &7.91 &20.12 &25.55 &27.9 &30.39   \\
         $m=1,\;n=6400$  &7.80 &19.70 &25.59 &28.12 &31.10   \\
         $m=1,\;n=32000$  &7.78 &19.12 &24.56 &27.57 &30.31   \\
         $m=1,\;n=64000$  &7.72 &19.37 &25.10 &27.87 &30.65   \\
         $m=1,\;n=128000$  &7.85 &19.39 &24.85 &28.10 &30.94   \\
         $m=4,\;n=1280$  &34.24 &107.24 &120.60 &123.18 &124.95   \\
         $m=4,\;n=6400$  &34.70 &128.85 &161.10 &169.34 &175.85   \\
         $m=4,\;n=32000$  &34.56 &127.79 &171.40 &191.76 &206.30   \\
         $m=4,\;n=64000$  &33.94 &123.56 &169.04 &191.28 &211.29   \\
         $m=4,\;n=128000$  &34.58 &127.21 &171.81 &193.67 &215.01 \\
        \end{tabular}
    \caption{Regret of $\tilde\pi_1$}
    \label{Table: KnownTotalCustomers}
\end{table}

\begin{table}[H]
    \centering
    \begin{tabular}{l c c c c c c}
    \hline
       \multicolumn{1}{l}{}  & \multicolumn{1}{c}{$K=2$} & \multicolumn{1}{c}{$K=8$} & \multicolumn{1}{c}{$K=32$} & \multicolumn{1}{c}{$K=64$} & \multicolumn{1}{c}{$K=128$} \\
       \hline
     $\lambda=10,\;T=128$  & 9.51 & 20.94 & 40.64 & 47.48 & 51.19  \\
     $\lambda=10,\;T=512$  & 9.46 & 31.37 & 42.92 & 46.02 & 53.31 \\
     $\lambda=10,\;T=1024$  & 9.31 & 32.00 & 42.04 & 46.27 & 52.66 \\
     $\lambda=50,\;T=128$  & 9.64 & 32.14 & 40.79 & 47.72 & 52.20 \\
     $\lambda=50,\;T=512$  & 9.29 & 30.43 & 41.19 & 45.84 & 51.28 \\
     $\lambda=50,\;T=1024$  & 9.71 & 30.60 & 40.74 & 46.78 & 51.96 \\
     $\lambda=100,\;T=128$  & 9.47 & 31.69 & 40.80 & 47.80 & 50.45 \\
     $\lambda=100,\;T=512$  & 9.37 & 30.36 & 42.36 & 46.26 & 51.92 \\
     $\lambda=100,\;T=1024$  & 9.80 & 31.11 & 41.73 & 46.02 & 51.42 \\
     $\lambda=500,\;T=128$  & 9.94 & 32.32 & 40.98 & 46.52 & 52.40 \\
     $\lambda=500,\;T=512$  & 9.61 & 30.54 & 42.04 & 46.87 & 51.81 \\
     $\lambda=500,\;T=1024$  & 9.61 & 31.20 & 41.41 & 46.56 & 51.66 \\
    \end{tabular}
    \caption{Regret of $\pi_2$}
    \label{Table: UnknownRate}
\end{table}
\begin{table}[H]
    \centering
    \begin{tabular}{l c c c c c c }
    \hline
       \multicolumn{1}{l}{}  & \multicolumn{1}{c}{$K=2$} & \multicolumn{1}{c}{$K=8$} & \multicolumn{1}{c}{$K=32$} & \multicolumn{1}{c}{$K=64$} & \multicolumn{1}{c}{$K=128$} \\
       \hline
     $\lambda=10,\;T=128$  & 9.66 & 19.75 & 28.39 & 34.18 & 36.76 \\
     $\lambda=10,\;T=512$  & 9.82 & 20.42 & 29.81 & 33.45 & 38.96 \\
     $\lambda=10,\;T=1024$  & 9.44 & 20.13 & 30.03 & 33.10 & 38.61 \\
     $\lambda=50,\;T=128$  & 9.62 & 20.17 & 28.67 & 33.57 & 38.66 \\
     $\lambda=50,\;T=512$  & 9.43 & 19.71 & 29.17 & 34.02 & 37.46 \\
     $\lambda=50,\;T=1024$  & 10.26 & 19.32 & 29.06 & 33.54 & 37.99 \\
     $\lambda=100,\;T=128$  & 9.82 & 19.94 & 29.21 & 33.71 & 37.49 \\
     $\lambda=100,\;T=512$  & 9.88 & 19.78 & 29.03 & 33.18 & 38.29 \\
     $\lambda=100,\;T=1024$  & 9.66 & 19.73 & 29.03 & 33.08 & 37.77 \\
     $\lambda=500,\;T=128$  & 9.61 & 20.30 & 29.33 & 33.28 & 38.19 \\
     $\lambda=500,\;T=512$  & 9.66 & 19.05 & 30.38 & 34.15 & 38.30 \\
     $\lambda=500,\;T=1024$  & 9.95 & 20.35 & 30.13 & 34.18 & 37.91 \\
    \end{tabular}
    \caption{Regret of Modified $\tilde\pi_2$}
    \label{Table: KnownRate}
\end{table}

\begin{table}[H]
    \centering
    \begin{tabular}{l *{5}{S[table-format=1.2]}}
        \toprule
        & \multicolumn{5}{c}{$K$} \\
        \cmidrule(lr){2-6}
        $n$ & {2} & {8} & {32} & {64} & {128} \\
        \midrule
        \multicolumn{6}{l}{\rule{0pt}{3ex}\textbf{$\alpha = 0$}} \\
        \midrule
        1280   & 0.66 & 2.23 & 4.25 & 5.06 & 5.90 \\
        6400   & 0.00 & 0.35 & 1.77 & 2.64 & 3.73 \\
        32000  & 0.00 & 0.00 & 0.22 & 0.67 & 1.48 \\
        64000  & 0.00 & 0.00 & 0.00 & 0.18 & 0.61 \\
        128000 & 0.00 & 0.00 & 0.00 & 0.00 & 0.20 \\
        \midrule
        \multicolumn{6}{l}{\rule{0pt}{3ex}\textbf{$\alpha = 0.25$}} \\
        \midrule
        1280   & 2.24 & 4.32 & 6.02 & 7.07 & 7.92 \\
        6400   & 2.21 & 4.56 & 6.52 & 7.32 & 8.43 \\
        32000  & 2.31 & 4.39 & 6.32 & 7.23 & 8.31 \\
        64000  & 2.07 & 4.03 & 5.92 & 6.84 & 7.93 \\
        128000 & 2.74 & 4.89 & 6.92 & 7.84 & 8.63 \\
        \midrule
        \multicolumn{6}{l}{\rule{0pt}{3ex}\textbf{$\alpha = 0.5$}} \\
        \midrule
        1280   & 2.30 & 4.33 & 6.05 & 7.01 & 7.69 \\
        6400   & 2.31 & 4.45 & 6.37 & 7.27 & 8.26 \\
        32000  & 2.15 & 4.03 & 5.72 & 6.79 & 7.75 \\
        64000  & 2.28 & 4.58 & 6.56 & 7.45 & 8.57 \\
        128000 & 2.29 & 4.21 & 6.08 & 7.16 & 8.32 \\
        \midrule
        \multicolumn{6}{l}{\rule{0pt}{3ex}\textbf{$\delta = 0.75$}} \\
        \midrule
        1280   & 2.38 & 4.46 & 6.37 & 7.16 & 7.93 \\
        6400   & 2.32 & 4.31 & 6.03 & 6.94 & 7.87 \\
        32000  & 2.28 & 4.32 & 6.09 & 7.09 & 8.03 \\
        64000  & 2.11 & 4.40 & 6.05 & 7.05 & 8.05 \\
        128000 & 2.33 & 4.31 & 6.33 & 7.25 & 8.25 \\
        \midrule
        \multicolumn{6}{l}{\rule{0pt}{3ex}\textbf{$\alpha = 1$}} \\
        \midrule
        1280   & 2.21 & 4.31 & 6.26 & 7.25 & 7.80 \\
        6400   & 2.29 & 4.09 & 6.10 & 7.12 & 8.14 \\
        32000  & 2.20 & 4.43 & 6.23 & 7.16 & 8.11 \\
        64000  & 2.22 & 4.23 & 6.22 & 7.19 & 8.14 \\
        128000 & 2.21 & 4.23 & 5.98 & 7.02 & 7.83 \\
        \bottomrule
    \end{tabular}
    \caption{Regret of Algorithm \ref{alg: finite_conti}}\label{Table: Finit_Conti}
\end{table}

\end{document}